\documentclass[pdflatex,sn-basic,iicol]{sn-jnl}

\usepackage{graphicx}%
\usepackage{multirow}%
\usepackage{amsmath,amssymb,amsfonts}%
\usepackage{amsthm}%
\usepackage{mathrsfs}%
\usepackage[title]{appendix}%
\usepackage{textcomp}%
\usepackage{manyfoot}%
\usepackage{booktabs}%
\usepackage{algorithm}%
\usepackage{algorithmicx}%
\usepackage{algpseudocode}%
\usepackage{listings}%
\usepackage[9pt]{extsizes}
\hypersetup{
  colorlinks=true,
  linkcolor=red,
  citecolor=[rgb]{0.05098039, 0.27843137, 0.63137255},
  urlcolor=[rgb]{0.2, 0.41176471, 0.11764706}
}

\usepackage[table]{xcolor}

\usepackage[T1]{fontenc}    
\usepackage{outlines}
\usepackage{bm}
\usepackage{siunitx}

\usepackage{accents}
\usepackage{multicol}
\usepackage{optidef}
\usepackage{subcaption}
\usepackage[export]{adjustbox}
\usepackage{comment}

\newif\ifshownextpaper
\shownextpaperfalse

\newif\ifshowcomments
\showcommentsfalse

\ifshownextpaper
    \newenvironment{fornextpaper}{
    \color{red} \it
        \todo[color=red,inline]{\bf The bellow content is for the next paper!!}
    }{\par}
\else
    \excludecomment{fornextpaper}
\fi

\ifshowcomments
    \newenvironment{toberemoved}{\color{ProcessBlue} \it
        \todo[color=blue!30,inline]{\bf The bellow content is to be removed!!}
    }{\par}
\else
    \excludecomment{toberemoved}
\fi

\usepackage[createShortEnv]{proof-at-the-end}

\DeclareMathAlphabet{\mathdutchcal}{U}{dutchcal}{m}{n}
\SetMathAlphabet{\mathdutchcal}{bold}{U}{dutchcal}{b}{n}
\DeclareMathAlphabet{\mathdutchbcal}{U}{dutchcal}{b}{n}
\DeclareMathAlphabet\mathzapf       {T1}{pzc} {b} {n}

\newcommand*{\dequations}[3]{
\begin{subequations}
\begin{equation*}
\refstepcounter{equation}\latexlabel{#3:1}
\refstepcounter{equation}\latexlabel{#3:2}
#1,\ #2
\tag{\ref*{#3:1}, \ref*{#3:2}}
\end{equation*}
\label{#3}
\end{subequations}
}

\newcommand*{\doubleequation}[2][]{%
\begin{subequations}
\begin{align}
#1\\
#2
\end{align}
\end{subequations}
}

\AtBeginDocument{\let\latexlabel\label}

\numberwithin{equation}{section}

\newcommand\munderbar[1]{%
  \underaccent{\bar}{#1}}

\newtheorem{definition}{Definition}
\newtheorem{theorem}{Theorem}[section]
\newtheorem{lemma}{Lemma}[section]

\begin{document}

\title[Correspondence Free Multivector Cloud Registration using Conformal Geometric Algebra]{Correspondence Free Multivector Cloud Registration using Conformal Geometric Algebra }

\author[1]{\fnm{Francisco} \sur{Vasconcelos}}\email{francisco.vasconcelos.99@tecnico.ulisboa.pt}

\author[1]{\fnm{Jacinto} \sur{C. Nascimento}}\email{jan@isr.tecnico.ulisboa.pt}


\affil[1]{\orgdiv{ISR-IST University of Lisbon} \orgname{}, \orgaddress{\street{Av. Rovisco Pais, 1}, \city{Lisbon}, \postcode{1049-001}, \country{Portugal}}}



\abstract{We present, for the first time, a novel theoretical approach to address the problem of correspondence free multivector cloud registration in conformal geometric algebra. Such formalism achieves several favorable properties. Primarily, it forms an orthogonal automorphism that extends beyond the typical vector space to the entire conformal geometric algebra
while respecting the multivector grading. Concretely, the registration can be viewed as an orthogonal transformation ({\it i.e.}, scale, translation, rotation) belonging to $SO(4,1)$ - group of special orthogonal transformations in conformal geometric algebra.
We will show that such formalism is able to: $(i)$ 
perform the registration without directly accessing the input multivectors.  
Instead, we use primitives or geometric objects provided by the conformal model - the multivectors, 
$(ii)$ the geometric objects are obtained by solving a multilinear eigenvalue problem to find sets of eigenmultivectors. In this way, we can explicitly avoid solving the correspondences in the registration process. 
Most importantly, this offers rotation and translation equivariant properties between the input multivectors and the eigenmultivectors.
Experimental evaluation is conducted in datasets commonly used in point cloud registration,  to testify the usefulness of the approach with emphasis to ambiguities arising from high levels of noise. The code is available at \href{https://github.com/Numerical-Geometric-Algebra/RegistrationGA}{Registration-GA}.\textbf{ { This work was submitted to the International Journal of Computer Vision and is currently under review.}}}




\maketitle

\section{Introduction}

Point cloud registration is one of the most visited topics in computer vision, being a fundamental but challenging task. Many applications have witnessed the success of using point cloud registration, including 3D scene reconstruction ~\cite{agarwal2011building,schonberger2016structure}, object pose 
estimation~\cite{dang2022learning,wong2017segicp}, and Lidar SLAM ~\cite{deschaud2018imls,zhang2014loam,lu2021hregnet}. In short, the registration aims to align two partially overlapping point clouds by estimating their relative rigid transformation ({\it  i.e.}, 3D rotation and translation). A popular venue to address the large-scale registration problem consists of a two-stage pipeline comprising the extraction of point descriptors 
~\cite{choy2019fully,deng2018ppf, frome2004recognizing,rusu2009fast,salti2014shot,zeng20173dmatch} 
followed by a correspondence stage  between the two point clouds, from which the transformation is obtained geometrically. 
Designing point description, thus becomes a crucial step to provide robustness for the pipeline above. Much effort has been dedicated mainly using traditional and deep learning-based descriptor approaches ({\it i.e.},~\cite{bai2020d3feat,huang2021predator,wang2022you}).
However, the resulting correspondences may still suffer from erroneous matchings, particularly in challenging cases, such as low-overlap, repetitive structures, or noisy point sets, leading to a degradation in the registration process.  
To face these challenges, many outlier filtering strategies have
been proposed to circumvent wrong matches. These include traditional rejection methods using random sample
consensus~\cite{fischler1981random}, point-wise descriptor similarity ~\cite{fischler1981random,lowe2004distinctive} or group-wise spatial consistency~\cite{yang2019performance}. 

In this paper we discuss a hitherto untouched aspect of conformal geometric algebra (CGA) - its application to {\it point cloud registration without correspondences}. CGA is emerging as a new approach to address computer vision problems ~\cite{pillai2020detection,ruhe2023clifford}, having the advantage of offering a simple yet efficient representation of geometric objects and transformations. Particularly, we design a new CGA algorithm that is able to determine rigid transformations between general multivectors. Our contribution is twofold: $(i)$ we introduce a novel CGA algorithm that estimates rigid transformations between two multivector clouds, and $(ii)$  a correspondence free algorithm that solves the registration by using CGA. 
The correspondence free algorithm comprises the two following steps: $(i)$ first, we solve an eigenvalue problem for each multivector cloud, $(ii)$ we then use the eigenmultivectors of each multivector cloud to estimate the rigid transformation. Finally, it will be shown that the proposed algorithm does not impose any constraints regarding the multivectors that are being registered.
To summarize, our contributions comprise the following novel algorithms:
\begin{outline}[enumerate]
        \1 An algorithm for extracting eigenmultivectors that is equivariant to both 3D rotations and translations. It is general since it can be used for any type of objects in CGA (Sec. \ref{sec:gen:approach}).
        \1 An algorithm, using a multivector coefficients approach, to estimate the rotation and translation between objects in CGA, being the objects of any type {\it e.g.}, mixed grade multivectors (Sec. \ref{sec:COEF:approach}).
\end{outline}

The following key advantages characterize the distinctiveness of our approach: $(i)$ no need to define a cost function for the multivector clouds (this is specially important for pseudo-Euclidean spaces where norms do not represent the usual notion of a `physical quantity'), $(ii)$ nor to find multivector correspondences between multivector clouds, 
$(iii)$ it is robust to high levels of noise, 
$(iv)$ the eigenmultivectors act as the principal components of the multivector cloud, which best explains the data statistics (which are $SO(p,q)$ equivariant to the multivector clouds), and $(v)$ the eigenvalues provide $SO(p,q)$ invariant quantities.

\section{Related work}

Traditional class of approaches can be framed in PCA based statistical methods~\cite{celik2008fast,gonzalez2009digital,rehman2018automatic} or {\it purely geometric} methods ~\cite{bustos2017guaranteed,chen2022sc2,leordeanu2005spectral,yang2020teaser,zhou2016fast}. However, with the advances of deep learning in the 3D vision field, other classes have rapidly emerged. {\it End-to-end registration methods} have  achieved increasing attention.
One of the pioneering works is the DCP~\cite{wang2019deep} that exploits feature similarity to establish pseudo correspondences for SVD-based transformation estimation.
But other end-to-end methods, {\it e.g.} \cite{choy2020deep,fu2021robust,li2020iterative,li2019pc,li2020unsupervised,qin2022geometric,zhu2020reference}, contributed for the relevance of this class of approaches.
Other class adopts {\it learning-based} strategy where feature descriptors for 3D matching are exploited constituting an advance when compared to the hand-crafted descriptors, {\it e.g.}~\cite{frome2004recognizing,rusu2009fast,salti2014shot}. 3DMAtch~\cite{zeng20173dmatch} is known as one of the representative learning-based methods using  Siamese 3D CNN to learn the local geometric feature via contrastive loss. Other approaches are also available, exploring fully convolutional network for dense feature extraction
~\cite{choy2019fully,bai2020d3feat,huang2021predator,wang2022you,qin2022geometric,li2022lepard}.
Despite significant progress in deep-based feature descriptor, generating mismatched correspondences (outliers) remains unavoidable. Outlier rejection  thus becomes as a natural step to be accounted for, often using RANSAC~\cite{fischler1981random} or its variants~\cite{barath2018graph,le2019sdrsac,li2020gesac} that use repeated sampling and verification for outlier rejection. However, these methods tend to have a high time cost, particularly in scenes with a significant outlier ratio.
{\it Correspondences-free Registration} methods, are however, able to surpass the above rejection step as they directly estimate the rigid transformation, usually achieved by establishing an end-to-end differentiable network. This can be accomplished by either using {\it soft correspondence}  methods~\cite{fu2021robust,lu2021hregnet,lu2019deepvcp,wang2019deep,yew2020rpm} or  {\it direct regression} methods. The latter class of methods was  pioneered by~\cite{aoki2019pointnetlk} that uses PointNet~\cite{qi2017pointnet} as to extract global features and introducing differentiable Lucas-Kanade algorithm~\cite{lucas1981iterative} to minimize
the feature distance. Several other works~\cite{huang2020feature,li2021pointnetlk,yuan2020deepgmr} followed this line of research. Although encouraging results have been reached, some difficulties still persist in both classes of methodologies. Concretely, soft correspondence struggles to generalize to unseen setups with several sensors, and the direct regression methods still face difficulties in large-scale scenes.

This paper proposes a different strategy, where the registration is viewed as a {\it geometric transformation} of the input data that is  performed using multivectors in CGA. Indeed, CGA~\cite{clifford1871preliminary,hestenes2015space,dorst2002geometric,artin2016geometric} has several favorable properties. CGA can easily avoid the PCA problem of $\it 180^\circ$ when computing the eigenvectors of the covariance matrix to estimate the rotation~\cite{rehman2018automatic}. Also, $(i)$ it
efficiently encodes the transformations and the invariant
elements of classic geometries, $(ii)$ it is intuitive to manipulate geometric objects and treat them as operators, meaning that, we can easily define transformations via the geometric product of objects, being the objects vectors, points, lines, and planes, $(iii)$ it generalizes over dimensions in the sense that transformations and objects are built independently of the space dimensionality, and $(iv)$ it allows to unify a large of mathematical systems like vector algebra, projective geometry, quaternions and Pl$\ddot{\rm u}$cker coordinates.
CGA is an emerging topic in computer vision and some works exist. In~\cite{ruhe2023geometric}  GCANs is proposed as a novel method to incorporate geometry-guided transformations into neural networks. CGAPoseNet+GCAN~\cite{pepe2024cgaposenet+} enhances CGAPoseNet~\cite{pepe2023cga} using Clifford Geometric Algebra to unify quaternions and translation vectors into a single mathematical object, the motor, which can be used to uniquely describe camera poses. 

In this paper, we develop a novel registration framework using CGA, with the following structure. In Sec.~\ref{sec:Bkg-Mot} we specifically address our contributions. Sec.~\ref{sec:EXP} thoroughly validates our proposal. The conclusions are addressed in Sec.~\ref{sec:CONCLUSIONS}. In the Appendix we provide all the proofs of the formulated theorems and review the quintessence of modern Geometric Algebra.

\section{Background and Motivation}\label{sec:Bkg-Mot}

This section outlines our contributions for multivector cloud registration using CGA and the exposition is as follows. In Sec.~\ref{sec:gen:approach}, the {\it eigenmultivector extraction} approach is proposed to extract eigenmultivectors from multivector clouds that will be used to estimate the rigid transformation.  We notice that our proposal deals with any type of multivectors. The rationale behind the proposal is that it takes as the input each multivector cloud and solves an eigenvalue problem for each cloud, retrieving unique eigenmultivectors, forming  a basis for the CGA. In the case where the eigenmultivectors are of unique grade, they can be interpreted as primitives objects such as spheres, point pairs, and circles. Otherwise, only form a basis for the entire conformal geometric algebra, not necessarily being blades. The obtained eigenmultivectors from Sec.~\ref{sec:gen:approach} are then used by the proposed {\it coefficients} method detailed in Sec.~\ref{sec:COEF:approach}, that  will estimate the translations and rotation. The coefficients method deals with multivectors which are, at least, bivector plus trivector. 
We provide solutions with and without correspondences assumption. 

\subsection{Why CGA for point cloud registration}

While traditional approaches, {\it e.g.} PCA based methods, only offer  rotation equivariant and invariant properties through the decomposition of a covariance matrix, with CGA, the eigenvalues and eigenmultivectors provide both rotation and translation invariant and equivariant properties. This offers an attractive property when dealing with point cloud registration. In Tab.~\ref{tab:LA-CGA} (left), we show how the decomposition of the covariance matrix ({\it i.e.} the linear function $f$) is affected by a rotation.  
In Tab.~\ref{tab:LA-CGA} (right), we show how the decomposition of a multilinear function $F$ is affected by a rotation and a translation. Note that the eigenvalues of $f$  are only unaffected by the rotation, while the eigenvalues of $F$ remain unchanged with both the translation and rotation.

\begin{table*}
\small
\caption{\small Difference between Vanilla Geometric Algebra and Conformal Geometric Algebra when performing registration. The vectors $\bm{p}_i$ are eigenvectors of $f$ with corresponding real eigenvalues $\lambda_i$ (left). The multivectors $\bm{P}_i$ are eigenmultivectors of $F$ with corresponding real eigenvalues $\nu_i$ (right). The $\bm{P}^j$'s are the reciprocal multivectors such that $\bm{P}_i*\bm{P}^j=\delta_{ij}$ ($*$ is the scalar product operator). The vectors $\bm{z},\bm{x}_i\in\mathcal{A}_3$ and the multivectors $\bm{Z},\bm{X}_i\in\mathcal{G}_{4,1}$.  }\label{tab:LA-CGA}
\centering
\begin{tabular}{c|c}
      \arrayrulecolor{black!100}\specialrule{2pt}{2\jot}{1pc}
      {\scshape {\bf V}anilla {\bf G}eometric {\bf A}lgebra} &  {\scshape {\bf C}onformal {\bf G}eometric {\bf A}lgebra} \\
      \arrayrulecolor{black!50}\specialrule{2pt}{2\jot}{0pc}
      \\
    {$\!\begin{aligned}
               f(\bm{z})&=\sum_i  \bm{z}\cdot \bm{x}_i \bm{x}_i = \sum_i \lambda_i \bm{z}\cdot \bm{p}_i\bm{p}_i^{-1}  \\
               &\quad \quad \Downarrow \ \bm{x}_i\rightarrow R(\bm{x}_i)\ \Downarrow\\
               g(\bm{z}) &= \sum_i \bm{z}\cdot R(\bm{x}_i) R(\bm{x}_i)  \\
               &=\sum_i \lambda_i \bm{z}\cdot R(\bm{p}_i)R(\bm{p}_i^{-1})\\
               \end{aligned}$}    &   
    {$\!\begin{aligned}
               F(\bm{Z})&=\sum_i  \bm{X}_i\bm{ZX}_i = \sum_i \nu_i \bm{Z}* \bm{P}^i\bm{P}_i \\
&\quad \quad \Downarrow\ \bm{X}_i \rightarrow \munderbar{T}\munderbar{R}(\bm{X}_i)\ \Downarrow\\
G(\bm{Z})&= \sum_i  \munderbar{T}\munderbar{R}(\bm{X}_i)\bm{Z}\munderbar{T}\munderbar{R}(\bm{X}_i) \\
&= \sum_i \nu_i \bm{Z}* \munderbar{T}\munderbar{R}(\bm{P}^i)\munderbar{T}\munderbar{R}(\bm{P}_i)\\
               \end{aligned}$}\\ 
               \arrayrulecolor{black!30}\midrule
\end{tabular}
\end{table*}

\subsection{Eigenmultivector Extraction} 
\label{sec:gen:approach}
In this section, we describe a novel methodology that takes as the inputs the multivectors from each cloud, and retrieves the corresponding eigenmultivectors. We further show that, under noise-free assumptions,  the eigenmultivectors of two multivector clouds are related via  rotation and translation.

Formally, let us consider two multivector clouds $\bm{X}_i\in\mathcal{G}_{4,1}$ and $\bm{Y}_i\in\mathcal{G}_{4,1}$, for $i=1,2,\dots,\ell$, where  $\mathcal{G}_{p,q}$  denotes the geometric algebra of a $2^{p+q}$ dimensional real linear vector space. We aim to find the rotator $\bm{R}\in\mathcal{G}_3^+$, such that $\bm{RR}^\dagger=1$, and a translation vector $\bm{t}\in\mathcal{A}_3$ which best aligns the multivector clouds. Here, $\mathcal{G}_3^+$ denotes the even subalgebra of $\mathcal{G}_3$ and $\mathcal{A}_{p,q}$ is a $(p+q)$-dimensional vector space over the field of real numbers $\mathbb{R}$. We start by assuming that 
\begin{equation}
 \bm{Y}_i = \munderbar{T}\munderbar{R}(\bm{X}_{j_i}) + \bm{N}_i\label{eq:model:Y=TR(X)+N}
\end{equation}
where $\bm{N}_i$ is zero mean Gaussian noise, 
and where we define ${T(\bm{z}) \equiv \bm{TzT}^\dagger}$, ${{R}(\bm{z}) \equiv \bm{RzR}^\dagger}$ for $\bm{z}\in\mathcal{A}_{4,1}$, with ${\bm{T} = 1 + \bm{e}_\infty\bm{t}/2}$. In~\eqref{eq:model:Y=TR(X)+N}, the notation $j_i$ stands for the unique assignment $\bm{X}_{j_i} \leftrightarrow \bm{Y}_{i}$. 

The notation $\munderbar{T}$,$\munderbar{R}$ (used in~\eqref{eq:model:Y=TR(X)+N}) and $\bar{T}$,$\bar{R}$, are used to denote the differential and the adjoint outermorphisms of a linear function $T$, $R$, respectively (for a detailed description of linear transformation  and outermorphisms see~\cite{hestenes_clifford_1984}).
Our proposal starts by defining the following multilinear functions for each of the multivector clouds
\dequations{
F(\bm{Z}) \equiv \sum_{i=1}^\ell \bm{X}_i\bm{Z}\bm{X}_i
}{\ G(\bm{Z}) \equiv \sum_{i=1}^\ell \bm{Y}_i\bm{Z}\bm{Y}_i}
{eq:covariance:functions:general}

Given the equalities in \eqref{eq:model:Y=TR(X)+N} and \eqref{eq:covariance:functions:general} a relationship between $F$ and $G$ is given by
\begin{subequations}
\begin{equation}
G(\bm{Z}) = \munderbar{U}F\bar{U}(\bm{Z}) + N(\bm{Z})\label{eq:relation:multilinear:equations}
\end{equation}
where $\munderbar{U}$ is the differential outermorphism of the linear function
\begin{equation}
U(\bm{z}) \equiv \bm{UzU}^\dagger = \bm{TRzR}^\dagger\bm{T}^\dagger = TR(\bm{z})
\label{eq:U_linear}
\end{equation}
where $\bm{z}\in\mathcal{A}_{4,1}$ and with the motor $\bm{U}\equiv\bm{TR}$. The differential outermorphism $\munderbar{U}$ of $U$, extends through the multivector space $\mathcal{G}_{4,1}$ and takes the form ${\munderbar{U}(\bm{Z}) = \bm{UZU}^\dagger}$, with $\bm{Z}\in\mathcal{G}_{4,1}$. $N$ is the noisy multilinear function given by 
\begin{equation}
N(\bm{Z}) = \sum_{i=1}^\ell \bm{N}_i\bm{Z}\bm{N}_i + \bm{N}_i\bm{Z}\bm{X}_i + \bm{X}_i\bm{Z}\bm{N}_i
\end{equation}
\end{subequations}

Notice that the relation (\ref{eq:relation:multilinear:equations}) is invariant under any permutation of the points $\bm{Y}_{j_i}$. An eigenmultivector can be viewed as the solution to an eigenvalue problem $F(\bm{Z})=\lambda\bm{Z}$ in the space of multivectors. Specifically, in the context of multilinear transformations the concept of eigenvector can be extended to eigenmultivectors (see {\it  Definition}~\ref{def:eigenmultivector} and Appendix \ref{sec:multi:trans}, for a detailed description of Eigenmultivectors).

In the following theorem, we explain how  the eigenmultivectors of $F$ and $G$ are related. These eigenmultivectors are equivariant  with respect to the multivector clouds $\bm{X}_i$ and $\bm{Y}_i$, for $i=1,2,\dots,\ell$, and used to estimate orthogonal transformations.

\begin{theoremEnd}[category=theo:eigmvs:noise:free,text link section]{theorem}
\label{theo:eigmvs:noise:free}
  Let ${\bm{P}_1,\bm{P}_2,\dots, \bm{P}_m\in\mathcal{G}_{4,1}}$ be unique eigenmultivectors of $F$, and let ${\bm{Q}_1,\bm{Q}_2,\dots, \bm{Q}_m\in \mathcal{G}_{4,1}}$ be unique eigenmultivectors of $G$. Assume that (\ref{eq:model:Y=TR(X)+N}) is noise-free. Then, the eigenvectors of $F$ and $G$ are related via a scaled rigid transformation. That is, the following equality holds
\begin{equation}
\bm{Q}_i = s_i \munderbar{U}(\bm{P}_i)
\label{eq:sign}
\end{equation}
where the scalar $s_i\in\mathbb{R}$. Furthermore the eigenvalues of $F$ and $G$ are equal.
\end{theoremEnd}

The proof follows directly from Lemma \ref{lemma:eigendecomp:G=UFU} in the {Appendix} \ref{sec:multi:trans}. The eigenmultivectors of the multilinear transformation $F$ have the important property making them equivariant to the input multivector cloud, {\it i.e.} $\bm{X}_1,\bm{X}_2,\dots,\bm{X}_\ell$. In particular, when the $\bm{X}_i$ suffers an orthogonal transformation $U$, then the eigenmultivectors will also suffer the same transformation (with the scaling ambiguity). Another important property is that if they have an associated unique eigenvalue then by ordering the eigenmultivectors by their eigenvalues the equality \eqref{eq:sign} holds. 
To estimate the scalar $s_i$ in (\ref{eq:sign}) we consider the following {\it Corollary}.
\begin{theoremEnd}[category=cor:scaled:relation,text link section,restate]{corollary}
\label{cor:scaled:relation}
Under the assumptions of Theorem \ref{theo:eigmvs:noise:free} the scalar
in (\ref{eq:sign}) can be determined as follows
\begin{equation}
    s_i = {\langle \bm{Q}_i \bm{Q}_{\mathrm{ref}}\rangle}{\langle \bm{P}_i \bm{P}_{\mathrm{ref}}\rangle}^{-1} \label{eq:signal:estimation}
\end{equation}
where $\bm{P}_{\mathrm{ref}}$ and $\bm{Q}_{\mathrm{ref}}$ are both multivectors in~$\mathcal{G}_{4,1}$ that are related as $\bm{Q}_{\mathrm{ref}} = \munderbar{U}(\bm{P}_{\mathrm{ref}})$.
\end{theoremEnd}
\begin{proofEnd}
Assuming $\bm{Q}_i = s_i \munderbar{U}(\bm{P}_i)$ and $\bm{Q}_{\refe} = \munderbar{U}(\bm{P}_{\refe})$ then we have
\begin{equation}
\begin{split}
\langle \bm{Q}_i\bm{Q}_{\refe}\rangle &= \langle s_i \munderbar{U}(\bm{P}_i) \munderbar{U}(\bm{P}_{\refe})\rangle \\
&= s_i \langle \bm{UP}_i\bm{U}^\dagger \bm{UP}_\refe\bm{U}^\dagger\rangle = s_i\langle \bm{P}_i\bm{P}_\refe\rangle
\label{eq:scalar:prod:scaled:relation}
\end{split}
\end{equation}
where we used (\ref{eq:commutation:property}) to reverse the order of the multivectors and also noting that $\bm{UU}^\dagger = 1$. Then by solving (\ref{eq:scalar:prod:scaled:relation}) for $s_i$, gives $s_i = {\langle \bm{Q}_i\bm{Q}_\refe\rangle}{\langle \bm{P}_i\bm{P}_\refe\rangle^{-1}}$, as intended.
\end{proofEnd}\! In the particular case where $\bm{X}_i$ and $\bm{Y}_i$ are grade one multivectors, we can readily provide simple, yet effective quantities which are used to estimate the scalars $s_i$. Particularly, the multivectors  in~\eqref{eq:signal:estimation} are chosen to be ${\bm{P}_{\mathrm{ref}} = (1+\bm{i})({\bm{e}}_\infty + \bar{\bm{X}}\wedge{\bm{e}}_\infty)}$ and ${\bm{Q}_{\mathrm{ref}} = (1+\bm{i})({\bm{e}}_\infty + \bar{\bm{Y}}\wedge{\bm{e}}_\infty})$, with ${\bar{\bm{X}} = \tfrac{1}{\ell} \sum_{i=1}^\ell\bm{X}_i}$ and ${\bar{\bm{Y}} = \tfrac{1}{\ell} \sum_{i=1}^\ell\bm{Y}_i}$ (see {\it Theorem} \ref{theo:sign:est:vectors}). 
\begin{toberemoved}
Note that $\bm{P}_{\mathrm{ref}}$ and $\bm{Q}_{\mathrm{ref}}$ span all the grades except scalar and pseudoscalar. These elements are invariant to special orthogonal transformations, thus, we do not need to define such quantities for those grades.
\end{toberemoved}
{\it Theorem}~\ref{theo:eigmvs:noise:free} and {\it Collorary}~\ref{cor:scaled:relation} above, suggest that the registration can be accomplished {\it without} correspondences as the eigenmultivectors are ordered by their respective (and unique) eigenvalues. The synopsis of {\it registration with eigenmultivector extraction} approach is shown in Alg.~\ref{alg:gen:app}. We notice that for the step 3 of the Alg.~\ref{alg:gen:app}, we introduce a new algorithm based on {\it multivector coefficients} that is described in Sec.~\ref{sec:COEF:approach}. Note that the registration algorithm computes the normalized eigenmultivectors using Alg.~\ref{alg:eigmv:extraction}, which in turn, takes as input data multiple multivectors and extracts equivariant multivectors by taking the eigendecomposition of the covariance matrix (the multilinear function $F$ and $G$) of the multivector clouds. 

\begin{algorithm}[H]
\small
\caption{Rigid Transformation Estimation with Eignemultivector Extraction}\label{alg:gen:app}
$\bf 1)$ {\bf Input Data:} $\bm{X}_1,\bm{X}_2,\dots,\bm{X}_\ell, \bm{Y}_1,\bm{Y}_2,\dots,\bm{Y}_\ell$\; \\
$\bf 2)$  Use Alg.~\ref{alg:eigmv:extraction} to compute the normalized eigenmultivectors  
$\bm{P}_i$ and $\bm{Q}_i$, for $i=1,\dots,m$ of the multivector clouds $\bm{X}_i$ and $\bm{Y}_i$ for $i=1,\dots,\ell$\; \\
$\bf 3)$ Take $\bm{P}_i$ and $\bm{Q}_i$ to estimate the rigid transformation $\widehat{\bm{U}}=\widehat{\bm{T}}\widehat{\bm{R}}$, using Alg.~\ref{alg:rbm:coeffs:approach} (Sec.~\ref{sec:COEF:approach}) 
\end{algorithm}

\subsection{The Multivector Coefficients Approach}

\label{sec:COEF:approach}

In this section we provide a novel algorithm that can be used to estimate a rigid transformation from a set of multivectors. Concretely, we provide a strategy that receives as the input any two multivector clouds in $\mathcal{G}_{4,1}$ and estimates the rigid transformation between them. The method 
described herein works with any known correpondences, in particular the  correspondences provided by the eigenmultivector extraction method introduced in Sec.~\ref{sec:gen:approach}.

\begin{algorithm}[H]
\small
\caption{Eigenmultivector Extraction}\label{alg:eigmv:extraction}
$\bf 1)$ {\bf Input Data:} $\bm{X}_1,\bm{X}_2,\dots,\bm{X}_\ell$\;\\
$\bf 2)$ Define the multilinear function $F(\bm{Z}) = \sum_{i=1}^\ell \bm{X}_i\bm{Z}\bm{X}_i$,  (eq. \eqref{eq:covariance:functions:general});\\
$\bf 3)$ Compute the eigenmultivectors  $\bm{P}_1,\bm{P}_1,\dots,\bm{P}_m$ and eigenvalues $\lambda_1,\lambda_2,\dots,\lambda_m$ of $F$  (Def.~\eqref{def:eigenmultivector})\;\\ 
$\bf 4)$ Order the eigenmultivectors by their respective eigenvalues\;
Let $\bm{P}_{\mathrm{ref}}$ be some equivariant function of the points $\bm{X}_i$\;\\
$\bf 5)$ Scale the eigenmultivectors as $\bm{P}_i\leftarrow \bm{P}_i/\langle \bm{P}_i\bm{P}_{\mathrm{ref}}\rangle$, \  (eq.~\eqref{eq:signal:estimation}); \\
$\bf 6)$ \textbf{Return} the eigenmultivectors $\bm{P}_1,\bm{P}_2,\dots,\bm{P}_m$
\end{algorithm}

Let  $\bm{P}_1,\bm{P}_2,\dots,\bm{P}_m$ and $\bm{Q}_1,\bm{Q}_2,\dots,\bm{Q}_m$ be two sets of multivectors that are related via ${\bm{Q}_i = \munderbar{U}(\bm{P}_i) = \munderbar{T}\munderbar{R}(\bm{P}_i)}$, that is, we assume they relate via a rigid transformations and  where the {\it correspondences are known}. Our approach starts by breaking down the multivectors $\bm{P}_i$ and $\bm{Q}_i$ in $\mathcal{G}_{4,1}$ into its constituent parts in $\mathcal{G}_3$. Considering a single multivector $\bm{P}\in\mathcal{G}_{4,1}$,  we can write 
\begin{equation}
\begin{split}
\bm{P} &\equiv \bm{e}_o\mathdutchbcal{P}_1 + \bm{e}_\infty\mathdutchbcal{P}_2 + \bm{e}_o\wedge \bm{e}_{\infty} \mathdutchbcal{P}_3 + \mathdutchbcal{P}_4  \\
\end{split}
\label{eq:coefficients:of:P:main:body}
\end{equation}
where $\mathdutchbcal{P}_1,\mathdutchbcal{P}_2,\dots,\mathdutchbcal{P}_4\in\mathcal{G}_{3}$ are the coefficients of the multivector $\bm{P}\in\mathcal{G}_{4,1}$, and $\bm{e}_o$ and $\bm{e}_\infty$ are defined in \eqref{eq:eo:einf:conf:mapping:CGA}. To determine the coefficients from a multivector we define a set of functions $\mathdutchcal{C}_i(\cdot)$, for $i=1,...,4$, which provide a mapping from a multivector $\bm{P}$ to its coefficients, this mapping is defined in \eqref{The:coefficients:functions}.

Next, in Sec.~\ref{sec:REL-COEF} we address how the coefficients $\mathdutchbcal{P}_1,\mathdutchbcal{P}_1,\ldots,\mathdutchbcal{P}_4$ of $\bm{P}$ are affected under a rigid transformations 
({\it Theorem} \ref{theo:coefficients}). In Sec.~\ref{sec:EST-R-T}, we  propose 
a methodology to separately estimate the rotation and translation by using {\it Theorems} \ref{theo:rot:opt} and \ref{theo:opt:trans}. In summary, we will first estimate the rotation using {\it Theorem} \ref{theo:rot:opt} followed by the translation using {\it Theorem} \ref{theo:opt:trans}.

\vspace{-2mm}
\subsubsection{Relation between the coefficients under the rigid transformation}
\vspace{-2mm}
\label{sec:REL-COEF}

The following theorem states how the coefficients are related if the two multivector clouds are related under a rigid transformation.

\begin{theoremEnd}[category=theo:coefficients,text link section]{theorem}
\label{theo:coefficients}

Let the multivectors $\bm{P}\in\mathcal{G}_{4,1}$ and $\bm{Q}\in\mathcal{G}_{4,1}$ be related via the composition of a translation $T$ and a rotation $R$, that is, $\bm{Q} = \munderbar{T}\munderbar{R}(\bm{P})$. Then the coefficients $\mathdutchbcal{P}_i = \mathdutchcal{C}_i(\bm{P})$, $\mathdutchbcal{Q}_i = \mathdutchcal{C}_i(\bm{Q})$ are related as
\begin{subequations}
\begin{equation}
\begin{split}
\mathdutchbcal{Q}_1 &= \munderbar{R}(\mathdutchbcal{P}_1)\\ 
\mathdutchbcal{Q}_2 &= \tfrac{1}{2}\bm{t}^2 \munderbar{R}(\mathdutchbcal{P}_1) - \bm{t}\wedge(\bm{t}\cdot \munderbar{R}(\mathdutchbcal{P}_1)) \\
&+ \munderbar{R}(\mathdutchbcal{P}_2) - \bm{t}\wedge \munderbar{R}(\mathdutchbcal{P}_3) + \bm{t}\cdot \munderbar{R}(\mathdutchbcal{P}_4)  \\
\mathdutchbcal{Q}_3 &= \bm{t}\cdot \munderbar{R}(\mathdutchbcal{P}_1) + \munderbar{R}(\mathdutchbcal{P}_3)\\
\mathdutchbcal{Q}_4 &= \bm{t}\wedge \munderbar{R}(\mathdutchbcal{P}_1) + \munderbar{R}(\mathdutchbcal{P}_4)
\end{split}
\label{eq:sys:rot:trans:CGA}
\end{equation}

Furthermore, we can also show how the coefficients are affected by the inverse transformation ${\bm{P} = \bar{R}\bar{T}(\bm{Q})}$, having
\vspace{-1mm}

\begin{equation}
\begin{split}
\mathdutchbcal{P}_1 &= \bar{R}(\mathdutchbcal{Q}_1)\\ 
\mathdutchbcal{P}_2 &= \tfrac{1}{2}\bm{t}^2 \bar{R}(\mathdutchbcal{Q}_1) + \bar{R}(\bm{t}\wedge(\bm{t}\cdot \mathdutchbcal{Q}_1)) \\
&+ \bar{R}(\mathdutchbcal{Q}_2) + \bar{R}(\bm{t}\wedge \mathdutchbcal{Q}_3) - \bar{R}(\bm{t}\cdot \mathdutchbcal{Q}_4)  \\
\mathdutchbcal{P}_3 &= -\bar{R}(\bm{t}\cdot \mathdutchbcal{Q}_1) + \bar{R}(\mathdutchbcal{Q}_3)\\
\mathdutchbcal{P}_4 &= -\bar{R}(\bm{t}\wedge \mathdutchbcal{Q}_1) + \bar{R}(\mathdutchbcal{Q}_4)
\end{split}
\label{eq:sys:rot:trans:CGA:Q}
\end{equation}

\end{subequations}
\end{theoremEnd}
\begin{proofEnd}

We start by demonstrating how different components are affected by translations $T$ defined via (\ref{eq:def:translation}). Consider $\bm{A}\in\mathcal{G}_{p,q}$ then the following holds
\begin{subequations}
\begin{align}
\begin{split}
T(\bm{e}_o) {}&= (1+\tfrac{1}{2}\bm{e}_\infty\bm{t})\bm{e}_o(1-\tfrac{1}{2}\bm{e}_\infty\bm{t}) = \bm{e}_o - \tfrac{1}{4}\bm{e}_{\infty}\bm{te}_o\bm{e}_{\infty}\bm{t} + \tfrac{1}{2}\bm{e}_\infty\bm{te}_o -\tfrac{1}{2}\bm{e}_o\bm{e}_\infty\bm{t}\\
    &= \bm{e}_o - \tfrac{1}{4} \bm{e}_\infty\bm{e}_o\bm{e}_\infty\bm{t}^2 -\tfrac{1}{2}(\bm{e}_\infty\bm{e}_o + \bm{e}_o\bm{e}_\infty)\bm{t}\\
    &=\bm{e}_o + \tfrac{1}{2}\bm{t}^2\bm{e}_\infty - \bm{e}_\infty\cdot\bm{e}_o\bm{t}\\
    &=\bm{e}_o + \bm{t} + \tfrac{1}{2}\bm{t}^2\bm{e}_\infty
\end{split}\\
\begin{split}
T(\bm{e}_\infty) {}&= \bm{e}_\infty + \tfrac{1}{2}\bm{e}_\infty\bm{te}_\infty - \tfrac{1}{2}\bm{e}_\infty\bm{e}_\infty\bm{t} - \tfrac{1}{4}\bm{e}_\infty\bm{t}\bm{e}_\infty\bm{e}_\infty\bm{t}\\
& = \bm{e}_\infty - \tfrac{1}{2}\bm{e}_\infty^2\bm{t} - \tfrac{1}{2}\bm{e}_\infty^2 -\tfrac{1}{4}\bm{e}_\infty\bm{t}^2\bm{e}_\infty^2 = \bm{e}_\infty
\end{split}\\
\begin{split}
T(\bm{A}) {}&= \bm{A} - \tfrac{1}{4}\bm{e}_\infty\bm{tAe}_\infty\bm{t} + \tfrac{1}{2}\bm{e}_\infty\bm{tA}- \tfrac{1}{2}\bm{Ae}_\infty\bm{t}\\
&= \bm{A} + \tfrac{1}{2}\sum_k \left( \bm{e}_\infty\bm{t}\langle\bm{A}\rangle_k - \langle\bm{A}\rangle_k\bm{e}_\infty\bm{t} \right)\\
&= \bm{A} + \tfrac{1}{2}\sum_k \left( \bm{e}_\infty\bm{t}\langle\bm{A}\rangle_k - (-1)^k\bm{e}_\infty\langle\bm{A}\rangle_k\bm{t} \right)\\
&= \bm{A} + \bm{e}_\infty \sum_k \tfrac{1}{2}\left(\bm{t}\langle\bm{A}\rangle_k - (-1)^k\langle\bm{A}\rangle_k\bm{t} \right)\\
&= \bm{A} + \bm{e}_\infty \sum_k \bm{t}\cdot\langle\bm{A}\rangle_k = \bm{A} + \bm{e}_\infty  \bm{t}\cdot\bm{A}
\end{split}\label{eq:trans:T(A):with:proof}
\end{align}
\end{subequations}
where we used $\bm{e}_\infty\bm{e}_o\bm{e}_\infty = -2\bm{e}_\infty$ and that $\bm{e}_\infty^2=0$ and the reordering rule for the wedge product (\ref{eq:reorder:rules:wedge}) to find that $\bm{e}_\infty\langle\bm{A}\rangle_k = (-1)^k \langle\bm{A}\rangle_k\bm{e}_\infty$. In the last two steps of (\ref{eq:trans:T(A):with:proof}) we used (\ref{eq:inner:geo:prod}) and (\ref{eq:sum:grades}) to recover the inner product of $\bm{t}$ with $\bm{A}$.

Now, applying a translation to (\ref{eq:coefficients:of:P}) gives
\begin{equation}
\begin{split}
T(\bm{P}) &= T(\bm{e}_o)\wedge T(\mathdutchbcal{P}_1) + T(\bm{e}_\infty)\wedge T(\mathdutchbcal{P}_2) + T(\bm{e}_o)\wedge T(\bm{e}_{\infty})\wedge T(\mathdutchbcal{P}_3) +  T(\mathdutchbcal{P}_4)\\
          &=  \left(\bm{e}_o + \bm{t} + \tfrac{1}{2}\bm{t}^2\bm{e}_\infty\right)\wedge \left( \mathdutchbcal{P}_1 + \bm{e}_\infty\bm{t}\cdot\mathdutchbcal{P}_1 \right) + \bm{e}_\infty\wedge (\mathdutchbcal{P}_2 + \bm{e}_\infty\bm{t}\cdot\mathdutchbcal{P}_2) + \\
          &+\left(\bm{e}_o + \bm{t} +\tfrac{1}{2}\bm{t}^2\bm{e}_\infty\right)\wedge\bm{e}_\infty\wedge(\mathdutchbcal{P}_3 + \bm{e}_\infty\bm{t}\cdot\mathdutchbcal{P}_3) + \mathdutchbcal{P}_4 + \bm{e}_\infty\bm{t}\cdot\mathdutchbcal{P}_4\\
          & = \left(\bm{e}_o + \bm{t} + \tfrac{1}{2}\bm{t}^2\bm{e}_\infty\right)\wedge \left( \mathdutchbcal{P}_1 + \bm{e}_\infty\bm{t}\cdot\mathdutchbcal{P}_1 \right) + 
          \bm{e}_\infty\wedge \mathdutchbcal{P}_2 + \left(\bm{e}_o + \bm{t}\right)\wedge\bm{e}_\infty\wedge\mathdutchbcal{P}_3 + \mathdutchbcal{P}_4 + \bm{e}_\infty\bm{t}\cdot\mathdutchbcal{P}_4
\end{split}
\end{equation}
which putting in the same form as (\ref{eq:coefficients:of:P}) we have
\begin{subequations}
\begin{equation}
T(\bm{P}) = \bm{e}_o\mathdutchbcal{P}_1 + \bm{e}_\infty\left( \tfrac{1}{2}\bm{t}^2 \mathdutchbcal{P}_1 - \bm{t}\wedge(\bm{t}\cdot\mathdutchbcal{P}_1) + \mathdutchbcal{P}_2 + \bm{t}\cdot\mathdutchbcal{P}_4 - \bm{t}\wedge\mathdutchbcal{P}_3 \right) + \bm{e}_o\wedge\bm{e}_{\infty}\left( \bm{t}\cdot\mathdutchbcal{P}_1 + \mathdutchbcal{P}_3 \right) + \bm{t}\wedge\mathdutchbcal{P}_1 + \mathdutchbcal{P}_4 \label{eq:translation:of:coeffs:of:P}
\end{equation}
The point at infinity $\bm{e}_\infty$ and the point at the origin $\bm{e}_o$ commute with rotors in $\mathcal{G}_{p,q}$,that is, $R(\bm{e}_o) = \bm{e}_o$ and $R(\bm{e}_\infty) = \bm{e}_\infty$. Thus, applying a rotation $R$ to (\ref{eq:coefficients:of:P}) gives
\begin{equation}
\begin{split}
R(\bm{P}) &= R(\bm{e}_o)\wedge R(\mathdutchbcal{P}_1) + R(\bm{e}_\infty)\wedge R(\mathdutchbcal{P}_2) + R(\bm{e}_o)\wedge R(\bm{e}_{\infty})\wedge R(\mathdutchbcal{P}_3) +  R(\mathdutchbcal{P}_4)\\
&=\bm{e}_o\wedge R(\mathdutchbcal{P}_1) + \bm{e}_\infty\wedge R(\mathdutchbcal{P}_2) + \bm{e}_o \wedge \bm{e}_{\infty}\wedge R(\mathdutchbcal{P}_3) +  R(\mathdutchbcal{P}_4)
\end{split}\label{eq:rotation:pf:coeffs:of:P}
\end{equation}
Thus, for the coefficients of rotations we have that 
\begin{equation}
\mathdutchcal{C}_i(R(\bm{P})) = R(\mathdutchcal{C}_i(\bm{P})) = R(\mathdutchbcal{P}_i),\ \text{for}\ i=1,2,3,4
\end{equation}
\end{subequations}
Since $R(\mathdutchbcal{P}_i)\in\mathcal{G}_{p,q}$ we can also use (\ref{eq:translation:of:coeffs:of:P}) to compose a rotation with a translation. Replace $\bm{P}$ by $R(\bm{P})$ and  $\mathdutchbcal{P}_i$ by $R(\mathdutchbcal{P}_i)$ in (\ref{eq:translation:of:coeffs:of:P}). Then use (\ref{The:coefficients:functions}) to extract the coefficients to arrive at (\ref{eq:sys:rot:trans:CGA}). For the inverse transformation $\bar{R}\bar{T}$ we replace $\bm{t}$ by $-\bm{t}$, $T$ by $\bar{T}$ and $\bm{P}$ by $\bm{Q}$ in (\ref{eq:translation:of:coeffs:of:P}) then apply $\bar{R}$ to that result using (\ref{eq:rotation:pf:coeffs:of:P}) to finally obtain (\ref{eq:sys:rot:trans:CGA:Q}).
\end{proofEnd} Note that from (\ref{The:coefficients:functions}) we can also write 
\dequations{\mathdutchbcal{Q}_i = \mathdutchcal{C}_i(\munderbar{T}\munderbar{R}(\bm{P}))}{\mathdutchbcal{P}_i = \mathdutchcal{C}_i(\bar{R}\bar{T}(\bm{Q})), \;{\rm for}\; i=1,...,4.}
{eq:coefficients:CiTRP:CiRTQ}



\subsubsection{Estimating R and t from multivector coefficients}\label{sec:EST-R-T}

Notice that the first coefficients of $\bm{P}$ and $\bm{Q}$ relate via a rotation (see first equation of~\eqref{eq:sys:rot:trans:CGA}). Now, let $\bm{A}_i = \mathdutchcal{C}_1(\bm{P}_i)\in\mathcal{G}_3$ and $\bm{B}_i = \mathdutchcal{C}_1(\bm{Q}_i)\in\mathcal{G}_3$. It is straightforward to obtain  $\bm{B}_i = \munderbar{R}(\bm{A}_i)$. To estimate the rotation between the multivectors $\bm{A}_i$ and $\bm{B}_i$ we state the following  theorem.
\begin{theoremEnd}[category=theo:opt:rot:trans,text link section]{theorem}
\label{theo:rot:opt}
Let $\bm{A}_i,\bm{B}_i\in\mathcal{G}_3$, with $i = 1, \ldots, \ell$, be two sets of $\ell$ multivectors and  assume that they are related by
\begin{subequations}
\begin{equation}
\bm{B}_i = \bm{RA}_i\bm{R}^\dagger + \bm{N}_i\label{eq:noisy:relation:Bi:Ai}
\end{equation}
where $\bm{RR}^\dagger = 1$ and $\bm{N}_i\in\mathcal{G}_3$ is Gaussian noise. Then, the optimal rotor $\bm{R}$ is determined by minimizing the following Lagrangian
\begin{equation}
\mathcal{L}(\bm{R})  = \sum_{i=1}^\ell \|\bm{B}_i - \bm{RA}_i\bm{R}^\dagger\|^2 + \lambda \langle \bm{RR}^\dagger -1\rangle
\end{equation}
where $\lambda$ is the Lagrange multiplier associated with the constraint $\bm{RR}^\dagger =1$. Since $\bm{R}$ is a rotor, then it has to be of grade zero and grade two, that is, $\bm{R} \equiv \langle \bm{R}\rangle + \langle \bm{R}\rangle_2$. Let 
\begin{equation}
L(\bm{R}) = \sum_{i=1}^\ell \bm{B}_i\bm{RA}_i^\dagger + \bm{B}_i^\dagger \bm{R}\bm{A}_i \label{eq:multilinear:function:opt:rot}
\end{equation}
\end{subequations}
then the optimal rotor is the eigenrotator of $\langle L(\bm{R})\rangle_{0,2}$ with the largest eigenvalue. Note that the multiple grade projection operation $\langle \cdot\rangle_{0,2}$ is defined in \eqref{eq:multiple:grade:proj}.
\end{theoremEnd}
\begin{proofEnd}
Given the relation (\ref{eq:noisy:relation:Bi:Ai}) we choose the Euclidean distance to define the cost function
\begin{equation}
J(\bm{R}) = \sum_{i=1}^\ell \|\bm{B}_i - \bm{RA}_i\bm{R}^\dagger\|^2
\end{equation}
Then by {\it Theorem} \ref{theo:versor:R3} we can express the Lagrangian associated to the problem of minimizing $J$ under the constraint that $\bm{R}$ is a rotator as
\begin{equation}
\mathcal{L}(\lambda,\bm{R}) = J(\bm{R}) + \lambda \langle \bm{RR}^\dagger - 1\rangle
\end{equation}
which by expanding we find 
\begin{equation}
    \mathcal{L}(\bm{R},\lambda) = \sum_{i=1}^\ell |\bm{B}_i|^2 + |\bm{A}_i|^2 - 2\langle \bm{B}_i^\dagger \bm{RA}_i\bm{R}^\dagger\rangle + \lambda \langle\bm{RR}^\dagger - 1\rangle
\end{equation}
taking the derivative while using (\ref{eq:derivatives:properties}), (\ref{eq:commutation:property}) and (\ref{eq:A:dagger:star:B}) we can show that 
\begin{equation}
    \begin{split}
    \partial_{\bm{R}^\dagger}\mathcal{L}(\bm{R},\lambda) &= \partial_{\bm{R}}^\dagger \mathcal{L}(\bm{R}) = \left(\partial_{\bm{R}} \mathcal{L}(\bm{R},\lambda) \right)^\dagger \\
&= -\tfrac{1}{2}\sum_{i=1}^\ell\left( \dot{\partial}_{\bm{R}} \langle \bm{B}_i^\dagger \dot{\bm{R}}\bm{A}_i\bm{R}^\dagger\rangle + \dot{\partial}_{\bm{R}}\langle \bm{B}_i^\dagger \bm{R}\bm{A}_i\dot{\bm{R}}^\dagger\rangle\right)^\dagger + \lambda\left(\partial_{\bm{R}} \langle\bm{RR}^\dagger - 1\rangle\right)^\dagger\\
&= -\tfrac{1}{2}\sum_{i=1}^\ell\left( \dot{\partial}_{\bm{R}} \dot{\bm{R}}*(\bm{A}_i\bm{R}^\dagger\bm{B}_i^\dagger) + \dot{\partial}_{\bm{R}}\dot{\bm{R}}* (\bm{A}_i^\dagger\bm{R}^\dagger\bm{B}_i)\right)^\dagger + \lambda\left(\partial_{\bm{R}} \bm{R}*\bm{R}^\dagger \right)^\dagger\\
&= -\tfrac{1}{2}\sum_{i=1}^\ell\left( \bm{A}_i\bm{R}^\dagger\bm{B}_i^\dagger + \bm{A}_i^\dagger\bm{R}^\dagger\bm{B}_i\right)^\dagger + \lambda (\bm{R}^\dagger)^\dagger\\
&= -\tfrac{1}{2}\sum_{i=1}^\ell \bm{B}_i\bm{R}\bm{A}_i^\dagger + \bm{B}_i^\dagger\bm{R}\bm{A}_i + \lambda \bm{R}
    \end{split}
\end{equation}
Taking $\langle\partial_{\bm{R}}^\dagger\rangle_{0,2} \mathcal{L} = 0$ and taking the isolated $\bm{R}$ to the other side of the equation we find
\begin{subequations}
\begin{equation}
\tfrac{1}{2}\langle L(\bm{R})\rangle_{0,2} = \lambda \bm{R}
\end{equation}
where 
\begin{equation}
L(\bm{R}) =\sum_{i=1}^\ell \bm{B}_i\bm{R}\bm{A}_i^\dagger + \bm{B}_i^\dagger\bm{R}\bm{A}_i.
\end{equation}
\end{subequations}
Then to show that the optimal solution is when $\lambda$ is the largest, we multiply both sides of $\bm{R}^\dagger$ and since $\bm{RR}^\dagger = 1$ thus 
\begin{equation}
\lambda = \tfrac{1}{2}L(\bm{R})\bm{R}^\dagger \label{eq:lambda:F(R)R}
\end{equation}
since the left hand side is a scalar so must also be the right hand side thus we have
\begin{equation}
\begin{split}
\tfrac{1}{2}L(\bm{R})\bm{R}^\dagger &= \tfrac{1}{2}L(\bm{R})*\bm{R}^\dagger = \tfrac{1}{2}\sum_{i=1}^\ell \langle \bm{B}_i^\dagger\bm{RA}_i\bm{R}^\dagger + \bm{B}_i\bm{RA}_i^\dagger\bm{R}^\dagger\rangle \\
& =\sum_{i=1}^\ell \langle \bm{B}_i^\dagger\bm{RA}_i\bm{R}^\dagger\rangle
\end{split}
\label{eq:F(R)R:cost:function}
\end{equation}
where we used (\ref{eq:dagger:scalar:product}) and (\ref{eq:dagger:prop:1}) in order find that 
\begin{equation}
\langle \bm{B}_i\bm{RA}_i^\dagger\bm{R}^\dagger\rangle = \langle \bm{B}_i^\dagger \left(\bm{RA}_i^\dagger\bm{R}^\dagger\right)^\dagger\rangle = \langle \bm{B}_i^\dagger\bm{RA}_i\bm{R}^\dagger\rangle.
\end{equation}
Comparing the last expression of (\ref{eq:F(R)R:cost:function}) with the cost function $J$, we immediately see that this is the component of $J$ which varies with $\bm{R}$. Thus, to minimize $J$, we must maximize this quantity. Also, because of (\ref{eq:lambda:F(R)R}) we can readily see that the optimal solution is obtained when $\lambda$ is the largest. 
We note that a real eigendecomposition must exist since under the product $\langle AB^\dagger \rangle$ (which is definite in $\mathcal{G}_{3}$), the multilinear function $F$ is symmetric (see Theorem~\ref{theo:symmetric:multi}), which means that a real decomposition of the matrix of $F$ must exist.  
\end{proofEnd}

Assume that the rotation $R$ is given, then we can set $\bm{S}_i = \munderbar{R}(\bm{P}_i)$ and consider the following cost function with respect to the translator $\bm{T}$
\begin{equation}
J(\bm{T}) = \sum_{i=1}^N d^2(\bm{TS}_i\bm{T}^\dagger,\bm{Q}_i) + d^2(\bm{S}_i,\bm{T}^\dagger\bm{Q}_i\bm{T})\label{eq:cost:function:opt:trans}
\end{equation}
where $d^2$ is given by (\ref{eq:def:dist:squared:CGA:translation}). The minimum of the above cost measure, under certain conditions, has a closed form solution given by the following theorem.


\begin{theoremEnd}[category=theo:opt:rot:trans,text link section]{theorem}
\label{theo:opt:trans}
Let $\bm{Q}_i\in\mathcal{G}_{4,1}^{k_i}$ and $\bm{S}_i\in\mathcal{G}_{4,1}^{k_i}$ be multivectors of the same grade $k_i>0$. After the minimization of the cost function \eqref{eq:cost:function:opt:trans},  the optimal translation $\bm{t}$ obtained 
under the constraint ${\bm{T} = 1 + \tfrac{1}{2}\bm{e}_{\infty}\bm{t}}$ is  given as 
\begin{subequations}
\begin{equation}
\begin{split}
\bm{t} &= s_N^{-1} \sum_{i=1}^N \Big\langle \left(\mathdutchcal{C}_1(\bm{S}_i)+\mathdutchcal{C}_1(\bm{Q}_i)\right)\Big(\mathdutchcal{C}_3(\bm{Q}_i) + \mathdutchcal{C}_4(\bm{Q}_i)\\
&\ \ - \mathdutchcal{C}_3(\bm{S}_i) - \mathdutchcal{C}_4(\bm{S}_i)\Big)^\dagger\Big\rangle_1\label{eq:t:opt:opt:trans}
\end{split}
\end{equation}
with 
\begin{equation}
s_N = \sum_{i=1}^N \|\mathdutchcal{C}_1(\bm{S}_i)\|^2 + \|\mathdutchcal{C}_1(\bm{Q}_i)\|^2.\label{eq:sum:squares:opt:trans}
\end{equation}
\label{eq:t:opt:opt:trans:sum:squares}
\end{subequations}
and where $\mathdutchcal{C}_i$ is given by (\ref{The:coefficients:functions}).
\end{theoremEnd}
\begin{proofEnd}
Let $\mathdutchbcal{Q}_{ij} = \mathdutchcal{C}_j(\bm{Q}_i)$, $\mathdutchbcal{S}_{ij} = \mathdutchcal{C}_j(\bm{S}_i)$ and let $\bm{S}_i = R(\bm{P}_i)$ then considering (\ref{eq:coefficients:CiTRP:CiRTQ}) and (\ref{eq:sys:rot:trans:CGA}) we find that 
\begin{subequations}
\begin{align}
\mathdutchcal{C}_3(\bm{TS}_i\bm{T}^\dagger) &= \bm{t}\cdot \mathdutchcal{C}_1(\bm{S}_i) + \mathdutchcal{C}_3(\bm{S}_i) = \bm{t}\cdot \mathdutchbcal{S}_{i1} + \mathdutchbcal{S}_{i3} \\
\mathdutchcal{C}_4(\bm{TS}_i\bm{T}^\dagger) &= \bm{t}\wedge \mathdutchcal{C}_1(\bm{S}_i) + \mathdutchcal{C}_4(\bm{S}_i) = \bm{t}\wedge \mathdutchbcal{S}_{i1} + \mathdutchbcal{S}_{i4} 
\end{align}
\label{eq:CTST}
\end{subequations}
Letting $R$ be the identity transformation in  (\ref{eq:sys:rot:trans:CGA:Q}) we may find that 
\begin{subequations}
\begin{align}
\mathdutchcal{C}_3(\bm{T}^\dagger\bm{Q}_i\bm{T}) &= -\bm{t}\cdot \mathdutchcal{C}_1(\bm{Q}_i) + \mathdutchcal{C}_3(\bm{Q}_i) = -\bm{t}\cdot \mathdutchbcal{Q}_{i1} + \mathdutchbcal{Q}_{i3} \\
\mathdutchcal{C}_4(\bm{T}^\dagger\bm{Q}_i\bm{T}) &= -\bm{t}\wedge \mathdutchcal{C}_1(\bm{Q}_i) + \mathdutchcal{C}_4(\bm{Q}_i) = -\bm{t}\wedge \mathdutchbcal{Q}_{i1} + \mathdutchbcal{Q}_{i4} 
\end{align}
\label{eq:CTQT}
\end{subequations}

Using the defining equations (\ref{eq:def:dist:squared:CGA:translation}) for $d^2$ we readily have that 
\begin{subequations}
\begin{align}
d^2(\bm{TS}_i\bm{T}^\dagger,\bm{Q}_i) &= \| \mathdutchcal{C}_3(\bm{TS}_i\bm{T}^\dagger) - \mathdutchcal{C}_3(\bm{Q}_i)\|^2 + \| \mathdutchcal{C}_4(\bm{TS}_i\bm{T}^\dagger) - \mathdutchcal{C}_4(\bm{Q}_i)\|^2 \\
d^2(\bm{S}_i,\bm{T}^\dagger\bm{Q}_i\bm{T}) &= \| \mathdutchcal{C}_3(\bm{S}_i) - \mathdutchcal{C}_3(\bm{T}^\dagger\bm{Q}_i\bm{T})\|^2 + \| \mathdutchcal{C}_4(\bm{S}_i) - \mathdutchcal{C}_4(\bm{T}^\dagger\bm{Q}_i\bm{T})\|^2
\end{align}
\end{subequations}
Using (\ref{eq:CTST}) and (\ref{eq:CTQT}) in the above equations and considering the cost function (\ref{eq:cost:function:opt:trans}) we may write
\begin{equation}
\begin{split}
J &= \|\bm{t}\cdot\mathdutchbcal{S}_{i1} + \mathdutchbcal{S}_{i3} - \mathdutchbcal{Q}_{i3}\|^2 + \|\bm{t}\wedge\mathdutchbcal{S}_{i1} + \mathdutchbcal{S}_{i4} - \mathdutchbcal{Q}_{i4}\|^2 + \\
&+\|\bm{t}\cdot\mathdutchbcal{Q}_{i1} + \mathdutchbcal{S}_{i3} - \mathdutchbcal{Q}_{i3}\|^2 + \|\bm{t}\wedge\mathdutchbcal{Q}_{i1} + \mathdutchbcal{S}_{i4} - \mathdutchbcal{Q}_{i4}\|^2\\
&= \|\bm{t}\cdot\mathdutchbcal{S}_{i1}\|^2 + \|\bm{t}\wedge\mathdutchbcal{S}_{i1}\|^2 + \|\bm{t}\cdot\mathdutchbcal{Q}_{i1}\|^2 + \|\bm{t}\wedge\mathdutchbcal{Q}_{i1}\|^2 + \\
&+\|\mathdutchbcal{S}_{i3} - \mathdutchbcal{Q}_{i3}\|^2 + \|\mathdutchbcal{S}_{i4} - \mathdutchbcal{Q}_{i4}\|^2 + 
\|\mathdutchbcal{S}_{i3} - \mathdutchbcal{Q}_{i3}\|^2 + \|\mathdutchbcal{S}_{i4} - \mathdutchbcal{Q}_{i4}\|^2\\
&+ 2\langle \bm{t}\cdot(\mathdutchbcal{S}_{i1} + \mathdutchbcal{Q}_{i1})(\mathdutchbcal{S}_{i3} - \mathdutchbcal{Q}_{i3})^\dagger\rangle + 2\langle \bm{t}\wedge(\mathdutchbcal{S}_{i1} + \mathdutchbcal{Q}_{i1})(\mathdutchbcal{S}_{i4} - \mathdutchbcal{Q}_{i4})^\dagger\rangle
\end{split}
\label{eq:cost:expanded:opt:trans}
\end{equation}
Where we omitted the sum $\sum_{i=1}^N$. Recall that we are summing over the index $i=1,2,\dots,N$.
Then using (\ref{eq:sum:square:inner:outer:geo}) and (\ref{eq:scalar:wedge:1:k:k+1/k-1}) to simplify the above expression and taking the derivative with respect to the vector $\bm{t}$ we have
\begin{equation}
\begin{split}
\partial_{\bm{t}} J &=  \partial_{\bm{t}} \sum_{i=1}^N  \bm{t}^2\left(\|\mathdutchbcal{S}_{i1}\|^2 + \|\mathdutchbcal{Q}_{i1}\|^2\right) +\\
&+ 2\partial_{\bm{t}}\sum_{i=1}^N\left(\bm{t}\cdot\langle (\mathdutchbcal{S}_{i1} + \mathdutchbcal{Q}_{i1})(\mathdutchbcal{S}_{i3} - \mathdutchbcal{Q}_{i3})^\dagger \rangle_1 + 2\bm{t}\cdot\langle (\mathdutchbcal{S}_{i1} + \mathdutchbcal{Q}_{i1})(\mathdutchbcal{S}_{i4} - \mathdutchbcal{Q}_{i4})^\dagger \rangle_1 \right)\\
&= 2\bm{t} \left(\sum_{i=1}^N \|\mathdutchbcal{S}_{i1}\|^2 + \|\mathdutchbcal{Q}_{i1}\|^2\right) + 2 \sum_{i=1}^N \langle (\mathdutchbcal{S}_{i1} + \mathdutchbcal{Q}_{i1})(\mathdutchbcal{S}_{i4} + \mathdutchbcal{S}_{i3} - \mathdutchbcal{Q}_{i4} - \mathdutchbcal{Q}_{i3})^\dagger \rangle_1
\end{split}\label{eq:t:derivative:of:J:opt:trans}
\end{equation}
which by solving for $\bm{t}$ we get (\ref{eq:t:opt:opt:trans}).

Note that, the use of (\ref{eq:scalar:wedge:1:k:k+1/k-1}) in (\ref{eq:t:derivative:of:J:opt:trans}) is valid since we are considering that $\bm{Q}_i$ and $\bm{P}_i$ are multivectors of unique grade equal to $k_i$. When $k_i>1$, then the coefficients $\mathdutchbcal{S}_{i1}$ and $\mathdutchbcal{Q}_{i1}$ are of grade $k_i-1$, $\mathdutchbcal{S}_{i3}$ and $\mathdutchbcal{Q}_{i3}$ are of grade $k_i-2$ and $\mathdutchbcal{S}_{i4}$ and $\mathdutchbcal{Q}_{i4}$ are of grade $k_i$. And in the particular case when we consider $k_i=1$ then $\mathdutchbcal{S}_{i3}$ and $\mathdutchbcal{Q}_{i3}$ are going to be equal to zero. 
\end{proofEnd}

Another approach can be used to estimate the translation, in particular, without noise. Knowing the rotation $R$, we can express the solution to the translation via one of the correspondences. Let $\bm{P}\equiv \bm{P}_k$ and $\bm{Q}\equiv\bm{Q}_k$ for some integer $k$, also note that ${\mathdutchbcal{P}_i = \mathdutchcal{C}_i(\bm{P}) = \mathdutchcal{C}_i(\bm{P}_k)}$, ${\mathdutchbcal{Q}_i = \mathdutchcal{C}_i(\bm{Q})=\mathdutchcal{C}_i(\bm{Q}_k)}$, then we can determine the translation exactly by the following theorem.
\begin{theoremEnd}[restate,category=theo:exact:trans,text link section]{theorem}
The translation vector $\bm{t}$ can be determined exactly from~\eqref{eq:sys:rot:trans:CGA} as 
\begin{equation}
\begin{split}
\bm{t} &= (\mathdutchbcal{Q}_3+\mathdutchbcal{Q}_4 - {\munderbar{R}}(\mathdutchbcal{P}_3+\mathdutchbcal{P}_4))\munderbar{R}(\mathdutchbcal{P}_1^{-1}) \\&= (\mathdutchbcal{Q}_3+\mathdutchbcal{Q}_4 - {\munderbar{R}}(\mathdutchbcal{P}_3+\mathdutchbcal{P}_4))\mathdutchbcal{Q}_1^{-1}\label{eq:coeffs:exact:translation}
\end{split}
\end{equation}
\end{theoremEnd}
\begin{proofEnd}
Take the bottom two equations of~\eqref{eq:sys:rot:trans:CGA} and add them together, then
\begin{equation*}
\mathdutchbcal{Q}_3 + \mathdutchbcal{Q}_4 = \bm{t}\cdot \munderbar{R}(\mathdutchbcal{P}_1) + \munderbar{R}(\mathdutchbcal{P}_3) + \bm{t}\wedge \munderbar{R}(\mathdutchbcal{P}_1) + \munderbar{R}(\mathdutchbcal{P}_4)
\end{equation*}
Using~\eqref{eq:geo:inner:outer:prods} we have ${\bm{t}\cdot \munderbar{R}(\mathdutchbcal{P}_1) + \bm{t}\wedge \munderbar{R}(\mathdutchbcal{P}_1) = \bm{t}\munderbar{R}(\mathdutchbcal{P}_1)}$ and considering that $\mathdutchbcal{P}_1$ is a simple multivector, then solving for $\bm{t}$ we find that 
\begin{equation*}
\bm{t} = \left( \mathdutchbcal{Q}_3 + \mathdutchbcal{Q}_4 - \munderbar{R}(\mathdutchbcal{P}_3 + \mathdutchbcal{P}_4)\right)\munderbar{R}(\mathdutchbcal{P}_1^{-1})
\end{equation*}
then using the first equation of~\eqref{eq:sys:rot:trans:CGA} we can put this in the form
\begin{equation*}
\bm{t} = \left( \mathdutchbcal{Q}_3 + \mathdutchbcal{Q}_4 - \munderbar{R}(\mathdutchbcal{P}_3 + \mathdutchbcal{P}_4)\right)\mathdutchbcal{Q}_1^{-1}
\end{equation*}
\end{proofEnd}
{ In summary, our proposal takes the first coefficient of each multivectors {\it i.e.},  ${\bm{A}_i= \mathdutchcal{C}_1(\bm{P}_i)}$ and ${\bm{B}_i = \mathdutchcal{C}_1(\bm{Q}_i)}$, to estimate the rotor $\widehat{\bm{R}}$, by solving the eigenvalue problem ${\langle L(\bm{R})\rangle_{0,2} = \lambda \bm{R}}$. Then we take ${\bm{S}_i = {\widehat{\bm{R}}}\bm{P}_i{\widehat{\bm{R}}}^\dagger}$ to determine the translation $\hat{\bm{ t}}$ as in (\ref{eq:t:opt:opt:trans:sum:squares}). The general algorithm is described in Alg.~ \ref{alg:rbm:coeffs:approach}.}

\begin{algorithm}[H]
\caption{Rigid Transformation Estimation}
\small
\label{alg:rbm:coeffs:approach}
$\bf 1)$ {\bf Input Data:} $\bm{P}_1,\bm{P}_2,\dots,\bm{P}_m, \bm{Q}_1,\bm{Q}_2,\dots,\bm{Q}_m$\;\\
$\bf 2)$ Extract the coefficients $\mathdutchbcal{P}_{ij} = \mathdutchcal{C}_j(\bm{P}_i)$ and $\mathdutchbcal{Q}_{ij} = \mathdutchcal{C}_j(\bm{Q}_i)$  (eq. ~\eqref{eq:coefficients:CiTRP:CiRTQ})\; \\
$\bf 3)$ Define the multilinear function $L(\bm{R}) = \sum_{i=1}^m \mathdutchbcal{Q}_{i1}^\dagger \bm{R}\mathdutchbcal{P}_{i1} + \mathdutchbcal{Q}_{i1} \bm{R}\mathdutchbcal{P}_{i1}^\dagger$  (eq. ~\eqref{eq:multilinear:function:opt:rot})\; \\
$\bf 4)$ Compute the eigenrotators $\bm{R}_i$ of $\langle L(\bm{R})\rangle_{0,2}$,  with $i=1,2,\dots, 4$\; \\
$\bf 5)$ Chose the eigenrotator $\widehat{\bm{R}} = \bm{R}_{i^*}$ with the largest eigenvalue\; \\
$\bf 6)$ Apply the rotation $\widehat{\bm{R}}$ to $\mathdutchbcal{P}_{ij}$ by taking $\mathdutchbcal{S}_{ij} = \widehat{\bm{R}}\mathdutchbcal{P}_{ij}\widehat{\bm{R}}^\dagger$\; \\
$\bf 7)$ Either estimate the translation as \\$\hat{\bm{t}} = s_N^{-1} \sum_{i=1}^N \langle \left(\mathdutchbcal{S}_{i1}+\mathdutchbcal{Q}_{i1}\right)\left(\mathdutchbcal{Q}_{i3} + \mathdutchbcal{Q}_{i4} - \mathdutchbcal{S}_{i3} - \mathdutchbcal{S}_{i4}\right)^\dagger\rangle_1$\ (eq. (\ref{eq:t:opt:opt:trans:sum:squares})),\ with $s_N=\sum_{i=1}^N \|\mathdutchbcal{S}_{i1}\|^2 + \|\mathdutchbcal{Q}_{i1}\|^2$\; \\
or as \\
$\bf 8)$  $\hat{\bm{t}} = (\mathdutchbcal{Q}_{k3} + \mathdutchbcal{Q}_{k4} - \mathdutchbcal{S}_{k3} - \mathdutchbcal{S}_{k4})\mathdutchbcal{Q}_{k1}^{-1}$, for some $k=1,2,\dots,m$ (eq. \eqref{eq:coeffs:exact:translation})\;
\end{algorithm}

\section{Experimental evaluation}
\vspace{-2mm}
\label{sec:EXP}

We apply the multivector extraction algorithm along with the multivector coefficients algorithm to the task of multivector cloud registration. In this section we first describe the datasets and the metrics used for evaluation (Sec.~\ref{sec:DATA-EVAL}). We then discuss the results with emphasis to the presence  of noise and comparing our approach with other related methods (Sec.~\ref{sec:RESULTS} and Sec.~\ref{sec:RESULTS-jacinto}).

\subsection{Datasets and evaluation}
\label{sec:DATA-EVAL}

We used the Stanford 3D Scanning Repository~\cite{Stanford} that contains a variety of data models and includes theirs full and partial overlaps with the corresponding ground truth transformations. Our experimental data mainly contains three models ``Bunny'', ``Armadillo'' and ``Dragon''.
The evaluation is  based on the  following two metrics: $(i)$ Relative Translation Error, ${\rm RTE} = \left \| \hat{\bm{t}} - {\bm{t}} \right \|$, where  $\hat{\bm{t}}$ is the estimated translation and ${\bm{t}}$ is the ground truth, and $(ii)$ Relative Rotation Error  ${\rm RRE} = 2\cos^{-1}\left(\widehat{\bm{R}}\!*\!\bm{R}^\dagger\right)$, where $\widehat{\bm{R}}$ and $\bm{R}$ are the estimated and ground truth rotations, respectively.

\subsection{Experimental Setup}
\label{sec:RESULTS}

In our experimental setup we considered target and source point clouds described by the two set of vectors,  $\bm{x}_1',\bm{x}_2',\cdots,\bm{x}_\ell'\in\mathcal{A}_3$ and $\bm{y}_1',\bm{y}_2',\cdots,\bm{y}_\ell'\in\mathcal{A}_3$, respectively. We assume the points are related as $\bm{y}_i' = R(\bm{x}_{j_i}') + \bm{t} + \bm{n}_i$ where $\bm{n}_i$ is uncorrelated Gaussian noise with variance $\sigma^2$.
Since we are aiming to solve the registration problem in CGA, we have to extend the points as $\bm{x}_i \equiv c(\bm{x}_i')$ and $\bm{y}_i \equiv c(\bm{y}_i')$ where $c$ is the conformal mapping defined by~\eqref{eq:the:conformal:mapping} (also, for a detailed description of this extension see { Appendix}~\ref{sec:ext:euc:CGA}). We then proceed to use  Alg.~\ref{alg:gen:app}, setting $\bm{X}_k\equiv \bm{x}_k$ and $\bm{Y}_k \equiv \bm{y}_k$, to estimate the rotation and translation in a correspondence free manner. In particular we use Alg.~\ref{alg:eigmv:extraction} to extract the eigenmultivectors associated with the points $\bm{x}_i$ and $\bm{y}_i$ and then the coefficients algorithm described in Alg.~\ref{alg:rbm:coeffs:approach}, where the input are the eigenmultivectors associated with $\bm{x}_i$ and $\bm{y}_i$. We will term this algorithm as \textbf{\textit{CGA-EVD}}, (CGA - EigenValueDecomposition).

The first methods to be compared, are based on PCA algorithms~\cite{celik2008fast,gonzalez2009digital,rehman2018automatic}. In PCA based approaches the translation is estimated using the center of mass, and the rotation is estimated from the decomposition of ${f(\bm{z}) = \sum_{k=1}^\ell \bm{z}\cdot(\bm{x}_k' - \bar{\bm{x}}')(\bm{x}_k' - \bar{\bm{x}}')}$ and ${g(\bm{z}) = \sum_{k=1}^\ell \bm{z}\cdot(\bm{y}_k' - \bar{\bm{y}}')(\bm{y}_k' - \bar{\bm{y}}')}$ where ${\bar{\bm{x}}' = \tfrac{1}{\ell}\sum_{k=1}^\ell \bm{x}_k'}$ and ${\bar{\bm{y}}' = \tfrac{1}{\ell}\sum_{k=1}^\ell \bm{y}_k'}$ are the centers of mass. 
Since we are using Vanilla Geometric Algebra to implement the PCA algorithm. We wil term this algorithm as \textit{\textbf{VGA-EVD}}. 

To benchmark the performance of our algorithm we provide results with increasing levels of noise. Specifically, Figs.~\ref{fig:armadillo:sigma:0},~\ref{fig:armadillo:sigma:0:001},~\ref{fig:armadillo:sigma:0:005},~\ref{fig:armadillo:sigma:0:01} and~\ref{fig:armadillo:sigma:0:02} show the Stanford Armadillo dataset corrupted with Gaussian noise with 
$\sigma\in\{0.001,0.002,0.005,0.01,0.02\}$, respectively. These figures helped us to find appropriate levels of noise to be applied to the datasets. We choose  $\sigma \leqslant 0.01$, because for larger values the clouds become very distorted, losing their shape (Fig.~\ref{fig:armadillo:sigma:0:01}).

We extend our experimental evaluation including the following five algorithms: $(i)$ ICP~\cite{besl1992method}, $(ii)$
DCP~\cite{wang2019deep}, $(iii)$ Go-ICP~\cite{yang2015go}, $(iv)$ TEASER++\cite{yang2020teaser}, and $(v)$ PASTA~\cite{marchi2023sharp}. 
The first three algorithms are very well known in the literature, while  TEASER++\cite{yang2020teaser} 
 uses graduated non-convexity (GNC) ~\cite{yang2020graduated} to estimate the rotation without solving a large semidefinite programming (SDP). It 
is proposed to deal with large amounts of outlier correspondences. PASTA~\cite{marchi2023sharp}, is originally proposed for estimating the rigid transformation that relate the robot’s current pose to the fixed reference frame from the LiDAR measurements.
See Appendix \ref{sec:code:details} for the implementation  details used in the experiments.

\subsection{Results}

\label{sec:RESULTS-jacinto}

Table~\ref{tab:All-jan} shows a comparison evaluation for the seven methods mentioned in Sec.~\ref{sec:RESULTS} and for all the datasets. The results are obtained for the highest noise level, {\it i.e.} $\sigma=0.01$, considering a rotation of $\theta=5^\circ$ and translation of $\|\bm{t}\|=0.01$. A detailed assessment, comprising multiple levels of noise are shown in the Tab.~\ref{tab:Bunny-jan},~\ref{tab:Armadillo-jan} and~\ref{tab:Dragon-jan} ($\theta=5^\circ$, $\|\bm{t}\|=0.01$). Tab.~\ref{tab:Bunny-jan1},~\ref{tab:Armadillo-jan1} and~\ref{tab:Dragon-jan1}, report the results in a more difficult scenario, {\it i.e.} $\theta\in[0^\circ,360^\circ]$, $\|\bm{t}\|=1$. From the results it is seen that the proposed \textit{\textbf{CGA-EVD}} is on par with \textit{\textbf{VGA-EVD}}. Notice that, the marginal superiority of the \textit{\textbf{VGA-EVD}} is due to the fact that the \textit{\textbf{CGA-EVD}} mapping 
distorts the noise, causing it to have a generalized chi-squared distribution (see {\it Theorem}~\ref{theo:noise:cga:embedding}). Note that both the  \textit{\textbf{CGA-EVD}} and \textit{\textbf{VGA-EVD}} algorithms are insensitive to the initial pose of the point clouds as their performances have negligible variation  within setups.
 
Figures~\ref{fig:benchmark:small:trans} and~\ref{fig:benchmark:big:trans} compare all methods showing the registration performance in terms of the ${\rm RTE}$ and ${\rm RRE}$ metrics, for increasing values of noise. Fig.~\ref{fig:benchmark:small:trans} considers the setup $\theta=5^\circ$, $\|\bm{t}\|=0.01$, while Fig.~\ref{fig:benchmark:big:trans} considers the setup $\theta\in[0^\circ,360^\circ]$, $\|\bm{t}\|=1$. It is clear that our method is robust to all values in the  rotation range while local methods, {\it i.e.} ICP and Go-ICP, only perform well for small rotation angles and translation magnitudes. Notice, however, that even for a small transformation setup ($\theta=5^\circ$, $\|\bm{t}\|=0.01$), the performance of these local methods degrade with the  increase of the noise, since the erroneous correspondences will increase as well.
Concerning the  \textit{\textbf{VGA-EVD}}, it is shown that it achieves comparable results. 
The TEASER++\cite{yang2020teaser} approach decreases its performance for higher levels of noise (see Fig.~\ref{fig:benchmark:small:trans} for the Armadillo and Dragon datasets), as this approach gives more emphasis on the presence of outliers and not for high levels of noise. Concerning DCP, it is shown that this method  provides less resilience to noise (see Figs.~\ref{fig:benchmark:small:trans},~\ref{fig:benchmark:big:trans}), this might be due to the fact that the 3D-Match dataset~\cite{zeng20173dmatch} was used in the pre-training stage.  

The inference time (in seconds) is reported in Tab.~\ref{tab:times}. Although our proposal is competitive, the PASTA~\cite{marchi2023sharp} exhibits the overall best results.


\begin{table*}
\small
\caption{Performance analysis for all methodologies using three datasets and for the highest level of noise $\sigma=0.01$. Benchmark results for rotation angle of $\theta=5^\circ$ and translation magnitude $\|\bm{t}\| = 0.01$. The  values refer to the \textit{mean} over $10$ iterations for the same level of noise.}
\label{tab:All-jan}
\centering
\renewcommand{\arraystretch}{1.2}
\scalebox{0.83}{
\begin{tabular}{|l|l|l|l|l|l|l|}
\hline

{Algorithm} & \multicolumn{2}{|c|}{Bunny} & \multicolumn{2}{|c|}{Armadillo} & \multicolumn{2}{|c|}{Dragon} \\ \hline
 & RRE (\textdegree) & RTE (m) & RRE (\textdegree) & RTE (m) & RRE (\textdegree) & RTE (m)\\ \hline
\textbf{\textit{VGA-EVD}}  & \num{7.620e-01} & \num{9.245e-04} & \num{2.889e-01} & \num{3.923e-04}    & {\num{1.485e-01}} & {\bf\num{4.319e-04}}  \\ \hline
\textbf{\textit{CGA-EVD}}  & {\bf\num{7.381e-01}}  & {\bf\num{ 9.018e-04}} & {\bf \num{2.747e-01}} & {\bf \num{3.537e-04}} & {\bf\num{1.462e-01}} & \num{4.356e-04}  \\ \hline
\textbf{\textit{ICP}}~\cite{besl1992method}     &\num{2.994e+00} & \num{4.829e-03} & \num{1.697e+00} & \num{4.779e-03}     &  \num{2.919e+00} & \num{7.262e-03}   \\ \hline
\textbf{\textit{PASTA-3D}}~\cite{marchi2023sharp} & \num{8.021e+00} & \num{1.031e-02} & \num{1.228e+01} & \num{2.580e-02}    &  \num{3.085e+00} & \num{7.459e-03}   \\ \hline
\textbf{\textit{DCP}}~\cite{wang2019deep}      & \num{2.374e+01} & \num{2.882e-02} & \num{6.820e+01} & \num{8.742e-02}    &  \num{5.306e+01} & \num{1.009e-01}   \\ \hline
\textbf{\textit{Go-ICP}}~\cite{yang2015go}   & \num{3.553e+00} & \num{6.434e-03} & \num{2.897e+00} & \num{7.133e-03}    &  \num{3.911e+00} & \num{9.215e-03}   \\ \hline
\textbf{\textit{TEASER++}}~\cite{yang2020teaser} & \num{2.163e+01} & \num{3.169e-02} & \num{8.149e+00} & \num{1.385e-02}    &  \num{1.119e+02} & \num{1.605e-01}   \\ \hline
\end{tabular}
}
\end{table*}

Figures~\ref{fig:armadillo:with:primitives} and~\ref{fig:bunny:with:primitives} show another important information provided by the CGA: the primitives. The primitives shown express the geometrical representation of the eigenmultivectors $\bm{P}_i$, $i=1,...,m$, obtained by the extraction algorithm in Alg.~\ref{alg:eigmv:extraction}. In particular, the circles shown in Figs.~\ref{fig:armadillo:circles} and~\ref{fig:bunny:circles} express the eigenbivectors which are well represented by dual circles and dual imaginary circles. The spheres in Figs.~\ref{fig:armadillo:spheres} and~\ref{fig:bunny:spheres} are the eigenvectors, which can be represented by dual spheres and dual imaginary spheres (for a more detailed description we refer to~\cite[Chapter 14]{dorst_geometric_2007}). Figs.~\ref{fig:armadillo:arrows} and~\ref{fig:bunny:arrows} display arrows which represent the translation invariant components of the circles that are used to estimate the rotation between the point clouds.

Figure~\ref{fig:noisy:stanford:with:primitives:unregistered} shows all datasets with non-registered noisy input point clouds, and their respective circle primitives, {\it i.e.} the bivectors $\langle \bm{P}_i\rangle_2$ and $\langle \bm{Q}_i\rangle_2$,  expressed as circles, computed from the eigendecomposition of $F$ and $G$, respectively (see eq.~(\ref{eq:covariance:functions:general})). It also illustrates the equivariant properties of the eigenmultivectors. Specifically, when the points suffer a rotation and a translation,  the corresponding primitives will also suffer the same transformation.
Fig.~\ref{fig:noisy:stanford:with:primitives:registered}  illustrates the resulted alignment after applying our registration algorithm along with the primitives of the respective point clouds. It is seen that, although difficult to visualize the registration due to the high levels of noise, the primitives indicate a good insight of the resulted alignment. Whereby we can better visualize that the registration was successful by looking at the primitives of both point clouds. 

Figures~\ref{fig:bunny:reg},~\ref{fig:armadillo:reg} and 
~\ref{fig:dragon:reg}, show, for a level of noise $\sigma=0.01$, the registration results for Stanford Bunny, Armadillo, Dragon, respectively. It is shown the initial displacement of the two noisy point clouds (in {\bf (a)}), with the registration results (in {\bf (c)}). For better visualization purposes, it is shown the initialization  (in {\bf (b)}) and the registration results (in {\bf (d)}) if the noise were removed. These figures also report the high quantitative performance of the registration as our framework provides low ${\rm RRE}$ and ${\rm RTE}$ scores.


\begin{figure*}[htb]
\centering
\includegraphics[width=0.99\linewidth]{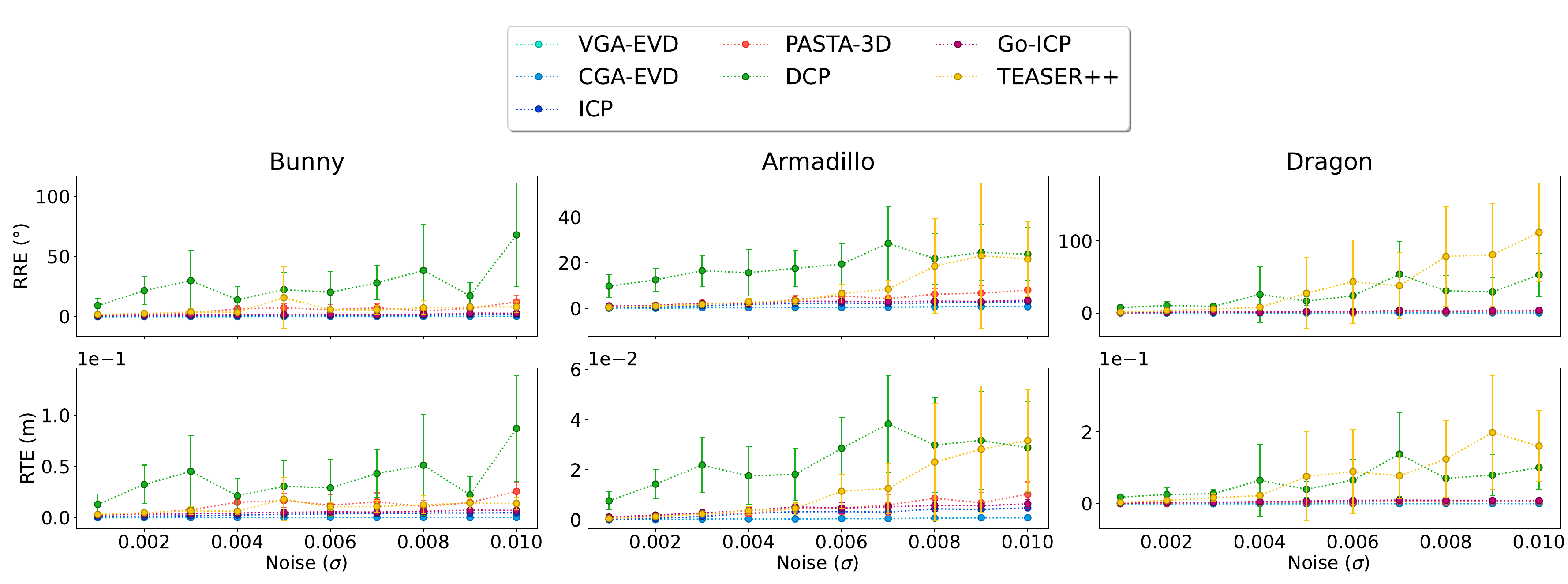}
\caption{\small Benchmark results for rotation angle of $\theta=5^\circ$ and translation magnitude $\|\bm{t}\| = 0.01$ of the \textit{\textbf{VGA-EVD}}, \textit{\textbf {CGA-EVD}}, \textit{\textbf{ICP}}~\protect\cite{besl1992method}, \textit{\textbf{PASTA-3D}}~\protect\cite{marchi2023sharp}, \textit{\textbf{DCP}}~\protect\cite{wang2019deep}, \textit{\textbf{Go-ICP}}~\protect\cite{yang2015go}, \textit{\textbf{TEASER++}}~\protect\cite{yang2020teaser} methods for the Bunny, Armadillo and Dragon datasets. Error in the rotation (top row), and  translation (bottom row).}
\label{fig:benchmark:small:trans}
\end{figure*}

\subsection{Illustration of the \textit{\textbf{CGA-EVD}} results}

In the following figures and tables we illustrate supplementary results (mentioned in the main paper) obtained in the Stanford 3D Scanning Repository~\cite{Stanford}.

\begin{itemize}
\item 
Figure~\ref{fig:armadillo:increased:noise} shows the Armadillo dataset for different values of $\sigma$. It is seen that for $\sigma=0.02$ the shape is highly distorted, making the choice of this parameter set to $\sigma\leq 0.01$.
\item 
Figures~\ref{fig:armadillo:with:primitives} and~\ref{fig:bunny:with:primitives} show the CGA primitives expressing the geometrical representation of the eigenmultivectors obtained by the extraction algorithm in Alg.~\ref{alg:eigmv:extraction}.
In these figures, no distinction is made between dual imaginary spheres and dual real spheres or between dual imaginary circles and dual real circles. Imaginary circles and imaginary spheres are circles and spheres with negative radius, they are drawn as if they would have positive radius.
\item Figure~\ref{fig:noisy:stanford:with:primitives:unregistered} shows all datasets with non-registered noisy input point clouds, and their respective circle primitives, that is, the bivectors, expressed as circles and computed from the eigendecomposition of $F$ and $G$, respectively (see eq.~(\ref{eq:covariance:functions:general})). 
Fig.~\ref{fig:noisy:stanford:with:primitives:registered}  illustrates the resulted alignment after applying our registration algorithm along with the primitives of the respective point clouds. 
\item Figures~\ref{fig:bunny:reg},~\ref{fig:armadillo:reg} and~\ref{fig:dragon:reg}, show, for a level of noise $\sigma=0.01$, the registration results for Stanford Bunny, Armadillo, Dragon, respectively.
\end{itemize}

\begin{figure}[H]
  \centering
     \begin{subfigure}{0.19\linewidth} \centering
       \includegraphics[width=0.95\linewidth,frame]{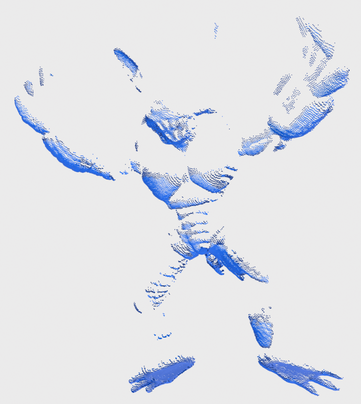}
       \caption{} \label{fig:armadillo:sigma:0}
     \end{subfigure}
     \begin{subfigure}{0.19\linewidth} \centering
     \includegraphics[width=0.95\linewidth,frame]{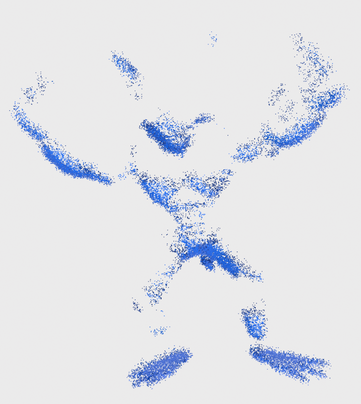}
       \caption{} \label{fig:armadillo:sigma:0:001}
     \end{subfigure}
     \begin{subfigure}{0.19\linewidth} \centering
     \includegraphics[width=0.95\linewidth,frame]{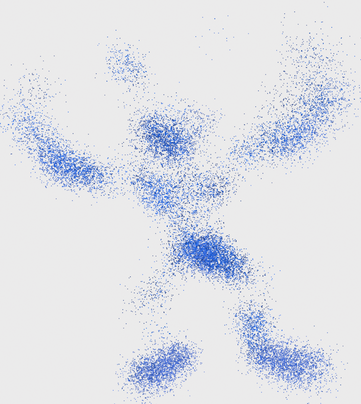}
       \caption{} \label{fig:armadillo:sigma:0:005}
     \end{subfigure}
     \begin{subfigure}{0.19\linewidth} \centering
     \includegraphics[width=0.95\linewidth,frame]{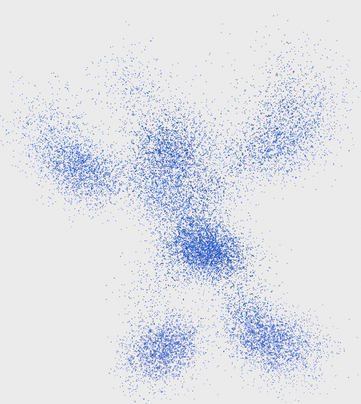}
       \caption{} \label{fig:armadillo:sigma:0:01}
     \end{subfigure}
     \begin{subfigure}{0.19\linewidth} \centering
     \includegraphics[width=0.95\linewidth,frame]{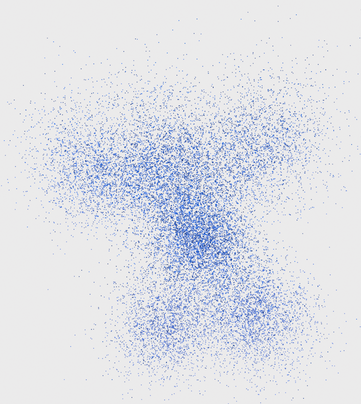}
     \caption{} \label{fig:armadillo:sigma:0:02}
     \end{subfigure}
  \caption{The Stanford Armadillo dataset for increasing  levels of Gaussian Noise, {\it i.e.} $\sigma\in\{0, 0.001, 0.005, 0.01, 0.02\}$} \label{fig:armadillo:increased:noise}
  \end{figure}

\begin{figure}[h]
\centering
   \begin{subfigure}{0.325\linewidth} \centering
     \includegraphics[width=0.95\linewidth,frame]{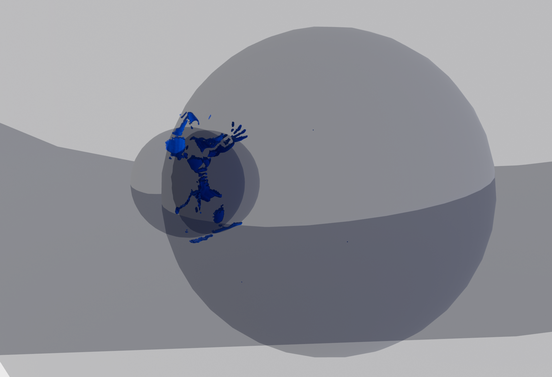}
     \caption{$\langle\bm{P}_i\rangle_1$} \label{fig:armadillo:spheres}
   \end{subfigure}
   \begin{subfigure}{0.325\linewidth} \centering
   \includegraphics[width=0.95\linewidth,frame]{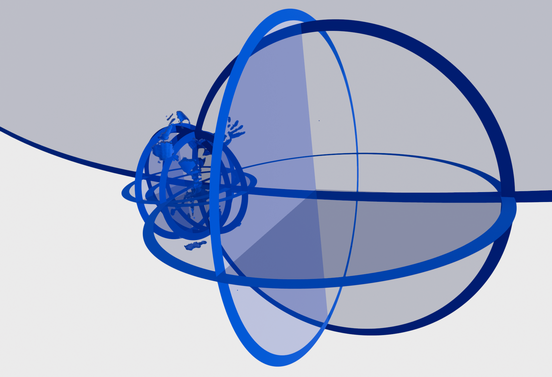}
     \caption{$\langle\bm{P}_i\rangle_2$} \label{fig:armadillo:circles}
   \end{subfigure}
   \begin{subfigure}{0.325\linewidth} \centering
   \includegraphics[width=0.95\linewidth,frame]{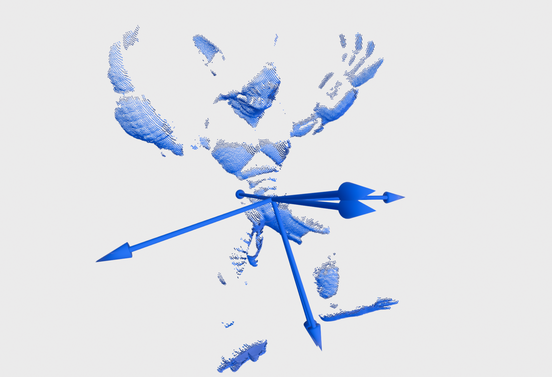}
     \caption{$\langle\mathdutchcal{C}_1(\bm{P}_i)\rangle_1$} \label{fig:armadillo:arrows}
   \end{subfigure}
\caption{The Stanford Armadillo dataset and its respective primitives: (a) eigenvectors, (b) eigenbivectors, (c) translation invariant vectors extracted from the eigenbivectors.} \label{fig:armadillo:with:primitives}
\end{figure}

\begin{figure}[h]
\centering
   \begin{subfigure}{0.325\linewidth} \centering
     \includegraphics[width=0.95\linewidth,frame]{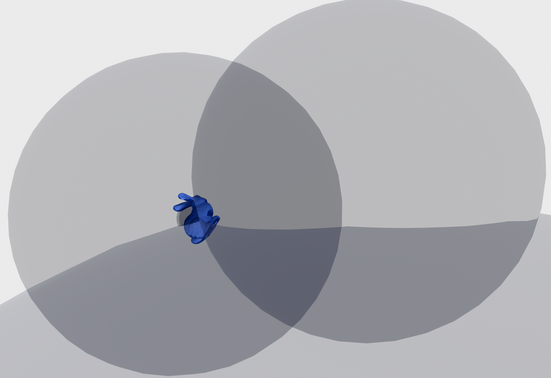}
     \caption{$\langle\bm{P}_i\rangle_1$} \label{fig:bunny:spheres}
   \end{subfigure}
   \begin{subfigure}{0.325\linewidth} \centering
   \includegraphics[width=0.95\linewidth,frame]{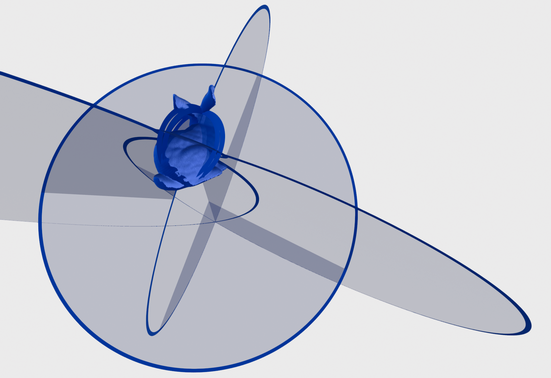}
     \caption{$\langle\bm{P}_i\rangle_2$} \label{fig:bunny:circles}
   \end{subfigure}
   \begin{subfigure}{0.325\linewidth} \centering
   \includegraphics[width=0.95\linewidth,frame]{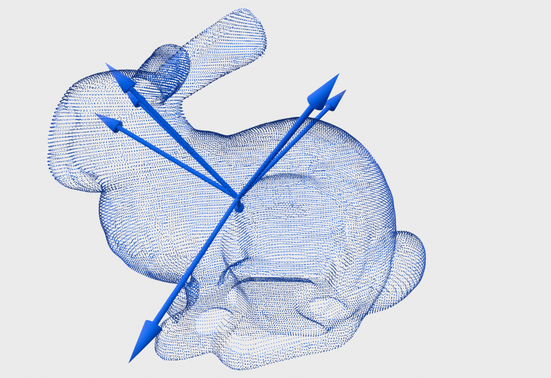}
     \caption{$\langle\mathdutchcal{C}_1(\bm{P}_i)\rangle_1$} \label{fig:bunny:arrows}
   \end{subfigure}
\caption[The Stanford Bunny dataset and its respective primitives]{The Stanford Bunny dataset and its respective primitives: (a) eigenvectors, (b) eigenbivectors, (c) translation invariant vectors extracted from the eigenbivectors.} \label{fig:bunny:with:primitives}
\end{figure}

\begin{figure}[H]
\centering
    \begin{subfigure}{0.325\linewidth} \centering
   \includegraphics[width=0.95\linewidth,frame]{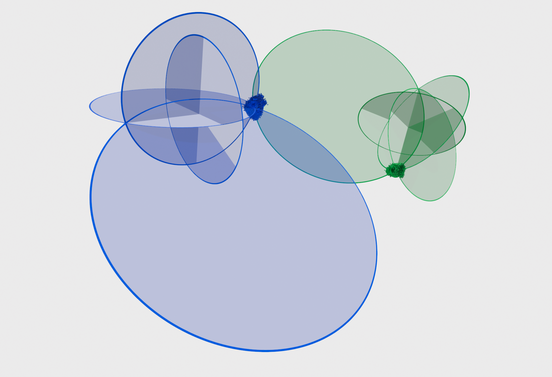}
     \caption{Bunny} \label{fig:bunny:unregistered}
   \end{subfigure}
   \begin{subfigure}{0.325\linewidth} \centering
   \includegraphics[width=0.95\linewidth,frame]{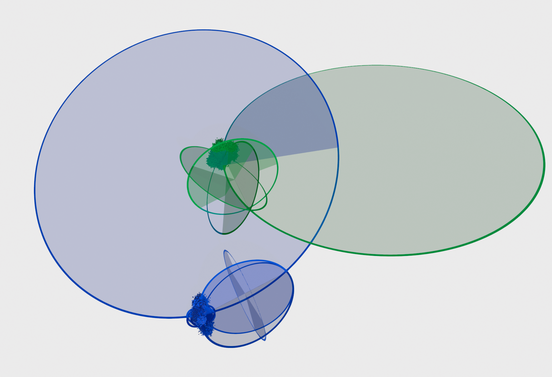}
     \caption{Dragon} \label{fig:dragon:unregistered}
   \end{subfigure}
   \begin{subfigure}{0.325\linewidth} \centering
   \includegraphics[width=0.95\linewidth,frame]{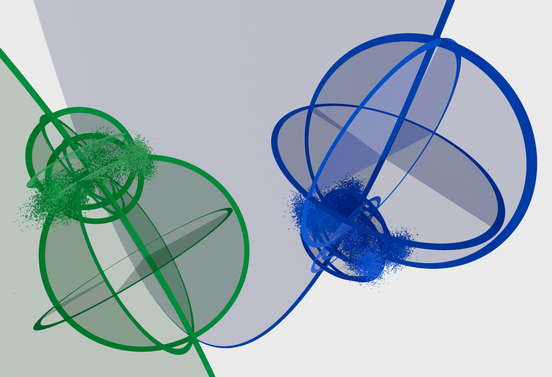}
     \caption{Armadillo} \label{fig:armadillo:unregistered}
   \end{subfigure}
\caption{Multiple non-registered point clouds of the Stanford dataset with added Gaussian noise, $\sigma=0.01$, and their respective circle primitives.} \label{fig:noisy:stanford:with:primitives:unregistered}
\end{figure}

\begin{figure}[H]
\centering
    \begin{subfigure}{0.325\linewidth} \centering
   \includegraphics[width=0.95\linewidth,frame]{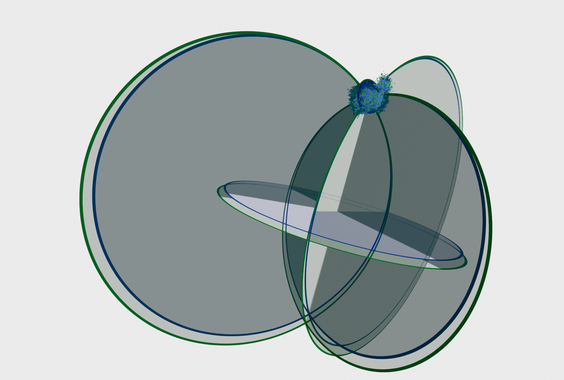}
     \caption{Bunny} \label{fig:bunny:registered}
   \end{subfigure}
   \begin{subfigure}{0.325\linewidth} \centering
   \includegraphics[width=0.95\linewidth,frame]{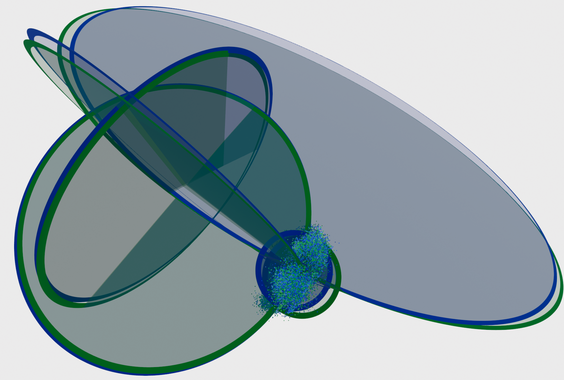}
     \caption{Dragon} \label{fig:dragon:registered}
   \end{subfigure}
   \begin{subfigure}{0.325\linewidth} \centering
   \includegraphics[width=0.95\linewidth,frame]{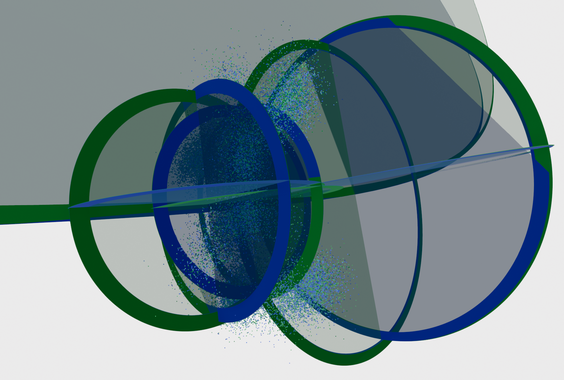}
     \caption{Armadillo} \label{fig:armadillo:registered}
   \end{subfigure}
\caption{Multiple registered point clouds of the Stanford dataset with added Gaussian noise, ${\sigma=0.01}$, and their respective circle primitives.} \label{fig:noisy:stanford:with:primitives:registered}
\end{figure}

\begin{figure}[H]
\centering
   \begin{subfigure}{0.24\linewidth} \centering
     \includegraphics[width=0.95\linewidth,frame]{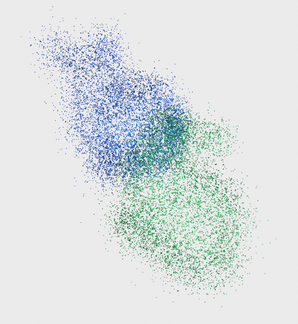}
     \caption{} \label{fig:bunny:noisy:unreg}
   \end{subfigure}
   \begin{subfigure}{0.24\linewidth} \centering
   \includegraphics[width=0.95\linewidth,frame]{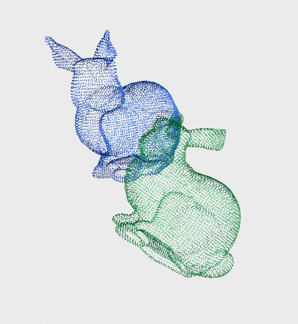}
   \caption{} \label{fig:bunny:not:noisy:unreg}
   \end{subfigure}
   \begin{subfigure}{0.24\linewidth} \centering
   \includegraphics[width=0.95\linewidth,frame]{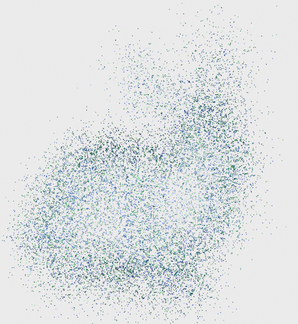}
   \caption{} \label{fig:bunny:noisy:reg}
   \end{subfigure}
   \begin{subfigure}{0.24\linewidth} \centering
   \includegraphics[width=0.95\linewidth,frame]{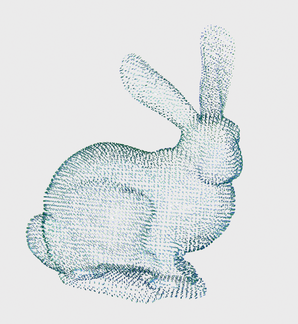}
   \caption{} \label{fig:bunny:not:noisy:reg}
   \end{subfigure}
\caption{Registration of the Stanford Bunny. {\bf (a)} Unregistered noisy point clouds with $\sigma=0.01$. {\bf (b)} Unregistered point clouds with noise removed. {\bf (c)} Registered noisy point clouds with $\sigma=0.01$. {\bf (d)} Registered point clouds with noise removed. The obtained  metric scores were ${\rm RRE}={\num{9.659e-1}}^\circ$, ${\rm RTE}=\num{9.528e-1}m$.} \label{fig:bunny:reg}
\end{figure}

\vspace{-5mm}   
\begin{figure}[H]
\centering
   \begin{subfigure}{0.24\linewidth} \centering
     \includegraphics[width=0.95\linewidth,frame]{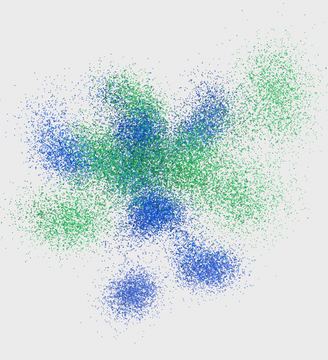}
     \caption{} \label{fig:armadillo:noisy:unreg}
   \end{subfigure}
   \begin{subfigure}{0.24\linewidth} \centering
   \includegraphics[width=0.95\linewidth,frame]{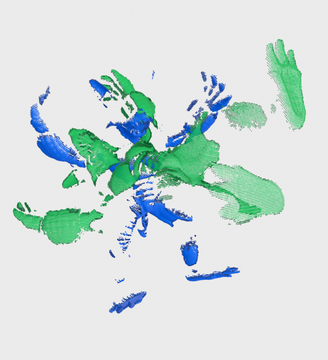}
   \caption{} \label{fig:armadillo:not:noisy:unreg}
   \end{subfigure}
   \begin{subfigure}{0.24\linewidth} \centering
   \includegraphics[width=0.95\linewidth,frame]{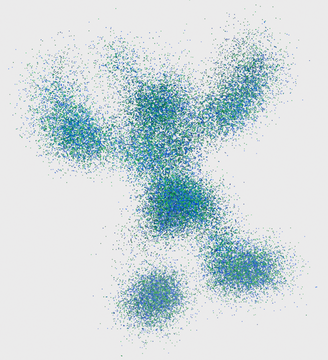}
   \caption{} \label{fig:armadillo:noisy:reg}
   \end{subfigure}
   \begin{subfigure}{0.24\linewidth} \centering
   \includegraphics[width=0.95\linewidth,frame]{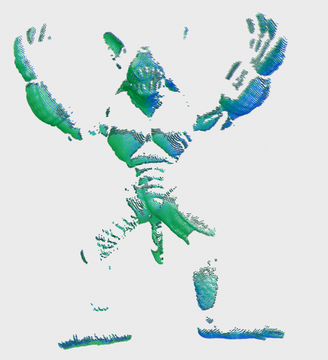}
   \caption{} \label{fig:armadillo:not:noisy:reg}
   \end{subfigure}
\caption{Registration of the Stanford Armadillo. {\bf (a)} Unregistered noisy point clouds with $\sigma=0.01$. {\bf (b)} Unregistered point clouds with noise removed. {\bf (c)} Registered noisy point clouds with $\sigma=0.01$. {\bf (d)} Registered point clouds with noise removed.  The obtained  metric scores were ${\rm RRE}={\num{2.67e-1}}^\circ$, ${\rm RTE}=\num{1.932e-3}m$.} \label{fig:armadillo:reg}
\end{figure}

\begin{figure}[H]
\centering
   \begin{subfigure}{0.24\linewidth} \centering
     \includegraphics[width=0.95\linewidth,frame]{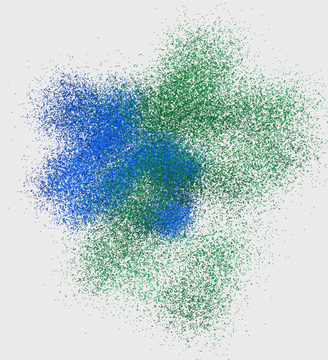}
     \caption{} \label{fig:dragon:noisy:unreg}
   \end{subfigure}
   \begin{subfigure}{0.24\linewidth} \centering
   \includegraphics[width=0.95\linewidth,frame]{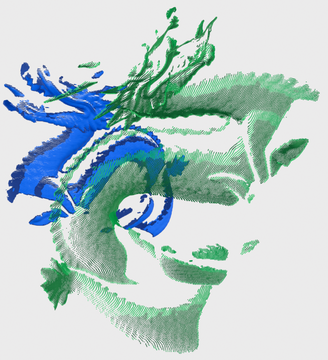}
   \caption{} \label{fig:dragon:not:noisy:unreg}
   \end{subfigure}
   \begin{subfigure}{0.24\linewidth} \centering
   \includegraphics[width=0.95\linewidth,frame]{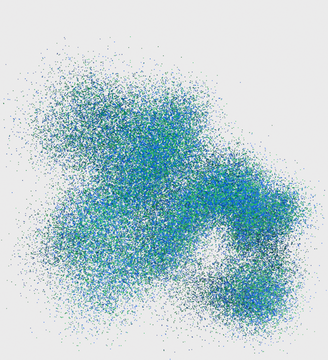}
   \caption{} \label{fig:dragon:noisy:reg}
   \end{subfigure}
   \begin{subfigure}{0.24\linewidth} \centering
   \includegraphics[width=0.95\linewidth,frame]{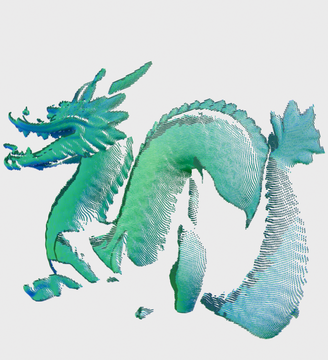}
   \caption{} \label{fig:dragon:not:noisy:reg}
   \end{subfigure}
\caption{Registration of the Stanford Dragon. {\bf (a)} Unregistered noisy point clouds with $\sigma=0.01$. {\bf (b)} Unregistered point clouds with noise removed. {\bf (c)} Registered noisy point clouds with $\sigma=0.01$. {\bf (d)} Registered point clouds with noise removed. The obtained  metric scores were ${\rm RRE}={\num{3.21e-1}}^\circ$, ${\rm RTE}=\num{3.660e-3}m$.} \label{fig:dragon:reg}
\end{figure}

\subsection{Comparison with other related point cloud registration methods}
\label{sec:code:details}

We provide a quantitative comparison of the methodologies and we detail some minor implementation details of the algorithms.  

The comparative results are framed into the following experimental setups:
\begin{itemize}
\item Tables~\ref{tab:Bunny-jan},~\ref{tab:Armadillo-jan},~\ref{tab:Dragon-jan} report the results for the setup with $\theta=5^\circ$, $\|\bm{t}\| = 0.01$
\item Tables~\ref{tab:Bunny-jan1},~\ref{tab:Armadillo-jan1},~\ref{tab:Dragon-jan1} report the results for the setup with $\theta\in[0^\circ,360^\circ]$, $\|\bm{t}\|=1$
\item From the Tab.~\ref{tab:Bunny-jan} to Tab.~\ref{tab:Dragon-jan1}, it is shown that the \textit{\textbf{VGA-EVD}} has a marginal superiority comparing to \textit{\textbf{CGA-EVD}}. This marginal superiority is due to the fact that the \textit{\textbf{CGA-EVD}} mapping distorts the noise, causing it to have a generalized chi squared distribution (see Theorem~\ref{theo:noise:cga:embedding}). It also applies a non linear operation to the input points $\bm{x}_i'$ and $\bm{y}_i'$, which, because of accumulated error due to floating point precision, degrades performance.
\end{itemize}

Concerning the code details of the methodologies used in Sec.~\ref{sec:EXP} and also for the next experiments, they are as follows: 

\begin{itemize}
\item For the ICP~\cite{besl1992method}, we use the Open3D library. Concerning the hyper-parameter, we use a convergence criteria relative fitness of \num{1e-06} and relative RMSE  of \num{1e-06}. The number of iterations is $30$. These hyperparameters are used for all the datasets.

\item For Go-ICP~\cite{yang2015go}, we use the python library available form the authors. The Go-ICP method is
based on a branch-and-bound (BnB) strategy, whose threshold is set to \num{1e-3}$N$, with $N$, the number of data points. 

\item In  PASTA~\cite{marchi2023sharp} although  is originally proposed for 2D scenario, we use a version provided from the authors that is appropriate for the 3D setting. As in our approach there is no need of having hyperparameters. Essentially, the convex hull is estimated  for each point cloud in the first stage. This is  followed by the computation of the first and second order moments and the computation of the eigenvectors for the second moment.

\item For TEASER++\cite{yang2020teaser}, we have used the FPFH algorithm (provided by {\rm Open3D}) to extract features and find correspondences, by which we then feed the point pair correspondences to the TEASER++ algorithm. We have also changed some of the hyperparameters, namely for the Stanford Bunny and Armadillo dataset. Concretely, we set the noise bound to $0.01$ for Bunny and Armadillo datasets and for the Dragon dataset we set it to $0.02$, as these values provided the best accuracy.
For the Graduated Non-Convex-Geman McClure (GNC-MG) we follow~\cite{yang2020graduated}, and used a fix annealing factor of 1.4, for all the datasets.

\item For DCP~\cite{wang2019deep}, we used the DCP-v2 version with the standard architecture. In concrete, we use five Dynamic Graph CNN layers. The number of filters for each layer are respectively given by 
[64, 64, 128, 256, 512]. 
For the  Transformer layer, the number
of heads in multi-head attention is four with the embedding
dimension of 1024. We use LayerNorm (without Dropout) with Adam optimizer and with an initial learning rate of 0.001. A  total of
250 epochs were used for training. The 3D Match dataset was used for the pretraining stage. Also, because of GPU memory issues, we applied an every $k$ points sampling before feeding the points to the algorithm. For the Bunny dataset we find that there was no need to sampling, while for the Dragon and Armadillo we set $k=10$ and $k=5$, respectively. 
\end{itemize}

\begin{table*}[htb]
\caption{Benchmark results for rotation angle of $\theta=5^\circ$ and translation magnitude $\|\bm{t}\| = 0.01$ for the Stanford Bunny dataset for various levels of noise. The values report the \textit{mean} over $10$ iterations for the same level of noise.}
\label{tab:Bunny-jan}
\centering
\renewcommand{\arraystretch}{1.2}
\setlength{\tabcolsep}{4pt}
\scalebox{0.65}{
\begin{tabular}{|l|l|l|l|l|l|l|l|l|}
\hline
Algorithm & \multicolumn{2}{|c|}{$\sigma=0.001$} & \multicolumn{2}{|c|}{$\sigma=0.002$} & \multicolumn{2}{|c|}{$\sigma=0.005$} & \multicolumn{2}{|c|}{$\sigma=0.01$}\\ \hline
 & RRE (\textdegree) & RTE (m) & RRE (\textdegree) & RTE (m) & RRE (\textdegree) & RTE (m) & RRE (\textdegree) & RTE (m)\\ \hline
\textbf{\textit{VGA-EVD}} & \num{9.594e-02} & {\bf\num{1.076e-04}} & \num{1.368e-01} & {\bf\num{1.574e-04}} & \num{4.196e-01} & \num{4.299e-04} & \num{7.620e-01} & \num{9.245e-04}\\ \hline
\textbf{\textit{CGA-EVD}} & {\bf\num{9.391e-02}} & \num{1.113e-04} & {\bf\num{1.288e-01}} & \num{1.600e-04} & {\bf\num{4.147e-01}} & {\bf\num{4.207e-04}} & {\bf\num{7.381e-01}} & {\bf\num{9.018e-04}}\\ \hline
\textbf{\textit{ICP}}~\cite{besl1992method} & \num{2.727e-01} & \num{2.686e-04} & \num{5.247e-01} & \num{6.764e-04} & \num{2.170e+00} & \num{3.297e-03} & \num{2.994e+00} & \num{4.829e-03}\\ \hline
\textbf{\textit{PASTA-3D}}~\cite{marchi2023sharp}& \num{6.178e-01} & \num{6.754e-04} & \num{1.371e+00} & \num{1.603e-03} & \num{3.830e+00} & \num{4.685e-03} & \num{8.021e+00} & \num{1.031e-02}\\ \hline
\textbf{\textit{DCP}}~\cite{wang2019deep} & \num{9.770e+00} & \num{7.657e-03} & \num{1.252e+01} & \num{1.432e-02} & \num{1.756e+01} & \num{1.821e-02} & \num{2.374e+01} & \num{2.882e-02}\\ \hline
\textbf{\textit{Go-ICP}}~\cite{yang2015go} & \num{1.086e+00} & \num{1.201e-03} & \num{1.227e+00} & \num{1.961e-03} & \num{3.110e+00} & \num{5.239e-03} & \num{3.553e+00} & \num{6.434e-03}\\ \hline
\textbf{\textit{TEASER++}}~\cite{yang2020teaser} & \num{5.021e-01} & \num{5.403e-04} & \num{1.045e+00} & \num{1.345e-03} & \num{3.252e+00} & \num{4.526e-03} & \num{2.163e+01} & \num{3.169e-02}\\ \hline
\end{tabular}}
\hspace{1pt}
\end{table*}

\begin{table*}[h]
\caption{Benchmark results for rotation angle of $\theta=5^\circ$ and translation magnitude $\|\bm{t}\| = 0.01$ for the Stanford Armadillo dataset for various levels of noise. The values report the \textit{mean} over $10$ iterations for the same level of noise.}
\label{tab:Armadillo-jan}
\centering
\renewcommand{\arraystretch}{1.2}
\setlength{\tabcolsep}{4pt}
\scalebox{0.65}{
\begin{tabular}{|l|l|l|l|l|l|l|l|l|}
\hline
Algorithm & \multicolumn{2}{|c|}{$\sigma=0.001$} & \multicolumn{2}{|c|}{$\sigma=0.002$} & \multicolumn{2}{|c|}{$\sigma=0.005$} & \multicolumn{2}{|c|}{$\sigma=0.01$}\\ \hline
 & RRE (\textdegree) & RTE (m) & RRE (\textdegree) & RTE (m) & RRE (\textdegree) & RTE (m) & RRE (\textdegree) & RTE (m)\\ \hline
\textbf{\textit{VGA-EVD}} & {\bf\num{3.153e-02}} & {\bf\num{3.944e-05}} & {\bf\num{5.227e-02}} & {\bf\num{7.543e-05}} & {\bf\num{1.410e-01}} & {\bf\num{1.929e-04}} & \num{2.889e-01} & \num{3.923e-04}\\ \hline
\textbf{\textit{CGA-EVD}} & \num{3.252e-02} & \num{4.537e-05} & \num{5.785e-02} & \num{9.648e-05} & \num{1.440e-01} & \num{2.230e-04} & {\bf\num{2.747e-01}} & {\bf\num{3.537e-04}}\\ \hline
\textbf{\textit{ICP}}~\cite{besl1992method} & \num{3.387e-01} & \num{8.484e-04} & \num{5.654e-01} & \num{1.462e-03} & \num{9.450e-01} & \num{3.323e-03} & \num{1.697e+00} & \num{4.779e-03}\\ \hline
\textbf{\textit{PASTA-3D}}~\cite{marchi2023sharp}& \num{1.194e+00} & \num{2.773e-03} & \num{2.231e+00} & \num{4.557e-03} & \num{7.351e+00} & \num{1.679e-02} & \num{1.228e+01} & \num{2.580e-02}\\ \hline
\textbf{\textit{DCP}}~\cite{wang2019deep} & \num{9.250e+00} & \num{1.307e-02} & \num{2.170e+01} & \num{3.255e-02} & \num{2.259e+01} & \num{3.079e-02} & \num{6.820e+01} & \num{8.742e-02}\\ \hline
\textbf{\textit{Go-ICP}}~\cite{yang2015go} & \num{9.617e-01} & \num{2.744e-03} & \num{1.262e+00} & \num{3.299e-03} & \num{1.727e+00} & \num{5.299e-03} & \num{2.897e+00} & \num{7.133e-03}\\ \hline
\textbf{\textit{TEASER++}}~\cite{yang2020teaser} & \num{1.857e+00} & \num{3.242e-03} & \num{2.654e+00} & \num{4.970e-03} & \num{1.584e+01} & \num{1.821e-02} & \num{8.149e+00} & \num{1.385e-02}\\ \hline
\end{tabular}}
\hspace{1pt}
\end{table*}

\begin{table*}[h]
\caption{Benchmark results for rotation angle of $\theta=5^\circ$ and translation magnitude $\|\bm{t}\| = 0.01$ for the Stanford Dragon dataset for various levels of noise. The values report the \textit{mean} over $10$ iterations for the same level of noise.}
\label{tab:Dragon-jan}
\centering
\setlength{\tabcolsep}{4pt}
\renewcommand{\arraystretch}{1.2}
\scalebox{0.65}{
\begin{tabular}{|l|l|l|l|l|l|l|l|l|}
\hline
Algorithm & \multicolumn{2}{|c|}{$\sigma=0.001$} & \multicolumn{2}{|c|}{$\sigma=0.002$} & \multicolumn{2}{|c|}{$\sigma=0.005$} & \multicolumn{2}{|c|}{$\sigma=0.01$}\\ \hline
 & RRE (\textdegree) & RTE (m) & RRE (\textdegree) & RTE (m) & RRE (\textdegree) & RTE (m) & RRE (\textdegree) & RTE (m)\\ \hline
\textbf{\textit{VGA-EVD}} & {\bf\num{2.830e-02}} & {\bf\num{7.308e-05}} & {\bf\num{4.731e-02}} & {\bf\num{1.267e-04}} & {\bf\num{9.216e-02}} & {\bf\num{2.268e-04}} & \num{1.485e-01} & {\bf\num{4.319e-04}}\\ \hline
\textbf{\textit{CGA-EVD}} & \num{2.939e-02} & \num{7.679e-05} & \num{4.895e-02} & \num{1.308e-04} & \num{9.409e-02} & \num{2.306e-04} & {\bf\num{1.462e-01}} & \num{4.356e-04}\\ \hline
\textbf{\textit{ICP}}~\cite{besl1992method} & \num{3.629e-01} & \num{8.018e-04} & \num{6.194e-01} & \num{2.041e-03} & \num{1.715e+00} & \num{6.741e-03} & \num{2.919e+00} & \num{7.262e-03}\\ \hline
\textbf{\textit{PASTA-3D}}~\cite{marchi2023sharp} & \num{3.708e-01} & \num{1.010e-03} & \num{7.604e-01} & \num{1.749e-03} & \num{1.785e+00} & \num{3.967e-03} & \num{3.085e+00} & \num{7.459e-03}\\ \hline
\textbf{\textit{DCP}}~\cite{wang2019deep} & \num{7.601e+00} & \num{1.857e-02} & \num{1.050e+01} & \num{2.585e-02} & \num{1.654e+01} & \num{4.021e-02} & \num{5.306e+01} & \num{1.009e-01}\\ \hline
\textbf{\textit{Go-ICP}}~\cite{yang2015go} & \num{8.019e-01} & \num{1.842e-03} & \num{9.176e-01} & \num{3.354e-03} & \num{2.038e+00} & \num{7.422e-03} & \num{3.911e+00} & \num{9.215e-03}\\ \hline
\textbf{\textit{TEASER++}}~\cite{yang2020teaser} & \num{1.321e+00} & \num{3.640e-03} & \num{3.429e+00} & \num{9.375e-03} & \num{2.761e+01} & \num{7.604e-02} & \num{1.119e+02} & \num{1.605e-01}\\ \hline
\end{tabular}}
\hspace{1pt}
\end{table*}

\begin{table*}[h]
\centering
\renewcommand{\arraystretch}{1.2}
\setlength{\tabcolsep}{4pt}
\caption{Benchmark results for random rotation angle and translation magnitude $\|\bm{t}\| = 1$ for the Stanford Bunny dataset for various levels of noise. The values report the \textit{mean} over $10$ iterations for the same level of noise.}
\label{tab:Bunny-jan1}
\scalebox{0.65}{
\begin{tabular}{|l|l|l|l|l|l|l|l|l|}
\hline
Algorithm & \multicolumn{2}{|c|}{$\sigma=0.001$} & \multicolumn{2}{|c|}{$\sigma=0.002$} & \multicolumn{2}{|c|}{$\sigma=0.005$} & \multicolumn{2}{|c|}{$\sigma=0.01$}\\ \hline
 & RRE (\textdegree) & RTE (m) & RRE (\textdegree) & RTE (m) & RRE (\textdegree) & RTE (m) & RRE (\textdegree) & RTE (m)\\ \hline
\textbf{\textit{VGA-EVD}} & {\bf\num{9.962e-02}} & {\bf\num{1.078e-04}} & \num{1.943e-01} & \num{1.668e-04} & \num{3.250e-01} & \num{4.039e-04} & \num{1.027e+00} & {\bf\num{1.307e-03}}\\ \hline
\textbf{\textit{CGA-EVD}} & \num{1.023e-01} & \num{1.119e-04} & {\bf\num{1.897e-01}} & {\bf\num{1.642e-04}} & {\bf\num{3.242e-01}} & {\bf\num{4.022e-04}} & {\bf\num{1.009e+00}} & \num{1.360e-03}\\ \hline
\textbf{\textit{ICP}}~\cite{besl1992method} & \num{6.580e+01} & \num{6.044e-02} & \num{1.120e+02} & \num{1.101e-01} & \num{1.232e+02} & \num{1.180e-01} & \num{1.176e+02} & \num{1.372e-01}\\ \hline
\textbf{\textit{PASTA-3D}} ~\cite{marchi2023sharp}& \num{7.993e-01} & \num{8.760e-04} & \num{1.452e+00} & \num{2.141e-03} & \num{3.048e+00} & \num{3.846e-03} & \num{9.020e+00} & \num{7.361e-03}\\ \hline
\textbf{\textit{DCP}}~\cite{wang2019deep} & \num{7.747e+01} & \num{4.311e-02} & \num{6.282e+01} & \num{3.198e-02} & \num{8.137e+01} & \num{5.097e-02} & \num{3.400e+01} & \num{5.041e-02}\\ \hline
\textbf{\textit{Go-ICP}}~\cite{yang2015go} & \num{8.885e+01} & \num{8.820e-02} & \num{6.234e+01} & \num{6.903e-02} & \num{1.265e+02} & \num{1.197e-01} & \num{6.797e+01} & \num{8.158e-02}\\ \hline
\textbf{\textit{TEASER++}}~\cite{yang2020teaser} & \num{4.813e+00} & \num{2.185e-03} & \num{2.065e+00} & \num{2.387e-03} & \num{2.420e+01} & \num{3.112e-02} & \num{3.567e+01} & \num{3.894e-02}\\ \hline
\end{tabular}}
\hspace{1pt}
\end{table*}

\begin{table*}
\centering
\renewcommand{\arraystretch}{1.2}
\setlength{\tabcolsep}{4pt}
\caption{Benchmark results for random rotation angle and translation magnitude $\|\bm{t}\| = 1$ for the Stanford Armadillo dataset for various levels of noise. The values report the \textit{mean} over $10$ iterations for the same level of noise.}
\label{tab:Armadillo-jan1}
\scalebox{0.65}{
\begin{tabular}{|l|l|l|l|l|l|l|l|l|}
\hline
Algorithm & \multicolumn{2}{|c|}{$\sigma=0.001$} & \multicolumn{2}{|c|}{$\sigma=0.002$} & \multicolumn{2}{|c|}{$\sigma=0.005$} & \multicolumn{2}{|c|}{$\sigma=0.01$}\\ \hline
 & RRE (\textdegree) & RTE (m) & RRE (\textdegree) & RTE (m) & RRE (\textdegree) & RTE (m) & RRE (\textdegree) & RTE (m)\\ \hline
\textbf{\textit{VGA-EVD}} & \num{1.755e-02} & {\bf\num{2.598e-05}} & {\bf\num{6.899e-02}} & {\bf\num{9.089e-05}} & {\bf\num{9.812e-02}} & {\bf\num{1.553e-04}} & {\bf\num{2.638e-01}} & {\bf\num{3.623e-04}}\\ \hline
\textbf{\textit{CGA-EVD}} & {\bf\num{1.722e-02}} & \num{2.916e-05} & \num{7.280e-02} & \num{1.005e-04} & \num{1.056e-01} & \num{1.785e-04} & \num{3.037e-01} & \num{3.689e-04}\\ \hline
\textbf{\textit{ICP}}~\cite{besl1992method} & \num{8.861e+01} & \num{1.438e-01} & \num{9.736e+01} & \num{1.635e-01} & \num{8.024e+01} & \num{1.218e-01} & \num{1.021e+02} & \num{1.426e-01}\\ \hline
\textbf{\textit{PASTA-3D}}~\cite{marchi2023sharp} & \num{1.318e+00} & \num{3.025e-03} & \num{1.713e+00} & \num{4.057e-03} & \num{6.381e+00} & \num{1.436e-02} & \num{9.846e+00} & \num{2.059e-02}\\ \hline
\textbf{\textit{DCP}}~\cite{wang2019deep} & \num{1.937e+01} & \num{2.830e-02} & \num{5.132e+01} & \num{6.084e-02} & \num{3.634e+01} & \num{5.294e-02} & \num{3.945e+01} & \num{5.552e-02}\\ \hline
\textbf{\textit{Go-ICP}}~\cite{yang2015go} & \num{9.449e+01} & \num{1.778e-01} & \num{7.929e+01} & \num{1.245e-01} & \num{6.117e+01} & \num{1.149e-01} & \num{7.724e+01} & \num{9.612e-02}\\ \hline
\textbf{\textit{TEASER++}}~\cite{yang2020teaser} & \num{2.295e+00} & \num{3.851e-03} & \num{1.430e+01} & \num{1.824e-02} & \num{2.127e+01} & \num{2.299e-02} & \num{2.232e+01} & \num{2.790e-02}\\ \hline
\end{tabular}}
\hspace{1pt}
\end{table*}

\begin{table*}[h]
\centering
\renewcommand{\arraystretch}{1.2}
\setlength{\tabcolsep}{4pt}
\caption{Benchmark results for random rotation angle and translation magnitude $\|\bm{t}\| = 1$ for the Stanford Dragon dataset for various levels of noise. The values report the \textit{mean} over $10$ iterations for the same level of noise.}
\label{tab:Dragon-jan1}
\scalebox{0.65}{
\begin{tabular}{|l|l|l|l|l|l|l|l|l|}
\hline
Algorithm & \multicolumn{2}{|c|}{$\sigma=0.001$} & \multicolumn{2}{|c|}{$\sigma=0.002$} & \multicolumn{2}{|c|}{$\sigma=0.005$} & \multicolumn{2}{|c|}{$\sigma=0.01$}\\ \hline
 & RRE (\textdegree) & RTE (m) & RRE (\textdegree) & RTE (m) & RRE (\textdegree) & RTE (m) & RRE (\textdegree) & RTE (m)\\ \hline
\textbf{\textit{VGA-EVD}} & \num{1.933e-02} & {\bf\num{4.375e-05}} & \num{4.442e-02} & {\bf\num{1.063e-04}} & {\bf\num{1.082e-01}} & {\bf\num{3.268e-04}} & \num{1.729e-01} & {\bf\num{3.529e-04}}\\ \hline
\textbf{\textit{CGA-EVD}} & {\bf\num{1.916e-02}} & \num{4.617e-05} & {\bf\num{4.401e-02}} & \num{1.096e-04} & \num{1.143e-01} & \num{3.474e-04} & {\bf\num{1.689e-01}} & \num{3.706e-04}\\ \hline
\textbf{\textit{ICP}}~\cite{besl1992method} & \num{3.769e+01} & \num{6.513e-02} & \num{6.589e+01} & \num{8.795e-02} & \num{6.080e+01} & \num{7.849e-02} & \num{6.074e+01} & \num{1.244e-01}\\ \hline
\textbf{\textit{PASTA-3D}}~\cite{marchi2023sharp} & \num{4.737e-01} & \num{1.147e-03} & \num{7.842e-01} & \num{2.088e-03} & \num{2.057e+00} & \num{3.693e-03} & \num{4.502e+00} & \num{4.413e-03}\\ \hline
\textbf{\textit{DCP}}~\cite{wang2019deep} & \num{7.939e+00} & \num{2.230e-02} & \num{1.211e+01} & \num{3.634e-02} & \num{3.319e+01} & \num{8.802e-02} & \num{5.787e+01} & \num{1.352e-01}\\ \hline
\textbf{\textit{Go-ICP}}~\cite{yang2015go} & \num{3.764e+01} & \num{5.438e-02} & \num{4.609e+01} & \num{6.590e-02} & \num{8.674e+01} & \num{1.139e-01} & \num{5.424e+01} & \num{9.500e-02}\\ \hline
\textbf{\textit{TEASER++}}~\cite{yang2020teaser} & \num{3.029e+00} & \num{7.468e-03} & \num{2.832e+01} & \num{3.789e-02} & \num{7.092e+01} & \num{1.456e-01} & \num{8.233e+01} & \num{1.760e-01}\\ \hline
\end{tabular}}
\hspace{1pt}
\end{table*}

\begin{table*}[h]
\centering
\renewcommand{\arraystretch}{1.2}
\caption{Mean inference time (in seconds) for the Stanford Bunny, Armadillo and Dragon datasets for the different algorithms.}
\label{tab:times}
\scalebox{0.85}{
\begin{tabular}{|l|l|l|l|}
\hline
Algorithm & Bunny & Armadillo & Dragon\\ \hline
\textbf{\textit{VGA-EVD}} & \num{9.280e-02} s & \num{2.014e-01} s & \num{3.805e-01} s\\ \hline
\textbf{\textit{CGA-EVD}} & \num{1.619e-01} s & \num{3.749e-01} s & \num{6.831e-01} s\\ \hline
\textbf{\textit{ICP}}~\cite{besl1992method} & \num{5.491e-02} s & \num{1.008e-01} s & \num{1.816e-01} s\\ \hline
\textbf{\textit{PASTA-3D}}~\cite{marchi2023sharp} & {\bf\num{2.453e-02}} s & {\bf\num{5.628e-02}} s & {\bf\num{1.022e-01}} s\\ \hline
\textbf{\textit{DCP}}~\cite{wang2019deep} & \num{8.816e+00} s & \num{1.751e+00} s & \num{1.493e+00} s\\ \hline
\textbf{\textit{Go-ICP}}~\cite{yang2015go} & \num{5.794e+00} s & \num{5.797e+00} s & \num{6.040e+00} s\\ \hline
\textbf{\textit{TEASER++}}~\cite{yang2020teaser} & \num{1.617e+00} s & \num{6.148e+00} s & \num{9.264e+00} s\\ \hline
\end{tabular}}
\hspace{1pt}
\end{table*}

\begin{figure*}
\centering
\includegraphics[width=0.99\linewidth]{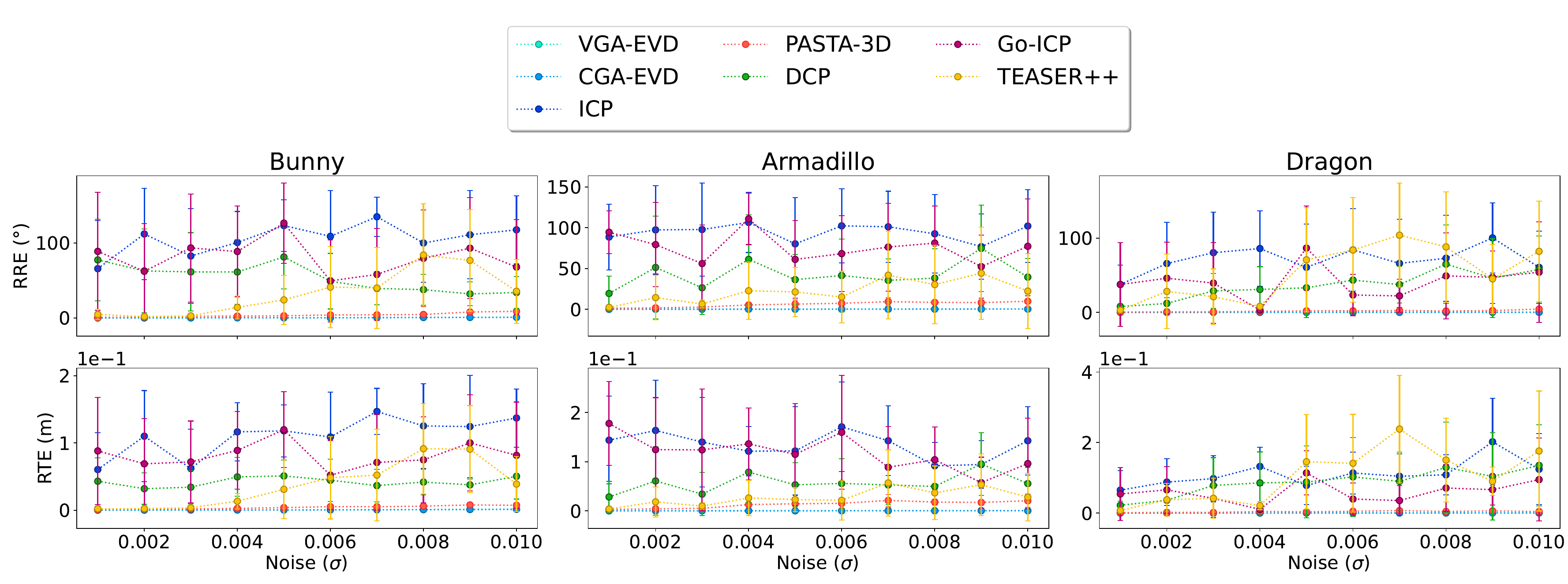}
\caption{Benchmark results for random rotation angle and translation magnitude $\|\bm{t}\| = 1$ of the \textit{\textbf{VGA-EVD}}, 
\textit{\textbf {CGA-EVD}}, \textit{\textbf{ICP}}~\protect\cite{besl1992method}, \textit{\textbf{PASTA-3D}}~\protect\cite{marchi2023sharp}, \textit{\textbf{DCP}}~\protect\cite{wang2019deep}, \textit{\textbf{Go-ICP}}~\protect\cite{yang2015go}, \textit{\textbf{TEASER++}}~\protect\cite{yang2020teaser} methods for the Bunny, Armadillo and Dragon datasets. Error in the rotation (top row), and  translation (bottom row).} \label{fig:algs:big}
\label{fig:benchmark:big:trans}
\end{figure*}

\section{Conclusion}

\label{sec:CONCLUSIONS}

We proposed a novel correspondence free multivector cloud registration in CGA. The distinctiveness of our proposal is characterized by the equivariant relationship between the eigenmultivectors between the source and target multivector clouds. This enables a robust identification of the transformation between the eigenmultivectors which in turn extends to the multivector clouds. 
We have shown how our approach is robust to high levels of noise and how it compares with similar approaches, such as the PASTA and PCA based algorithms, (\textit{i.e. \textbf{VGA-EVD}}). 
Further work will address the problem of low overlap between the two point clouds.  This will be achieved by extending the notion of eigenmultivectors to some specific region in the multivector cloud. Using this approach and by dividing the point cloud into regions could provide a robust rotation and translation equivariant descriptors (the eigenmultivectors), as well as invariant descriptors (the eigenvalues) for multivector cloud data.

\backmatter


\section*{Declarations}
This work was supported by LARSyS funding
(DOI: 10.54499/LA/P/0083/2020,10.54499/UIDP/50009/2020 and 10.54499/UIDB/50009/2020) and projects,
MIA-BREAST [10.54499/2022.04485.PTDC], PT SmartRetail [PRR-C645440011-00000062], Center for Responsible AI [PRR-C645008882-0000005]

Data supporting this study are openly available from \cite{Stanford} at
\href{http://www.graphics.stanford.edu/data/3Dscanrep/}{Stanford 3D Scanning Repository}.
\bigskip

%
%
%
%

\onecolumn

\begin{appendices}

\newcommand{\Gpq}{\mathcal{G}_{p,q}}
\newcommand{\refe}{{\mathrm{ref}}}
\textbf{\LARGE{Introduction to the Appendix}}\mbox{}\\

This appendix is structured as follows:

\begin{itemize}


\item Section~\ref{sec:GA} introduces Geometric Algebra (GA), with core concepts, specifying the notation which is used throughout the paper, (Sec.~\ref{sec:EST-R-T} of the paper).

\item Section~\ref{sec:diff:mv} details how the multivectors differentiation is computed and introduces differentiation as a limit definition. Also, it  describes some derivatives that are further used in Section~\ref{sec:opt:appendix} (optimization), (Sec.~\ref{sec:EST-R-T} of the paper).


\item Section~\ref{sec:CGA} presents Conformal Geometric Algebra (CGA) by introducing the conformal mapping. It describes 3D translations and rotation in CGA and how to decompose a multivector in CGA into its constituent parts (the coefficients), concretizing them analytically, (Secs.~\ref{sec:COEF:approach}, ~\ref{sec:REL-COEF} ~\ref{sec:EST-R-T}, ~\ref{sec:RESULTS}, ~\ref{sec:RESULTS-jacinto} of the paper).

\item Section~\ref{sec:multi:trans} defines multilinear and differential multilinear transformations that are used to define the symmetry of these family of functions. It describes the eigendecomposition of symmetric multilinear functions, (Sec.~\ref{sec:gen:approach} of the paper).


\item Section~\ref{sec:opt:appendix} introduces the Lagrangian with respect to multivector constraints, {\it i.e.} the multivector Lagrange Multiplier. Solutions for the optimal rotor and the optimal translation problems are provided, (Sec.~\ref{sec:EST-R-T} of the paper).


\end{itemize}

\newpage
\section{Geometric Algebra}\label{sec:GA}

The contents described in the section are mainly based on the book by David Hestenes~\cite{hestenes_clifford_1984}. Our goal is to provide some background and theoretical cornerstones in GA used in this paper, contributing for the selfcontainess of the document. Also note that, this section provides some definitions that will be used in Sec.~\ref{sec:CGA}.

Let $\mathcal{A}_{p,q}$ be a $(p+q)$-dimensional vector space over the field of real numbers $\mathbb{R}$, with canonical vector basis
\begin{equation}
\{\underbrace{\bm{e}_1,\bm{e}_2,\dots,\bm{e}_p}_{\bm{e}_i^2=+1},\underbrace{{\bm{e}}_{p+1},{\bm{e}}_{p+2},\dots,{\bm{e}}_{p+q}}_{\bm{e}_j^2=-1}\}
\end{equation}
Also, let the product be defined in $\mathcal{A}_{p,q}$ according to 

\begin{align}
\bm{e}_i\bm{e}_j + \bm{e}_i\bm{e}_j &= 2\eta_{ij}
\label{eq:com-prod}
\end{align}
where 
\begin{equation}
\eta_{ij} = 
\begin{cases}
    +1 & \text{if}\ i = j,\  1 \leqslant i \leqslant p\\
    -1 & \text{if}\ i = j,\  p+1 \leqslant i \leqslant p + q\\
    0 & \text{if}\ i \neq j
\end{cases}
\end{equation}
The non-commutative product in~\eqref{eq:com-prod} is called the geometric product, and generates the associative $2^{p+q}$ dimensional geometric algebra $\mathcal{G}_{p,q}=\mathcal{G}(\mathcal{A}_{p,q})$ over the vector space $\mathcal{A}_{p,q}$.

The geometric algebra $\Gpq$ is a $2^{p+q}$ dimensional real linear vector space with canonical basis
\begin{equation}
\{\bm{e}_J\ :\ J\subseteq \{ 1,2,\dots,p+q\}\}\label{eq:GA:canonical:basis}
\end{equation}
where 
\begin{equation}
\bm{e}_J \equiv \bm{e}_{j_1j_2\cdots j_k}\equiv \bm{e}_{j_1}\bm{e}_{j_2}\cdots\!\ \bm{e}_{j_k},\quad\text{for}\ 1\leqslant j_1<j_2<\cdots<j_k \leqslant p+q\label{eq:basis:multivectors:Gpq}
\end{equation}
and $\bm{e}_\emptyset = 1$. Elements of the basis $\{\bm{e}_J\}$ are called basis blades.

A basis blade $\bm{e}_{J}$ which can be written as the geometric product of $k$ basis vectors is said to be of grade $k$. This means that the geometric algebra $\Gpq$ can be decomposed into linear subspaces of different grade, that is,

\begin{equation}
\mathcal{G}_{p,q} = \bigoplus_{k=0}^{p+q} \mathcal{G}^k_{p,q}.
\end{equation}
\begin{sloppypar}
\noindent where the $k$-grade subspace is denoted by $\mathcal{G}^k_{p,q}$. Note that the scalars ${\mathbb{R}\equiv\Gpq^0}$ and the vectors ${\mathcal{A}_{p,q}\equiv\Gpq^1}$ are both subspaces of the geometric algebra $\Gpq$.  Geometric Algebra can also be decomposed into even grade linear subspaces $\mathcal{G}^+_{p,q}$ and odd grade linear subspaces $\mathcal{G}^-_{p,q}$, {\it i.e.}
\end{sloppypar}
\begin{equation}
\mathcal{G}_{p,q} = \mathcal{G}^+_{p,q}\oplus\mathcal{G}^-_{p,q}
\end{equation}
where the even graded subspace $\mathcal{G}^+_{p,q}$ is on itself a subalgebra of $\mathcal{G}_{p,q}$. We can call it the {\it even subalgebra of} $\mathcal{G}_{p,q}$. A general element of the geometric algebra $\mathcal{G}_{p,q}$ is called a multivector. 
We assume that any multivector $\bm{X}$ can be written as the sum 
\begin{equation}
\bm{X} \equiv \sum_{k=0}^{p+q} \langle \bm{X}\rangle_k\label{eq:sum:grades}
\end{equation}
where $\langle \bm{X} \rangle_k$ is the grade projection operator defined as
\begin{subequations}
\begin{equation}
\langle \cdot \rangle_k\ :\ \mathcal{G}_{p,q}\mapsto \mathcal{G}^k_{p,q}\ :\ \bm{X}\mapsto \langle \bm{X} \rangle_k
\end{equation}
We drop the subscript when $k=0$,  {\it i.e.} when we do the projection to the scalar part  $\langle \bm{X} \rangle_0\equiv\langle \bm{X} \rangle\in \mathbb{R}$.  
The grade projection operator is distributive
\begin{equation}
\langle \bm{X} + \bm{Y}\rangle_k = \langle \bm{X}\rangle_k + \langle\bm{Y}\rangle_k,\quad \forall \bm{X}, \bm{Y}\in \Gpq
\end{equation}
It is associative and commutative in terms of scalar multiplication
\begin{equation}
\alpha\langle \bm{X}\rangle_k = \langle \alpha \bm{X}\rangle_k = \langle \bm{X}\rangle_k\alpha,\quad \text{if }\ \alpha=\langle\alpha\rangle\ \text{and }\ \forall\bm{X}\in \Gpq
\end{equation}

A multivector $\bm{A}$ is said to be a $k$-vector if it is an element of the $k$-grade subspace $\mathcal{G}_{p,q}^k$. Then if $\bm{A}\equiv \langle\bm{A}\rangle_k$ we can call $\bm{A}$ a multivector of grade $k$ or a $k$-vector. We will often write bivector to refer to $2$-vectors, trivector for $3$-vectors and vector to refer to $1$-vectors. 
We will often employ the notation $\langle\cdot\rangle_{k_1,k_2,\dots,k_s}$ to refer to grade projection to multiple grades simultaneously, that is
\begin{equation}
    \langle\bm{X}\rangle_{k_1,k_2,\dots,k_s} = \langle\bm{X}\rangle_{k_1} + \langle\bm{X}\rangle_{k_2} + \cdots + \langle\bm{X}\rangle_{k_s} \label{eq:multiple:grade:proj}
\end{equation}
\end{subequations}


The reversion operation ${}^\dagger$ reorders the factor in a product. It can be defined by  changing the order of the basis vectors of a given multivector. For a basis blade $\bm{e}_J$ it is defined as
\begin{equation}
(\bm{e}_{j_1}\bm{e}_{j_2}\cdots\!\ \bm{e}_{j_k})^\dagger \equiv \bm{e}_{j_k} \bm{e}_{j_{k-1}}\cdots\!\ \bm{e}_{j_1}
\end{equation}

From the definition above, for $\bm{A},\bm{B}\in\Gpq$ and $\bm{a}\in\mathcal{A}_{p,q}$, we can find how the reversion operation enjoys the properties
\begin{subequations}
\begin{align}
(\bm{AB})^\dagger &= \bm{B}^\dagger \bm{A}^\dagger\label{eq:dagger:prop:1}\\
(\bm{A}+\bm{B})^\dagger &= \bm{A}^\dagger + \bm{B}^\dagger\\
\langle\bm{A}^\dagger\rangle &= \langle\bm{A}\rangle\\
\bm{a}^\dagger &= \bm{a} \quad \text{where} \ \bm{a}=\langle\bm{a}\rangle_1
\end{align}
\end{subequations}
it follows that the reverse of a product of vectors is
\begin{equation}
(\bm{a}_1\bm{a}_2\cdots\bm{a}_k)^\dagger= \bm{a}_k\bm{a}_{k-1}\cdots\bm{a}_1 \label{eq:dagger:product:of:vectors}
\end{equation}
where $\bm{a}_i\in\mathcal{A}_{p,q}$.

Let $\bm{A}_k = \bm{a}_1\bm{a}_2\cdots\bm{a}_k$ and assume that the vectors $\bm{a}_i$ anti-commute, that is $\bm{a}_i\bm{a}_j = -\bm{a}_j\bm{a}_i$ for $i\neq j$, then reordering the vectors on the right hand side of (\ref{eq:dagger:product:of:vectors}) we easily prove that 
\begin{equation}
\bm{A}_{k}^\dagger \equiv (-1)^{k(k-1)/2}\bm{A}_{ k} \label{eq:dagger:contas}
\end{equation}
Then, the reversion operator acts on a multivector $\bm{X}$ as
\begin{equation}
\bm{X}^\dagger = \sum_{i=0}^{p+q} (-1)^{i(i-1)/2}\langle\bm{X}\rangle_i
\end{equation}

We define the \textit{inner product} of multivector of unique grade by
\begin{subequations}
\begin{align}
\bm{A}_r\cdot\bm{B}_s &\equiv \langle\bm{A}_r\bm{B}_s\rangle_{|r-s|},\quad \text{for}\ r,s>0\\
\bm{A}_r\cdot\bm{B}_s &\equiv 0, \quad \text{if}\ r = 0\ \text{or}\ s = 0
\end{align}
The inner product for arbitrary multivectors is then defined by
\begin{align}
\bm{A}\cdot\bm{B} \equiv \sum_{i=0}^{n} \langle\bm{A}\rangle_i\cdot\bm{B} \equiv \sum_{j=0}^{n} \langle\bm{B}\rangle_j\cdot\bm{A} \equiv \sum_{i=0}^{n}\sum_{j=0}^{n} \langle\bm{A}\rangle_i\cdot\langle\bm{B}\rangle_j
\end{align}
\end{subequations}
with $n=p+q$. In a similar way, we define the \textit{outer product} of multivectors of unique grade by
\begin{subequations}
\begin{equation}
\bm{A}_r\wedge\bm{B}_s \equiv \langle\bm{A}_r\bm{B}_s\rangle_{r+s}
\end{equation}
The \textit{outer product} for arbitrary multivectors is then defined by
\begin{align}
\bm{A}\wedge\bm{B} \equiv \sum_{i=0}^{n} \langle\bm{A}\rangle_i\wedge\bm{B} \equiv \sum_{j=0}^{n} \langle\bm{B}\rangle_j\wedge\bm{A} \equiv \sum_{i=0}^{n}\sum_{j=0}^{n} \langle\bm{A}\rangle_i\wedge\langle\bm{B}\rangle_j
\end{align}
\end{subequations}

\textbf{The \textit{scalar product}} is defined as the scalar part of the geometric product of two arbitrary multivectors $\bm{A},\bm{B}$ and is defined as
\begin{subequations}
\begin{equation}
\bm{A}*\bm{B} = \langle\bm{AB}\rangle
\end{equation}
We can show that the scalar product has the following  properties
\begin{align}
\bm{A}^\dagger *\bm{B}^\dagger &= \bm{A}*\bm{B}\quad \Leftrightarrow\quad \langle\bm{A}^\dagger \bm{B}^\dagger\rangle = \langle\bm{A}\bm{B}\rangle \label{eq:dagger:scalar:product} \\
\bm{A}*\bm{B} &= \bm{B}*\bm{A} 
\end{align}
that is, we can change the order inside the scalar product. Moreover, inside of the scalar grade projection we have the commutative property 
\begin{equation}
\langle \bm{ABC}\rangle = \langle \bm{CAB}\rangle = \langle \bm{BCA}\rangle \label{eq:commutation:property}
\end{equation}
As a particular case of (\ref{eq:dagger:scalar:product}),  we also have  
\begin{equation}
\bm{A}^\dagger * \bm{B} = \bm{A}*\bm{B}^\dagger\label{eq:A:dagger:star:B}
\end{equation}

The scalar product between two arbitrary multivectors can be expanded in terms of its graded components, that is
\begin{align}
\bm{A}*\bm{B} \equiv \sum_{i=0}^{n} \langle\bm{A}\rangle_i*\bm{B} \equiv \langle \bm{A}\rangle\langle\bm{B}\rangle + \sum_{i=1}^{n} \langle\bm{A}\rangle_i\cdot\langle\bm{B}\rangle_i
\end{align}
\end{subequations}
The inner product between two $k$-vectors $\bm{A}_k$ and $\bm{B}_k$, for $k\neq 0$ can be expressed with the scalar product,
\begin{equation}
\bm{A}_k\cdot\bm{B}_k = \langle \bm{A}_k\bm{B}_k \rangle = \bm{A}_k*\bm{B}_k \label{eq:scalar:inner}
\end{equation}
The scalar product of two multivectors of different grade is zero, that is,
\begin{equation}
\langle\langle \bm{A}\rangle_k \langle\bm{B}\rangle_r\rangle = \langle \bm{A}\rangle_k *\langle\bm{B}\rangle_r = 0\quad\text{if}\ r\neq k\label{eq:scalar:diff:grade}
\end{equation}

We call a multivector $\bm{X}$ ``simple'' when it satisfies 
\begin{equation}
\bm{XX}^\dagger = \langle \bm{XX}^\dagger\rangle 
\end{equation}
For simple multivectors,  the inverse can be computed as
\begin{equation}
\bm{X}^{-1} = \frac{\bm{X}^\dagger}{\bm{X}\bm{X}^\dagger} = \frac{\bm{X}^\dagger}{\langle \bm{XX}^\dagger\rangle} \label{eq:inverse:simple:mlvectors}
\end{equation}
when $\bm{XX}^\dagger\neq 0$. 

The norm of a multivector $\bm{X}$ is defined as the absolute value of the scalar product of $\bm{X}$ with its reverse, that is
\begin{equation}
|\bm{X}|^2 = |\langle \bm{XX}^\dagger\rangle| \label{norm:positive:negative:sigs}
\end{equation}
The absolute value in the right hand side of (\ref{norm:positive:negative:sigs}) is important since in pseudo-Euclidean spaces, the scalar product $\langle \bm{XX}^\dagger\rangle$ can have negative values. 

To define orthogonal transformations, we will use  the following  norm squared definition of a multivector as 
\begin{equation}
\|\bm{X}\|^2 = \langle \bm{XX}^\dagger\rangle  \label{norm:negative:negative:sigs}
\end{equation}
In contrast to (\ref{norm:positive:negative:sigs}), this norm can be negative.


\begin{definition}
We call an element $\bm{A}_k$ a \textit{$k$-blade} if it is a $k$-vector which can be written as the product of $k$ anticommuting vectors $\bm{a}_1,\bm{a}_2,\dots,\bm{a}_k\in \mathcal{A}_{p,q}$, that is
\begin{subequations}
\begin{equation}
\bm{A}_k = \bm{a}_1\bm{a}_2\cdots\bm{a}_k,
\end{equation}
where
\begin{equation}
\bm{a}_j\bm{a}_i = -\bm{a}_i\bm{a}_j\label{eq:anticommute:k:blade}
\end{equation}
for $i,j=1,2,\dots,k$, and $j\neq i$. Furthermore, a $k$-blade is simple thus it squares to a scalar quantity, that is
\begin{equation}
\bm{A}_k^2 \equiv \langle \bm{A}_k^2\rangle = (-1)^{k(k-1)}\langle \bm{A}_k\bm{A}_k^\dagger\rangle = (-1)^{k(k-1)}\|\bm{A}_k\|^2
\end{equation}
we can always write a \textit{$k$-blade} as the wedge product of $k$ linearly independent vectors $\bm{b}_k$ then
\begin{equation}
\bm{A}_k = \bm{b}_1\wedge\bm{b}_2\wedge\cdots\wedge\bm{b}_k
\end{equation}
\end{subequations}
Note how we used~\eqref{eq:dagger:contas} to write $\bm{A}_k^2$ in terms of the norm squared of $\bm{A}_k$.
\end{definition}

We can express the inner and outer product between a vector $\bm{a}$ and a $k$-vector $\bm{A}_k$ as
\begin{subequations}
\begin{align}
\bm{a}\cdot\bm{A}_k &= \tfrac{1}{2}(\bm{a}\bm{A}_k - (-1)^k\bm{A}_k\bm{a})\label{eq:inner:geo:prod}\\
\bm{a}\wedge\bm{A}_k &= \tfrac{1}{2}(\bm{a}\bm{A}_k + (-1)^k\bm{A}_k\bm{a})
\end{align}
\label{eq:inner:outer:geo}
We can also write the geometric product with respect to the inner and outer product, concretely let $\bm{A}$ be a multivector and $\bm{a}$ a vector, then
\begin{align}
\bm{A}\bm{a} &= \bm{A}\cdot\bm{a} + \bm{A}\wedge\bm{a}\\
\bm{aA} &= \bm{a}\cdot\bm{A} + \bm{a}\wedge\bm{A} \label{eq:geo:inner:outer:prods}
\end{align}
\end{subequations}

\begin{subequations}
The outer and the inner product satisfy the reordering rules
\begin{align}
\bm{A}_i\cdot\bm{B}_j &= (-1)^{i(j-1)}\bm{B}_j\cdot\bm{A}_i\label{eq:reorder:rules:inner}\\
\bm{A}_i\wedge\bm{B}_j &= (-1)^{ij}\bm{B}_j\wedge\bm{A}_i\label{eq:reorder:rules:wedge}
\end{align}
\end{subequations}

\textbf{{Projections}} to flat spaces can be defined via an inner product, we define a projection to an invertible $s$-blade $\bm{A}$ as
\begin{equation}
P_{\bm{A}}(\bm{X}_k) = 
\begin{cases}
\bm{X}_k & \text{if}\ k = 0\\
\bm{X}_k\cdot\bm{A}\bm{A}^{-1} & \text{if}\ 0 < k \leqslant s\\
0 & \text{otherwise}
\end{cases}
\label{eq:projection:def}
\end{equation}
projections satisfy the outermorphism property
\begin{equation}
P_{\bm{A}}(\bm{X}\wedge\bm{Y}) = P_{\bm{A}}(\bm{X})\wedge P_{\bm{A}}(\bm{Y})
\end{equation}
See proof in~\cite{hestenes_clifford_1984} (Section 1.2).
\textbf{{The Unit Pseudoscalar}} $\bm{I}\in\mathcal{G}_{p,q}^{p+q}$ of a geometric algebra $\mathcal{G}_{p,q}$ is the element of greatest grade that has magnitude one, $|\bm{I}|=1$. Concretely, we can define it from the basis  vectors
\begin{subequations}
\begin{equation}
\bm{I} = \bm{e}_1\bm{e}_2\cdots\!\ \bm{e}_{n} = \bm{e}_1\wedge \bm{e}_2\wedge\cdots\wedge\bm{e}_{n}
\end{equation}
\begin{equation}
\bm{I} \equiv \langle \bm{I}\rangle_{n} \equiv \langle \bm{I}\rangle_{p+q},\quad |\bm{I}| = 1
\end{equation}
\label{eq:unit:pseudoscalar}
\end{subequations}

\begin{fornextpaper}
{
\subsection{Positive Norm in Pseudo-Euclidean spaces}
In this section we provide a construction of a norm under a newly defined operator. The operation $\bm{X}^+$ is almost equivalent to the reverse operation, except it acts on the basis vectors, thus having the properties
\begin{subequations}
\begin{align}
(\bm{X}\bm{Y})^+ = \bm{Y}^+\bm{X}^+\\
\bm{e}_i^+ = \bm{e}_i\ \text{if}\ i \leqslant p\\
\bm{e}_i^+ = -\bm{e}_i\ \text{if}\ i > p
\end{align}
\end{subequations}
where $\bm{X}\in\mathcal{G}_{p,q}$. We shall call this operator the plus operator. 
\begin{definition}
\label{def:the:plus:norm}
Under this operator we define a positive definite norm as
\begin{equation}
\| \bm{X}\|^2_+ \equiv \langle \bm{XX}^+\rangle 
\end{equation}
\end{definition}
the inner product under this operator is symmetric thus satisfying
\begin{equation}
\bm{X}*\bm{Y}^+ = \bm{X}^+*\bm{Y}\label{eq:synmmetric:plus:inner:prod}
\end{equation}
it commutes with the dagger operation
\begin{equation}
(\bm{X}^+)^\dagger = \bm{X}^{+\dagger} = \bm{X}^{\dagger +} = (\bm{X}^\dagger)^+
\end{equation}
Furthermore we have 
\begin{equation}
\|\bm{X} + \bm{Y}\|^2_+ = \|\bm{X}\|^2_+ + \|\bm{Y}\|^2_+ + 2 \langle \bm{XY}^+\rangle = \|\bm{X}\|^2_+ + \|\bm{Y}\|^2_+ + 2 \langle \bm{YX}^+\rangle\label{eq:norm:plus:expanded}
\end{equation}

}
\end{fornextpaper}

\subsection{Identities}
In this section we provide short proofs for some identities used. Consider the relations (\ref{eq:inner:outer:geo}), and the definition (\ref{norm:negative:negative:sigs}). Let $\bm{A}\equiv \sum_i \bm{A}_i\equiv \sum_i \langle \bm{A}_i\rangle_i$ and $\bm{a}\equiv\langle\bm{a}\rangle_1$, now we show that
\begin{subequations}
\begin{equation}
\begin{split}
\|\bm{a}\cdot\bm{A}\|^2 &= \sum_i \|\bm{a}\cdot\bm{A}_i\|^2 = \sum_i \tfrac{1}{4} \|\bm{aA}_i - (-1)^i\bm{A}_i\bm{a}\|^2 =  \sum_i\tfrac{1}{4}\left(\bm{aA}_i - (-1)^i\bm{A}_i\bm{a}  \right)\left(\bm{A}_i^\dagger\bm{a} - (-1)^i\bm{a}\bm{A}_i^\dagger  \right)\\
& =\sum_i \tfrac{1}{2}\bm{a}^2 \|\bm{A}_i\|^2 - \tfrac{1}{2}(-1)^i \langle \bm{A}_i\bm{aA}_i^\dagger \bm{a}\rangle = \tfrac{1}{2}\bm{a}^2\|\bm{A}\| - \tfrac{1}{2}\sum_i (-1)^i \langle \bm{A}_i\bm{aA}_i^\dagger \bm{a}\rangle
\end{split}
\end{equation}

\begin{equation}
\begin{split}
\|\bm{a}\wedge\bm{A}\|^2 &= \sum_i \|\bm{a}\wedge\bm{A}_i\|^2 = \sum_i \tfrac{1}{4} \|\bm{aA}_i + (-1)^i\bm{A}_i\bm{a}\|^2 =  \sum_i\tfrac{1}{4}\left(\bm{aA}_i + (-1)^i\bm{A}_i\bm{a}  \right)\left(\bm{A}_i^\dagger\bm{a} + (-1)^i\bm{a}\bm{A}_i^\dagger  \right)\\
& =\sum_i \tfrac{1}{2}\bm{a}^2 \|\bm{A}_i\|^2 + \tfrac{1}{2}(-1)^i \langle \bm{A}_i\bm{aA}_i^\dagger \bm{a}\rangle = \tfrac{1}{2}\bm{a}^2\|\bm{A}\| + \tfrac{1}{2}\sum_i (-1)^i \langle \bm{A}_i\bm{aA}_i^\dagger \bm{a}\rangle
\end{split}
\end{equation}
By which it is straightforward to show that  
\begin{equation}
\|\bm{a}\wedge\bm{A}\|^2 + \|\bm{a}\cdot\bm{A}\|^2 = \|\bm{aA}\|^2 = \bm{a}^2\|\bm{A}\|^2 \label{eq:sum:square:inner:outer:geo}
\end{equation}
\end{subequations}

\begin{subequations}
Let $\bm{A}_k$ be a $k$-vector, $\bm{B}_{k-1}$ a $(k-1)$-vector and $\bm{B}_{k+1}$ a $(k+1)$-vector then using (\ref{eq:dagger:contas}) and (\ref{eq:inner:geo:prod}), we obtain the following equalities
\begin{equation}
\begin{split}
\langle \bm{a}\cdot\bm{A}_k\bm{B}_{k-1}\rangle &= \tfrac{1}{2}\langle \bm{aA}_k\bm{B}_{k-1}\rangle -\tfrac{1}{2}(-1)^k\langle \bm{A}_k\bm{a}\bm{B}_{k-1}\rangle\\
                                               &= \tfrac{1}{2}\langle \bm{aA}_k\bm{B}_{k-1}\rangle -\tfrac{1}{2}(-1)^k\langle \bm{A}_k^\dagger\bm{B}_{k-1}^\dagger\bm{a}\rangle\\
                                               & = \langle \bm{aA}_k\bm{B}_{k-1}\rangle = \bm{a}\cdot \langle \bm{A}_k\bm{B}_{k-1}\rangle_1
\end{split}
\end{equation}

\begin{equation}
\begin{split}
\langle \bm{a}\wedge\bm{A}_k\bm{B}_{k-1}\rangle &= \tfrac{1}{2}\langle \bm{aA}_k\bm{B}_{k+1}\rangle +\tfrac{1}{2}(-1)^k\langle \bm{A}_k\bm{a}\bm{B}_{k+1}\rangle\\
                                               &= \tfrac{1}{2}\langle \bm{aA}_k\bm{B}_{k+1}\rangle +\tfrac{1}{2}(-1)^k\langle \bm{A}_k^\dagger\bm{B}_{k+1}^\dagger\bm{a}\rangle\\
                                               & = \langle \bm{aA}_k\bm{B}_{k+1}\rangle = \bm{a}\cdot \langle \bm{A}_k\bm{B}_{k+1}\rangle_1
\end{split}
\end{equation}
\label{eq:scalar:wedge:1:k:k+1/k-1}
\end{subequations}

Let $\bm{A}\equiv \langle \bm{A}\rangle_r$ be an invertible $r$-blade, then 
\begin{equation}
\bm{x}\cdot\bm{A}\bm{A}^{-1} = \tfrac{1}{2}\left(\bm{xA} -(-1)^r\bm{Ax}\right)\bm{A}^{-1}  = \tfrac{1}{2}(\bm{xAA}^{-1} - (-1)^r\bm{AxA}^{-1}) = \tfrac{1}{2}\bm{x} - \tfrac{1}{2}(-1)^r \bm{AxA}^{-1} \label{eq:proj:geo:prod}
\end{equation}
Any vector $\bm{x}$ can be expressed as the component in $\bm{A}$ and as the components orthogonal to $\bm{A}$, having
\begin{equation}
\bm{x} = \bm{xAA}^{-1} = \bm{x}\cdot\bm{AA}^{-1} + \bm{x}\wedge\bm{AA}^{-1} = P_{\bm{A}}(\bm{x}) + P^\perp_{\bm{A}}(\bm{x})
\end{equation}
where $P_{\bm{A}}(\bm{x}) = \bm{x}\cdot\bm{AA}^{-1}$ and $P^\perp_{\bm{A}}(\bm{x}) = \bm{x}\wedge\bm{AA}^{-1}$.

Let $\bm{X}$ and $\bm{Y}$ be general multivectors, then we have
\begin{equation}
P_{\bm{A}}(\bm{XY}) = \bm{X}P_{\bm{A}}(\bm{Y}),\quad\quad \text{if}\ \bm{X} = P_{\bm{A}}(\bm{X})
\end{equation}

\begin{fornextpaper}
{{\bf\it Inverse of multivectors}}, let $\bm{X}$ be a blade,  then we have
\begin{equation}
(1+\bm{X})^{-1} = \frac{1}{1+\bm{X}} = \frac{1-\bm{X}}{(1+\bm{X})(1-\bm{X})} = \frac{1-\bm{X}}{1-\bm{X}^2}\label{eq:inverse:1+X}
\end{equation}
\end{fornextpaper}

\section{Differentiation by Multivectors}\label{sec:diff:mv}
The geometric derivative can be defined with respect to the directional derivative as
\begin{subequations}
\begin{equation}
\partial = \sum_{i=1}^n \bm{a}^i\bm{a}_i\cdot\partial 
\end{equation}
where the vectors $\bm{a}_i$ are any set of vectors that span all $\mathcal{A}_{p,q}$, and satisfy the following reciprocity relation
\begin{equation}
\bm{a}_i\cdot\bm{a}^j = \delta_{ij} 
\end{equation}
for $i,j=1,2,\dots,n$. Where
\begin{equation}
\delta_{ij} = 
\begin{cases}
1 & \text{if}\ i\neq j\\
0 & \text{otherwise}
\end{cases}
\end{equation}
and where we define the directional derivative via the limit definition
\begin{equation}
\bm{a}\cdot\partial F(\bm{x}) = \lim_{\tau\rightarrow 0}\frac{F(\bm{x}+\tau \bm{a}) - F(\bm{x})}{\tau}\label{eq:dir:der:vectors}
\end{equation}
\end{subequations}

The chain rule succintly satisfies
\begin{subequations}
\begin{equation}
\partial_{\bm{x}}f(g(\bm{x})) = \bar{g}(\partial_{\bm{x}'})f(\bm{x}') \label{eq:chain:rule}
\end{equation}
where 
\begin{equation}
\bar{g}(\bm{a}) = \partial_{\bm{x}} g(\bm{x})\cdot\bm{a}\\
\end{equation}
and 
\begin{equation}
\bar{g}(\partial_{\bm{x}'}) = \partial_{\bm{x}} g(\bm{x})\cdot\partial_{\bm{x}'}
\end{equation}
\end{subequations}

Differentiation with respect to multivectors has to be defined via a multivector basis then
\begin{equation}
\partial_{\bm{X}} = \sum_{J}\bm{e}^J\bm{e}_J*\partial_{\bm{X}}
\end{equation}
where $\bm{e}_J$ is the basis for the geometric algebra $\mathcal{G}_{p,q}$ defined in~\eqref{eq:basis:multivectors:Gpq}. The directional derivative $\bm{A}*\partial_{\bm{X}}$, which is a generalization of~\eqref{eq:dir:der:vectors}, is defined with respect to a limit, having 
\begin{equation}
\bm{A}*\partial F(\bm{X}) = \lim_{\tau\rightarrow 0}\frac{F(\bm{A}+\tau \bm{A}) - F(\bm{A})}{\tau}\label{eq:dir:der:multivectors}
\end{equation}

Differentiating scalar functions of multivector variable can be employed by understanding some properties of the derivative
\begin{subequations}
\begin{align}
\dot{\partial}_{\bm{X}}\langle\bm{A}(\bm{X})\dot{\bm{X}} \rangle &= \dot{\partial}_{\bm{X}}\langle\dot{\bm{X}}\bm{A}(\bm{X}) \rangle = \dot{\partial}_{\bm{X}} \dot{\bm{X}}*\bm{A}(\bm{X})  = \bm{A}(\bm{X})\\
{\partial}_{\bm{X}} \phi(\bm{X}^\dagger) &= {\partial}_{\bm{X}^\dagger} \phi(\bm{X}) = {\partial}_{\bm{X}}^\dagger \phi(\bm{X}) = ({\partial}_{\bm{X}} \phi(\bm{X}))^\dagger\\
\partial_{ \bm{X}} \phi(\langle\bm{X}\rangle_k) &= \partial_{\langle \bm{X}\rangle_k} \phi(\bm{X}) = \langle \partial_{\bm{X}}\rangle_k \phi(\bm{X}) = \langle \partial_{\bm{X}} \phi(\bm{X}) \rangle_k \\
\partial_{\bm{X}}\phi(P(\bm{X})) &= P(\partial_{\bm{X}})\phi(\bm{X}) = P(\partial_{\bm{X}}\phi(\bm{X})) = P(\partial_{\bm{X}}\phi(P(\bm{X}))), \quad \text{if}\ \bm{X} = P(\bm{X})\\
\partial_{\bm{X}} \langle f(\bm{X})g(\bm{X})\rangle_k &=  \dot{\partial}_{\bm{X}} \langle \dot{f}(\bm{X})g(\bm{X})\rangle_k + \dot{\partial}_{\bm{X}} \langle {f}(\bm{X})\dot{g}(\bm{X})\rangle_k\label{eq:prod:rule:grade:proj}\\
\partial_{\bm{X}} \langle \bm{XAX^\dagger\bm{B}}\rangle &= \dot{\partial}_{\bm{X}} \langle \dot{\bm{X}}\bm{AX}^\dagger\bm{B}\rangle + \dot{\partial}_{\bm{X}} \langle {\bm{X}}\bm{A}\dot{\bm{X}}^\dagger\bm{B}\rangle = \bm{AX}^\dagger\bm{B} + \bm{A}^\dagger\bm{X}^\dagger\bm{B}^\dagger
\end{align}
\label{eq:derivatives:properties}
\end{subequations}
where the dot indicates the element in the product which is to be differentiated. 
Some directional derivatives of multivectors can be succinctly expressed as
\begin{equation}
\bm{Y}*\partial_{\bm{X}} \bm{AXB} = \bm{A}\bm{Y}*\partial_{\bm{X}} \bm{XB} = \bm{AYB} \label{eq:dir:deriv}
\end{equation}

\begin{fornextpaper}
\section{Linear Transformations}
\label{sec:lin:trans}

Given a vector function of a vector variable $f:\mathcal{A}_{p,q}\mapsto \mathcal{A}_{p,q}$ we define the differential and the adjoint transformations as
\begin{subequations}
\begin{align}
\munderbar{f}(\bm{x}) \equiv \bm{x}\cdot\partial_{\bm{z}} f(\bm{z}) \\
\bar{f}(\bm{x}) \equiv \partial_{\bm{z}} \bm{x}\cdot f(\bm{z}) \label{eq:adjoint:computation}
\end{align}
\end{subequations}
We can relate the adjoint to the transpose of the Jacobian matrix of a function, while the adjoint is related to the Jacobian matrix.
When composing multiple functions we find it convenient to superimpose the concatenation notation, that is, we have the equivalent notation
\begin{equation}
fg(\bm{x}) \equiv f(g(\bm{x})), \quad\quad fg \equiv f\circ g
\end{equation}

We call a transformation $f$ a linear transformation when $f:\mathcal{A}_{p,q}\mapsto \mathcal{A}_{p,q}$ is linear on its arguments and is vector valued of a vector variable. We reserve the word multilinear transformation for a transformation between multivectors.

\begin{definition}
A \textit{symmetric transformation} $f$ is a transformation for which the differential $\bar{f}$ is equal to the adjoint $\munderbar{f}$, that is 
\begin{equation*}
\bar{f}(\bm{x}) = \munderbar{f}(\bm{x})
\end{equation*}
A \textit{skew-symmetric transformation} $f$ is a transformation for which the differential $\bar{f}$ is equal to minus the adjoint $\munderbar{f}$, that is 
\begin{equation*}
\bar{f}(\bm{x}) = -\munderbar{f}(\bm{x})
\end{equation*}
\end{definition}
\end{fornextpaper}
\begin{fornextpaper}
\begin{theorem}
\label{theo:proj:symetry}
The projection operation is symmetric. The projection is defined as 
\begin{equation}
P(\bm{x}) = \bm{x}\cdot\bm{A}\bm{A}^{-1}
\end{equation}
with $\bm{A}\equiv\langle\bm{A}\rangle_k$ a $k$-vector. Then this theorem implies that 
\begin{equation}
\bar{P}(\bm{x}) = P(\bm{x})
\end{equation}
\end{theorem}
\begin{proof}
To show how the above theorem holds we use (\ref{eq:adjoint:computation}), (\ref{eq:inner:geo:prod}), (\ref{eq:commutation:property}) and (\ref{eq:inverse:simple:mlvectors}) to show that
\begin{equation*}
\begin{split}
\bar{P}(\bm{x}) &= \partial_{\bm{z}} \bm{x}\cdot P(\bm{z}) = \partial_{\bm{z}} \langle \bm{xz}\cdot\bm{AA}^{-1}\rangle = \tfrac{1}{2}\partial_{\bm{z}}\langle\bm{xzA} - (-1)^k\bm{xAzA}^{-1}\rangle\\
&= \tfrac{1}{2}\partial \langle\bm{xz}\rangle -(-1)^k\tfrac{1}{2}\partial_{\bm{z}}\langle \bm{xAzA}^{-1}\rangle = \tfrac{1}{2}\partial\bm{z}\cdot\bm{x} - (-1)^k\tfrac{1}{2}\partial_{\bm{z}}\bm{z}\cdot(\bm{A}^{-1}\bm{xA})\\
& = \tfrac{1}{2}(\bm{x} - (-1)^k\bm{A}^{-1}\bm{xA}) = \tfrac{1}{2}(\bm{xA} - (-1)^k\bm{Ax})\bm{A}^{-1} = \bm{x}\cdot\bm{A}\bm{A}^{-1}
\end{split}
\end{equation*}
\end{proof}

\begin{theorem}
In $\mathcal{G}_{p,q}$ any symmetric transformation $f$ can have the spectral decomposition

\begin{equation}
f(\bm{x}) =\sum_{k=1}^m \alpha_k \bm{A}_k\bm{x}\bm{A}_k^\dagger
\end{equation}
where $\bm{A}_i\bm{A}_j = \bm{A}_i\wedge \bm{A}_j$, for $i\neq j$ are orthogonal blades of arbitrary grade. Furthermore by using (\ref{eq:inner:geo:prod}) we can express the spectral decomposition as 
\begin{equation}
f(\bm{x}) =\sum_{k=1}^m \beta_k \bm{x}\cdot\bm{A}_k\bm{A}_k^\dagger
\end{equation}
for some scalars $\beta_k$.

\end{theorem}

\begin{theorem}
The following linear function is symmetric 
\begin{align}
g(\bm{x}) &= \sum_{k=1}^m \beta_k \langle \bm{B}_k\bm{x}\bm{B}_k\rangle_1 \label{eq:sym:func:grade:one:proj}
\end{align}
where $\bm{B}_k\in \mathcal{G}_{p,q}$ and $\beta_k$  are scalars.
\end{theorem}
\begin{proof}
To show that $g$ is symmetric we start by pointing out that because of (\ref{eq:scalar:inner}) we have $\bm{a}\cdot\langle\bm{A}\rangle_1 = \langle\bm{aA}\rangle$ then also using (\ref{eq:commutation:property}) we can easily show that 
\begin{align*}
\bar{g}(\bm{z}) &= \partial_{\bm{z}} \bm{x}\cdot g(\bm{z}) = \partial_{\bm{z}}\sum_{k=1}^m \beta_k\bm{x}\cdot \langle \bm{B}_k\bm{zB}_k\rangle_1= \partial_{\bm{z}}\sum_{k=1}^m \langle\bm{x} \bm{B}_k\bm{zB}_k\rangle \\
&= \partial_{\bm{z}}\sum_{k=1}^m \langle\bm{zB}_k\bm{x} \bm{B}_k\rangle = \partial_{\bm{z}}\sum_{k=1}^m \beta_k\bm{z}\cdot \langle \bm{B}_k\bm{xB}_k\rangle_1 = \sum_{k=1}^m \beta_k \langle \bm{B}_k\bm{xB}_k\rangle_1
\end{align*}
\end{proof}

\begin{theorem}[Decomposition of bivectors and skew symmetric transformations]
\label{theo:skew:decomp}
Let $\bm{B} = \bm{B}_1 + \bm{B}_2 + \cdots + \bm{B}_s$ be a bivector expressed as the sum of orthogonal $2$-blades. A bivector $\bm{B}$ can be decomposed into a set of orthogonal $2$-blades by finding the decomposition of 
\begin{equation}
f(\bm{x}) = \bm{x}\cdot\bm{B} \label{eq:skew:symmetric}
\end{equation}
furthermore the decomposition of $f$ can be obtained by forming the symmetric transformation 
\begin{equation}
F(\bm{x}) = \bar{f}\munderbar{f}(\bm{x}) = (\bm{x}\cdot\bm{B})\cdot\bm{B}^\dagger\label{eq:symmetric:from:skew}
\end{equation}
and finding the eigenvalues and eigenvectors of $F$. The decomposition of (\ref{eq:skew:symmetric}) is only equivalent to the decomposition of (\ref{eq:symmetric:from:skew}) if the eigenvalues of $F$ associated with each $\bm{B}_i$ are unique.
\end{theorem}
\begin{proof}[Proof of Theorem \ref{theo:skew:decomp}]

\end{proof}
\end{fornextpaper}

\begin{fornextpaper}
\subsection{Orthogonal Transformations}
\begin{subequations}
\begin{definition}
A linear transformation $U=U(\bm{x})$ of $\mathcal{A}_{p,q}$ into itself, is said to be an orthogonal transformation of $\mathcal{A}_{p,q}$ if 
\begin{equation}
U(\bm{x})\cdot U(\bm{y}) = \bm{x}\cdot\bm{y} \label{eq:definition:orthogonal:transformation}
\end{equation}
for each $\bm{x}$ and $\bm{y}$ in $\mathcal{A}_{p,q}$. The group of all orthogonal transformations of  $\mathcal{A}_{p,q}$ is called the orthogonal group of $\mathcal{A}_{p,q}$ and denoted by $O(\mathcal{A}_{p,q})$ or more briefly, by $O(p,q)$.
\end{definition}
In particular orthogonal transformations are magnitude preserving, that is
\begin{equation}
\|\munderbar{U}(\bm{x})\|^2 = \|\bm{x}\|^2 \label{eq:ortho:norm:square}
\end{equation}
and they satisfy 
\begin{equation}
 {U}(\bm{x})\cdot\bm{y} = \bm{x}\cdot \bar{U}(\bm{y}) \label{eq:ortho:scalar:prod}
\end{equation}
Furthermore, the inverse of an orthogonal transformation is its transpose
\begin{equation}
{U}^{-1}(\bm{x}) = \bar{U}(\bm{x})
\end{equation}
\end{subequations}
Any orthogonal transformation can be expressed as the composition of at most $n=p+q$ reflections of invertible vectors $\bm{u}_i\in\mathcal{A}_{p,q}$, with $|\bm{u}_i| \neq 0$, $i=1,2,\dots,s$ where $s\leqslant n$, as such
\begin{equation}
U(\bm{x}) = (-1)^s\bm{u}_1\bm{u}_2\cdots\bm{u}_s\bm{x}\bm{u}_s^{-1}\cdots\bm{u}_{2}^{-1}\bm{u}_1^{-1}\label{eq:reflections:orthotransformation}
\end{equation}
Let $\bm{U}=\bm{u}_1\bm{u}_2\cdots\bm{u}_s$ then we 
can always write an orthogonal transformation as the following ternary product (aka sandwich product)
\begin{subequations}
\begin{equation}
U(\bm{x}) = (-1)^s\bm{UxU}^{-1}\label{eq:decomposition:UxU}
\end{equation}
where $\bm{UU}^{-1} = 1$.  

Note that the inverse of $\bm{U}$ can be written with respect to a scalar and the reverse operation, then 
\begin{equation}
\bm{U}^{-1} = \frac{\bm{U}^\dagger}{\|\bm{U}\|^2}
\end{equation}
The adjoint of~\eqref{eq:decomposition:UxU} can be determined using (\ref{eq:adjoint:computation}), (\ref{eq:commutation:property}) and (\ref{eq:scalar:inner}), it is expressed as
\begin{equation}
\bar{U}(\bm{x}) = (-1)^s\bm{U}^{-1}\bm{xU}
\end{equation}
\end{subequations}

Note that (\ref{eq:decomposition:UxU}) is easily shown to satisfy (\ref{eq:ortho:scalar:prod}), replacing the above into (\ref{eq:ortho:scalar:prod}) and using (\ref{eq:commutation:property}) we find that 
\begin{equation}
{U}(\bm{x})\cdot\bm{y} = \langle (-1)^s\bm{UxU}^{-1} \bm{y}\rangle  = \langle \bm{x} (-1)^s\bm{U}^{-1} \bm{yU}\rangle =  \bm{x}\cdot \bar{U}(\bm{y})
\end{equation}

\begin{definition}
We call a multivector $\bm{U}$ a versor if it can be factored into a product of invertible vectors. The multivector 
\begin{equation}
    \bm{U} = \bm{u}_1\bm{u}_2\cdots\bm{u}_m
\end{equation}
is a versor when $|\bm{u}_i| \neq 0$, for $i=1,2,\cdots,m$
\end{definition}

\begin{definition}
A simple reflection along some invertible vector $\bm{u}$, $|\bm{u}| \neq 0$, is the simplest kind of orthogonal transformation 
\begin{equation}
U(\bm{x}) = -\bm{uxu}^{-1}\label{eq:reflection}
\end{equation}
\end{definition}
A reflection is an ``involutory'' transformation, that is, its inverse equals to itself. Being an involuntary and orthogonal transformation means that it is also symmetric. Expanding the geometric product in terms of the outer and inner products, the reflection can be written as
\begin{equation}
U(\bm{x}) = \bm{x}\wedge\bm{uu}^{-1} - \bm{x}\cdot \bm{uu}^{-1}
\end{equation}
making it clear that the eigenblades of $U$ are $\bm{Iu}$ and $\bm{u}$ with eigenvalues $+1$ and $-1$, respectively.


\begin{definition}
\label{def:unitary:multivector}
A unitary multivector $\bm{X}$ is a multivector of unit magnitude that is, if $\bm{X}$ is unitary then
\begin{equation}
    |\bm{X}| = 1
\end{equation}
\end{definition}
By definition (\ref{def:unitary:multivector}) and recalling that a versor satisfies $\bm{UU}^\dagger = \langle \bm{UU}^\dagger\rangle = \|\bm{U}\|^2$, then a unitary versor is a versor that satisfies $\bm{UU}^\dagger = \pm 1$.

\begin{definition}
A versor which can be expressed as the product of two invertible vectors will be called a simple rotor, and the corresponding rotation will be called a simple rotation. 
We may also refer to simple rotors in the context of the spectral decomposition as `complex' eigenvalues.
\end{definition}
Consider a simple rotor 
\begin{equation}
\bm{R} = \bm{uv} = \bm{u}\cdot\bm{v} + \bm{u}\wedge\bm{v}
\end{equation}

The rotor above can be written in exponential form $\bm{R}=\rho e^{\theta\bm{B}}$. Note that depending on the type of bivector $\bm{B}$, we have different types of rotations. For elliptic, hyperbolic and parabolic rotations respectively, we have the special forms for the exponential 
\begin{equation}
e^{\theta\bm{B}} = 
\begin{cases}
    \cos\theta + \bm{B}\sin\theta, & \text{if}\ \bm{B}^2 = -1\\
    \cosh\theta + \bm{B}\sinh\theta, & \text{if}\ \bm{B}^2 = 1\\
    1 + \theta\bm{B}, & \text{if}\ \bm{B}^2 = 0
\end{cases}
\end{equation}

\begin{lemma}
\label{lemma:rotations:projections}
Let $\bm{R} = e^{\theta\bm{B}/2}$ be a simple elliptic or hyperbolic rotor, and let $R(\bm{x}) = \bm{RxR}^\dagger$ be its corresponding simple rotation. We can write a rotation in terms of the components in the plane of rotation and the components orthogonal to the plane of rotation, then
\begin{equation}
R(\bm{x}) = RP(\bm{x}) + P_\perp(\bm{x})
\end{equation}
with $P(\bm{x}) = \bm{x}\cdot\bm{B}\bm{B}^{-1}$ and $P_\perp(\bm{x}) = \bm{x}\wedge\bm{B}\bm{B}^{-1}$
Furthermore the projected component $P(\bm{x}) = \bm{x}\cdot\bm{B}\bm{B}^\dagger$ commutes with $\bm{R}$, thus we can also write a rotation as
\begin{equation}
\bm{RxR}^\dagger = \bm{R}^2P(\bm{x}) + P_\perp(\bm{x})\label{eq:proj:rej:rotation}
\end{equation}
\end{lemma}
\begin{proof}
Let $\bm{x} = P(\bm{x}) + P_{\perp}(\bm{x})$. Then using \eqref{eq:reorder:rules:inner} and \eqref{eq:reorder:rules:wedge} we show that 
\begin{subequations}
\begin{align}
\bm{B}P(\bm{x}) &= \bm{B}(\bm{x}\cdot\bm{B}\bm{B}^{-1}) = \bm{B}\bm{B}^{-1}\bm{B}\cdot\bm{x} = \bm{B}\cdot\bm{x} = -\bm{x}\cdot\bm{B} = -\bm{x}\cdot\bm{BB}^{-1}\bm{B} =- P(\bm{x})\bm{B}\\
\bm{B}P_{\perp}(\bm{x}) &= \bm{B}(\bm{x}\wedge\bm{B}\bm{B}^{-1}) = \bm{B}\bm{B}^{-1}\bm{B}\wedge\bm{x} = \bm{B}\wedge\bm{x} = \bm{x}\wedge\bm{B} = \bm{x}\wedge\bm{BB}^{-1}\bm{B} = P_\perp(\bm{x})\bm{B}
\end{align}
then we have 
\begin{align}
\bm{R}P(\bm{x}) &= (\alpha + \beta\bm{B})P(\bm{x}) = P(\bm{x})(\alpha - \beta\bm{B}) = P(\bm{x})\bm{R}^\dagger\\
\bm{R}P_\perp(\bm{x}) &= (\alpha + \beta\bm{B})P_\perp(\bm{x}) = P_\perp(\bm{x})(\alpha + \beta\bm{B}) = P_\perp(\bm{x})\bm{R}
\end{align}
and
\begin{align}
\bm{R}P(\bm{x})\bm{R}^\dagger &= \bm{RR}P(\bm{x}) = \bm{R}^2P(\bm{x})\\
\bm{R}P_\perp(\bm{x})\bm{R}^\dagger &= \bm{RR}^\dagger P_\perp(\bm{x}) = P_\perp(\bm{x})
\end{align}
\end{subequations}
\end{proof}
\begin{lemma}
\label{lemma:rotation:rotor:from:eigenvector}
Let $R(\bm{x}) = \bm{RxR}^\dagger$ be a simple rotation, then we can always find a vector $\bm{a}$ such that 
\begin{equation}
    R(\bm{a})\bm{a}^{-1} = \bm{R}^2\label{eq:R(a)a-1:Rsq}
\end{equation}
\end{lemma}
\begin{proof}
Let $\bm{a}$ be some vector such that $\bm{a}=P(\bm{a})$, then from {\it Lemma}~\ref{lemma:rotations:projections} and using~\eqref{eq:proj:rej:rotation} we have
\begin{equation}
R(\bm{a}) = \bm{R}^2P(\bm{a}) + P_\perp(\bm{a}) = \bm{R}^2\bm{a}
\end{equation}
multiplying on the right by $\bm{a}^{-1}$ we find that~\eqref{eq:R(a)a-1:Rsq} holds.
\end{proof}

Even thought we proved {\it Lemma} \ref{lemma:rotation:rotor:from:eigenvector} only for simple elliptical and hyperbollic rotations through {\it Lemma}~\ref{lemma:rotations:projections}, {\it Lemma} \ref{lemma:rotation:rotor:from:eigenvector} also holds true for simple parabolic rotations. 
\end{fornextpaper}
\begin{toberemoved}
{
\begin{theorem}
Any orthogonal transformation which also includes parabolic simple rotations can be written in the form
\begin{equation}
U(\bm{x}) = \sum_{k=1}^m \bm{R}_kP_k(\bm{x})\bm{R}_k^\dagger
\end{equation}
where $\bm{R}_k$ is a simple rotor. And $P_k$ is the projection to some space.
\end{theorem}

Assume that $\bm{R}_1$ is a parabolic simple rotor, and $\bm{R}_2$ is either an elliptic or hyperbolic rotor, also assume they commute, then we can write
\begin{equation}
\begin{split}
\bm{R}_1\bm{R}_2\bm{x}\bm{R}_2^\dagger\bm{R}_1^\dagger &= \bm{R}_1P_\perp(\bm{x})\bm{R}_1^\dagger + \bm{R}_1\bm{R}_2P(\bm{x})\bm{R}_2^\dagger\bm{R}_1^\dagger\\
&= \bm{R}_1P_\perp(\bm{x})\bm{R}_1^\dagger + \bm{R}_2^2P(\bm{x})
\end{split}
\end{equation}
where $P$ is the projection to the plane of rotation of $\bm{R}_2$.

\paragraph{Parabolic Rotations}
The case of parabolic rotations cannot be included in Lemma \ref{lemma:rotations:projections} since projection to null bivectors are not well defined.

Let $\bm{R} = e^{\bm{B}/2} = 1 + \bm{B}/2$ with $\bm{B}=\bm{uv}=\bm{u}\wedge\bm{v}$ where we consider $\bm{u}^2=0$ and $\bm{v}^2\neq 0$ and let ${R=R(\bm{x}) = \bm{RxR}^\dagger}$ be the simple rotation associated with $\bm{R}$. For each vector $\bm{u}$ and $\bm{v}$ we have a clear distinction on how a rotation $R$ acts on those two elements. When $\bm{B}$ is a null vector we can express a rotation ${R}$ as
\begin{equation}
\begin{split}
R(\bm{x}) &= \left(1+\tfrac{1}{2}\bm{B}\right)\bm{x}\left(1-\tfrac{1}{2}\bm{B}\right) = \bm{x} - \tfrac{1}{4}\bm{BxB} + \tfrac{1}{2}(\bm{Bx} - \bm{xB})\\
          &= \bm{x} - \tfrac{1}{2}\bm{x}\cdot\bm{BB}^\dagger - \bm{x}\cdot\bm{B}
\end{split}
\end{equation}
where we used 
\begin{subequations}
\begin{align}
\tfrac{1}{2}(\bm{Bx} - \bm{xB}) &= \bm{B}\cdot\bm{x} = -\bm{x}\cdot\bm{B}\\
\bm{x}\cdot\bm{BB}^\dagger &=-\bm{x}\cdot\bm{BB} = -\tfrac{1}{2}(\bm{xB} - \bm{Bx})\bm{B} = -\tfrac{1}{2}(\bm{xB}^2 - \bm{BxB}) =  \tfrac{1}{2}\bm{BxB}
\end{align}
\end{subequations}
We use (...) to show that 
\begin{subequations}
\begin{align}
\bm{x}\cdot\bm{B} &= \bm{x}\cdot \bm{uv} - \bm{x}\cdot\bm{vu}\\
\bm{x}\cdot\bm{B}\bm{B}^\dagger &= \bm{x}\cdot\bm{uu}\bm{v}^2 + \bm{x}\cdot\bm{vv}\bm{u}^2 = \bm{x}\cdot\bm{uu}\bm{v}^2
\end{align}
\end{subequations}
Note how $\bm{x}\cdot\bm{B}\bm{B}^\dagger$ is non-zero only when $\bm{x}$ has a component which is a scalar multiple of the reciprocal vector to $\bm{u}$. Then
\begin{equation}
R(\bm{x}) = \bm{x} - \tfrac{1}{2}\bm{x}\cdot\bm{uu}\bm{v}^2 - \bm{x}\cdot\bm{uv} + \bm{x}\cdot\bm{vu}
\end{equation}
Now we can finally compute what a rotation does to $\bm{u}$ and $\bm{v}$
\begin{subequations}
\begin{align}
R(\bm{u}) &= \bm{u} - \tfrac{1}{2}\bm{u}\cdot\bm{uu}\bm{v}^2 - \bm{u}\cdot\bm{uv} + \bm{u}\cdot\bm{vu} = \bm{u}\\
R(\bm{v}) &= \bm{v} - \tfrac{1}{2}\bm{v}\cdot\bm{uu}\bm{v}^2 - \bm{v}\cdot\bm{uv} + \bm{v}\cdot\bm{vu} = \bm{v} + \bm{v}^2\bm{u}
\end{align}
We can readily see that $\bm{u}$ behaves as a component orthogonal to the plane of rotation, that is even though it is in the plane of rotation it satisfies $R(\bm{u}) = \bm{u}$ for this particular vector we have $\bm{RuR}^\dagger = \bm{R}^2\bm{u} = \bm{u}$ which shows how it is collinear to $\bm{B}$ yet under rotations it behaves like a perpendicular vector. Let $\bm{b}$ be a vector which satisfies $\bm{b}\cdot\bm{u} = 1$ and $\bm{b}\cdot\bm{v} = 0$ then
\begin{equation}
R(\bm{b}) = \bm{b} - \tfrac{1}{2}\bm{b}\cdot\bm{uu}\bm{v}^2 - \bm{b}\cdot\bm{uv} + \bm{b}\cdot\bm{vu} = \bm{b} - \bm{v} - \tfrac{1}{2}\bm{u}\bm{v}^2 
\end{equation}
\end{subequations}
$\bm{b}$ is neither a point in the plane of rotation or a point orthogonal to the plane of rotation, yet since $\bm{B}$ is null we cannot factor $\bm{b}$ into components orthogonal to $\bm{B}$ and in $\bm{B}$.
\begin{equation}
R(\bm{u}) = \bm{R}^2\bm{u} = \bm{u}\label{eq:rot:of:null}
\end{equation}

Recall that a vector $\bm{a}$ in the plane of rotation $\bm{B}$ satisfies
\begin{equation}
\bm{a}\wedge \bm{B}\quad \Leftrightarrow \quad \bm{RaR}^\dagger = \bm{R}^2\bm{a}
\end{equation}
$\bm{a}$ can of course be written as a linear combination of $\bm{u}$ and $\bm{v}$. Then let $\bm{a} = \alpha \bm{u} + \beta\bm{v}$. Then any non-null vector $\bm{a}$ can be used to determine the rotor $\bm{R}$ via 
\begin{equation}
\bm{R}^2 = R(\bm{a})\bm{a}^{-1}
\end{equation}
consider (\ref{eq:rot:of:null}) then a rotation of $\bm{a}$ will simplify to
\begin{equation}
R(\bm{a}) = \alpha\bm{u} + \beta\bm{R}^2\bm{v}
\end{equation}
Note that the inverse of $\bm{a}$ is given by $\bm{a}^{-1} = \bm{a}/\bm{a}^2 = \bm{a}/(\alpha\bm{u} + \beta\bm{v})^2 =\bm{a}(\beta\bm{v})^{-2}$ thus
\begin{equation}
\begin{split}
R(\bm{a})\bm{a}^{-1} &= (\alpha\bm{u} + \beta\bm{R}^2\bm{v})(\alpha \bm{u} + \beta\bm{v})(\beta\bm{v})^{-2} \\
                     &= (\alpha^2\bm{u}^2 + \beta^2\bm{R}^2\bm{v}^2 + \alpha\beta\bm{uv} + \alpha\beta\bm{R}^2\bm{vu})(\beta\bm{v})^{-2}\\
                     &= (\beta^2\bm{R}^2\bm{v}^2 + \alpha\beta\bm{uv} + \alpha\beta\bm{vu})(\beta\bm{v})^{-2}\\
                     &= \beta^2\bm{R}^2\bm{v}^2(\beta\bm{v})^{-2}\\
                     &= R(\bm{v})\bm{v}^{-1}
\end{split}
\end{equation}
where we used 
\begin{equation}
\bm{R}^2\bm{vu} = -\bm{R}^2\bm{uv} = -\bm{uv} = \bm{vu}
\end{equation}
Thus the component of $\bm{a}$ in $\bm{u}$ does not contribute for the determination of the rotor $\bm{R}^2$.

In elliptic and hyperbolic rotations we can study the eigenvectors of rotations via the decomposition of the symmetric part. But for parabolic rotations using the symmetric part will not help finding the rotor since the eigenvectors are not unique. Let the symmetric part of $R$ be given by 
\begin{equation}
R_+(\bm{x}) = \tfrac{1}{2}(\munderbar{R}(\bm{x}) + \bar{R}(\bm{x})) = \bm{x} - \tfrac{1}{2}\bm{x}\cdot\bm{BB}^\dagger
\end{equation}
We now compute what is $R_+$ of $\bm{v}$, $\bm{u}$ and $\bm{b}$
\begin{subequations}
\begin{align}
R_+(\bm{v}) &= \bm{v}\\
R_+(\bm{u}) &= \bm{u}\\
R_+(\bm{b}) &= \bm{b} - \bm{u}\bm{v}^2
\end{align}
\end{subequations}

The skew-symmetric part can be computed as 
\begin{equation}
R_-(\bm{x}) = \tfrac{1}{2}(\munderbar{R}(\bm{x}) - \bar{R}(\bm{x})) = \bm{x} - \tfrac{1}{2}\bm{x}\cdot\bm{BB}^\dagger = -\bm{x}\cdot\bm{B}
\end{equation}
then 
\begin{subequations}
\begin{align}
R_-(\bm{v}) & = \bm{v}^2\bm{u}\\
R_-(\bm{u}) &= 0\\
R_-(\bm{b}) &= \bm{b} - \bm{v}
\end{align}
\end{subequations}

I'm trying to determine the versor $\bm{V}$ from an orthogonal transformation $H$ in 3D CGA. If $H$ does not include a parabolic rotation then I can use tools from linear algebra to determine the versor $\bm{V}$. The algorithm can determine the eigenvalues of $H_+$ but it cannot use the eigenvectors  of $H_+$ to find the eigendecomposition of $H$.
}
\end{toberemoved}

\begin{fornextpaper}
\begin{theorem}
In $\mathcal{G}_{p,q}$ an orthogonal transformation $U$, which does not include parabolic rotations, can be expressed via the spectral decomposition
\begin{equation}
U(\bm{x}) = \sum_{k=1}^m \bm{\lambda}_k P_k(\bm{x}) \label{eq:spectral:decomp:geo:algebra}
\end{equation}
where 
\begin{equation}
P_k(\bm{x}) = \bm{x}\cdot\bm{A}_k\bm{A}_k^{-1}, \label{eq:proj:spectral:theorem}
\end{equation}
$\bm{\lambda}_k = \langle \bm{\lambda}_k\rangle + \langle \bm{\lambda}_k\rangle_2 = e^{\theta_k\bm{B}_k}$, with $\bm{B}_k$ an orthogonal 2-blade, that is, it satisfies $\bm{B}_i\bm{B}_j = \bm{B}_i\wedge\bm{B}_j$. The $\bm{\lambda}_k$'s are unitary thus $\bm{\lambda}_k\bm{\lambda}_k^\dagger=1$. The $\bm{A}_k$'s are either vectors or 2-blades, they are orthogonal thus satisfying $\bm{A}_i\bm{A}_j = \bm{A}_i\wedge\bm{A}_j $. When the $\bm{A}_k$ is a vector then the $\bm{\lambda}_k$ must be either equal to one or to minus one. When the $\bm{A}_k$'s are bivectors they relate with the $\bm{B}_k$'s as 
\begin{equation}
\bm{B}_k\bm{A}_k = \bm{B}_k\cdot\bm{A}_k \label{eq:Aks:Bks:ortho:relation}
\end{equation}
The spectral decomposition is unique!
\end{theorem}
\begin{definition}
An eigenvector $\bm{a}$ of an orthogonal transformation $U$ is a vector which for some $k$ satisfies
\begin{equation}
P_k(\bm{a}) = \bm{a}
\end{equation}
where $P_k$ is defined by (\ref{eq:proj:spectral:theorem}).
\end{definition}

If $\bm{a}$ is an eigenvector of $U$ then we can determine the corresponding eigenvalue $\bm{\lambda}$ via
\begin{equation}
\bm{\lambda} = U(\bm{a})\bm{a}^{-1}
\end{equation}
\begin{theorem}
If the scalar part of the `complex' eigenvalues of $U$ are unique then the eigenvectors of an orthogonal transformation $U$ can be determined by computing the eigenvectors of the symmetric part of $U$. The symmetric part is defined as
\begin{equation}
U_+(\bm{x}) = \tfrac{1}{2}(\munderbar{U}(\bm{x}) + \bar{U}(\bm{x})) \label{eq:symmetric:part:ortho:theo}
\end{equation}
\end{theorem}

\begin{proof}
First we show that 
\begin{equation}
\bar{U}(\bm{x}) = \sum_{k=1}^m \bm{\lambda}_k^\dagger P_k(\bm{x}) \label{eq:adj:ortho:transformation}
\end{equation}
then we can show that the symmetric part of $U$ (\ref{eq:symmetric:part:ortho:theo}) has the decomposition
\begin{equation}
U_+(\bm{x}) = \sum_{k=1}^m \alpha_k P_k(\bm{x}) \label{eq:symmtric:plus:of:ortho}
\end{equation}
where $P_k$ is defined by (\ref{eq:proj:spectral:theorem}) and $\alpha_k=\langle\bm{\lambda}_k\rangle$.

To show how we obtain (\ref{eq:adj:ortho:transformation}) we determine the adjoint with (\ref{eq:adjoint:computation}) thus
\begin{align*}
\bar{U}(\bm{x}) &= \partial_{\bm{z}} \sum_{k=1}^m \bm{x}\cdot(\bm{\lambda}_k P_k(\bm{z})) =  \sum_{k=1}^m \partial_{\bm{z}} \langle\bm{x}\bm{\lambda}_kP_k(\bm{z})\rangle = \sum_{k=1}^m \bar{P}_k(\langle\bm{x}\bm{\lambda}_k\rangle_1) = \sum_{k=1}^m \bm{\lambda}_k^\dagger P_k(\bm{x})
\end{align*}

The second to last step can be shown by using the chain rule (\ref{eq:chain:rule}) on $P_k$, and realizing that because of theorem \ref{theo:proj:symetry} we have $\bar{P}_k = \munderbar{P}_k$, then we can readily see that 
\begin{align*}
\partial_{\bm{z}} \langle \bm{x} \bm{\lambda}_kP_k(\bm{z})\rangle &= P_k(\partial_{\bm{z}'}) \langle \bm{x} \bm{\lambda}_k\bm{z}'\rangle = P_k(\partial_{\bm{z}'} \langle \bm{x} \bm{\lambda}_k\bm{z}'\rangle) \\
&= P_k\left( \partial_{\bm{z}'} \bm{z}'\cdot \langle \bm{x} \bm{\lambda}_k \rangle_1 \right)
= P_k(\langle \bm{x}\bm{\lambda}_k\rangle_1)
\end{align*}
\todo[inline,color=green]{Still need to prove that $(\bm{x}\cdot\bm{A})\cdot\bm{A} = (\bm{x}\cdot\bm{A})\bm{A}$}
The last step we show by using the relation (\ref{eq:Aks:Bks:ortho:relation}) then we have $P_k(\bm{x}\cdot \bm{B}_k) = P_k(\bm{x})\cdot \bm{B}_k = -\bm{B}_k\cdot P_k(\bm{x}) = -\bm{B}_kP_k(\bm{x})$ and noting that $\langle\bm{x}\bm{\lambda}_k\rangle_1 = \bm{x}\langle \bm{\lambda}_k\rangle +  \bm{x}\cdot \langle \bm{\lambda}_k\rangle_2 = \bm{x}\langle \bm{\lambda}_k\rangle + \bm{x} \cdot \bm{B}_k \sin\theta_k$, we can get the result
\begin{equation*}
    P_k(\langle \bm{x}\bm{\lambda}_k\rangle_1) = P_k(\bm{x})\langle \bm{\lambda}_k\rangle  - \langle\bm{\lambda}_k\rangle_2 P_k(\bm{x}) = \bm{\lambda}_k^\dagger P_k(\bm{x})
\end{equation*}

To show how to get (\ref{eq:symmtric:plus:of:ortho}) we use (\ref{eq:dagger:contas}), then it is enough to show that 
\begin{equation*}
\tfrac{1}{2}\left(\bm{\lambda}_k + \bm{\lambda}_k^\dagger\right) =
\tfrac{1}{2}\left(\langle\bm{\lambda}_k\rangle + \langle \bm{\lambda}_k\rangle_2 + \langle\bm{\lambda}_k\rangle^\dagger + \langle \bm{\lambda}_k\rangle_2^\dagger \right)
= \tfrac{1}{2}\left(\langle\bm{\lambda}_k\rangle + \langle \bm{\lambda}_k\rangle_2 + \langle\bm{\lambda}_k\rangle - \langle \bm{\lambda}_k\rangle_2 \right) = 
\langle\bm{\lambda}_k\rangle
\end{equation*}
\end{proof}


\begin{theorem}
We can use the eigenvectors of $U_+$ (\ref{eq:symmtric:plus:of:ortho}) to determine the canonical form (\ref{eq:decomposition:UxU}) of $U$. Furthermore the rotor $\bm{U}$ can be determined via the square root of the eigenvalues $\bm{\lambda}_i$ of $U$. Since the `complex' eigenvalues of $U$ commute we can find $\bm{U}$ with
\begin{equation}
    \bm{U} = \sqrt{\bm{\lambda}_1}\sqrt{\bm{\lambda}_2}\cdots\sqrt{\bm{\lambda}_m}
\end{equation}
\end{theorem}

\paragraph{The non-uniqueness of the decomposition}
We consider the case of odd geometric algebras, in this algebras the unit pseudoscalar $\bm{I}\equiv \langle \bm{I}\rangle_n$, with $n$ odd, commutes with all multivectors ({\it i.e} ${\bm{IA}_r = \bm{I}\cdot\bm{A}_r = (-1)^{(n-1)r}\bm{A}_r\bm{I} = \bm{A}_r\bm{I}}$). This property of the pseudoscalar enables us to turn reflections into rotations just by noting that we can write 
\begin{equation}
(\bm{IU})\bm{x}(\bm{IU})^\dagger = \bm{UIxI}^\dagger\bm{U}^\dagger = \bm{II}^\dagger \bm{UxU}^\dagger = \|\bm{I}\|^2 \bm{UxU}^\dagger
\end{equation}
Thus an orthogonal transformation by $\bm{IU}$ is equivalent to an orthogonal transformation by $\bm{U}$ times some sign $\|\bm{I}\|^2$.


Any orthogonal transformation can be expressed as the composition of a symmetric orthogonal transformation with a special orthogonal transformation, in particular we can write an orthogonal transformation $H$ as 
\begin{equation}
U(\bm{x}) = (-1)^r \bm{AUxU}^{-1}\bm{A}^{-1}\label{eq:special:ortho:ortho:symmetric}
\end{equation}
where $\bm{A}\equiv\langle\bm{A}\rangle_r$ is an $r$-blade and $\bm{U}=\bm{U}_1\bm{U}_2\cdots\!\ \bm{U}_s$ is an even versor with $\bm{U}_i=\alpha_i + \bm{B}_i$, with $\bm{U}_i\bm{U}_i^\dagger =\pm1$ and the $\bm{B}_i$'s are $2$-blades. The $\bm{B}_i$'s along with $\bm{A}$ form a set of orthogonal blades that expresses the space where the rotations are taken place, 
\begin{subequations}
\begin{align}
\bm{B}_i\bm{B}_j &= \bm{B}_i\wedge\bm{B}_j = \bm{B}_j \bm{B}_i \\
\bm{B}_i\bm{A} &= \bm{B}_i\wedge\bm{A} = \bm{A}\bm{B}_i
\end{align}
\end{subequations}
where we used (\ref{eq:reorder:rules:wedge}) to revert the order of the products. Concretely we can also show that 
${\bm{AB}_1\bm{B}_2\cdots\bm{B}_s = \bm{A}\wedge\bm{B}_1\wedge\bm{B}_2\wedge\cdots\wedge\bm{B}_s}$. Let $P_{\bm{A}}(\bm{x})$ be the projection to $\bm{A}$ and $P_{\bm{A}}^\perp(\bm{x})$ be the projection to the orthogonal space to $\bm{A}$ then we define
\begin{subequations}
\begin{align}
P(\bm{x}) &= P_{\bm{A}}^\perp(\bm{x}) = P_{\bm{B}}(\bm{x}) = \bm{x}\cdot\bm{B}\bm{B}^{-1}\\
P_\perp(\bm{x}) &= P_{\bm{A}}(\bm{x}) = P_{\bm{B}}^\perp(\bm{x}) = \bm{x}\wedge\bm{B}\bm{B}^{-1}
\end{align}
\end{subequations}
with $\bm{B}=\bm{B}_1\bm{B}_2\cdots\bm{B}_s$. Then we can express $H$ via 
\begin{equation}
\begin{split}
U(\bm{x}) &= (-1)^r \left(\bm{A}P_\perp(\bm{x})\bm{A}^{-1} + \bm{U}P(\bm{x})\bm{U}^{-1}\right)\\
          &= \bm{UxU}^{-1} - 2 P_\perp(\bm{x})
\end{split}
\end{equation}
where we have used (\ref{eq:proj:geo:prod}) to transform the product $\bm{AxA}$ into the sum of a projection with the identity transformation.
\begin{theorem}
Let $U$ be given by (\ref{eq:special:ortho:ortho:symmetric}) then the following holds
\begin{equation}
-\tfrac{1}{2}\partial\wedge U \equiv -\tfrac{1}{2}\partial_{\bm{x}} \wedge U(\bm{x}) =  \bm{B}_1 + \bm{B}_2 + \cdots + \bm{B}_s
\end{equation}
\end{theorem}

\paragraph{Decomposition of a bivector into {\it two} distinct bivectors}\mbox{}\\
Consider a bivector $\bm{B}$ that can decomposed into only two orthogonal $2$-blades, that is
\begin{equation}
\bm{B} = \bm{B}_1 + \bm{B}_2
\end{equation}
with $\bm{B}_1\bm{B}_2 = \bm{B}_1\wedge\bm{B}_2 = \bm{B}_2\wedge\bm{B}_1$. The blades $\bm{B}_1$ and $\bm{B}_2$ can be null bivectors $\bm{B}_i^2=0$, negative bivectors $\bm{B}_i^2 < 0$ or positive bivectors $\bm{B}_i^2 > 0$. We aim to find the $2$-blades in all cases. We start by considering that ${\bm{B}_1^2\neq \bm{B}_2^2 \neq 0}$, in such a case we refer to [hestenes CA2GC, chapter 3.4], then the solution is computed as
\begin{subequations}
\begin{align}
\lambda_i &= -\tfrac{1}{2}\left( \|\bm{B}\|^2\pm \sqrt{\|\bm{B}\|^4 - \|\bm{B}\wedge\bm{B}\|^2} \right)\label{eq:lambdai:decomp:biv}\\
\bm{B}_i &= \frac{\bm{B}}{1+\tfrac{1}{2}\lambda_i^{-1}\bm{B}\wedge\bm{B}} = \bm{B}\left(1-\tfrac{1}{2}\lambda_i^{-1}\bm{B}\wedge\bm{B}\right)\left( 1-\tfrac{1}{4}\lambda_i^{-2}\|\bm{B}\wedge\bm{B}\|^2\right)^{-1}
\end{align}
\end{subequations}
where we used (\ref{eq:inverse:1+X}) to compute the inverse of $1+\tfrac{1}{2}\lambda_1^{-1}\bm{B}\wedge\bm{B}$. The other bivector can be computed as $\bm{B}_2 = \bm{B} - \bm{B}_1$, or by choosing a different sign in (\ref{eq:lambdai:decomp:biv}). 

For the particular case when $\bm{B}_1^2=0$ we have 
\begin{equation}
\bm{B}_1 = -\tfrac{1}{2}\frac{\bm{BB}\wedge\bm{B}}{\|\bm{B}\|^2}
\end{equation}
To prove the above we start by computing
\begin{subequations}
\begin{equation}
\begin{split}
\bm{B}\bm{B}\wedge\bm{B} &= \bm{B}(\bm{B}_1 + \bm{B}_2)\wedge(\bm{B}_1+\bm{B}_2) = 2\bm{B}\bm{B}_1\wedge\bm{B}_2 = 2\bm{BB}_1\bm{B}_2 = 2(\bm{B}_1 + \bm{B}_2)\bm{B}_1\bm{B}_2 \\
&=2\bm{B}_1^2\bm{B}_2 + 2\bm{B}_2^2\bm{B}_1 = 2\bm{B}_2^2 \bm{B}_1 = -2\|\bm{B}_2\|^2\bm{B}_1
\end{split}
\end{equation}
then notice that 
\begin{equation}
\|\bm{B}\|^2 = \|\bm{B}_1 + \bm{B}_2\|^2 = \|\bm{B}_1\|^2 + \|\bm{B}_2\|^2 = \|\bm{B}_2\|^2
\end{equation}
\end{subequations}

\begin{theorem}
\label{eq:theo:ortho}
The function 
\begin{equation}
H(\bm{x}) = \sum_{i=1}^5\bm{x}\cdot\bm{p}_i\bm{q}_i^{-1}
\end{equation}
is orthogonal when the $\bm{p}_i$'s and the $\bm{q}_i$'s form an orthogonal set of vectors and when $\bm{p}_i^2 = \bm{q}_i^2$.
\end{theorem}
\begin{proof}
We start by proving that (\ref{eq:H:eig:relations:CGA}) is indeed an orthogonal transformation, we use the definition (\ref{eq:definition:orthogonal:transformation}), we recall that the eigenvectors forms a set of orthonormal vectors, that is
\begin{subequations}
\doubleequation{\bm{p}_i\cdot\bm{p}_j^{-1} = \delta_{ij},}{\bm{q}_i\cdot\bm{q}_j^{-1} = \delta_{ij}}
\label{eq:reciprocal:P:Q}
\end{subequations}
and that $\bm{p}_i^2 = \bm{q}_i^2$. Then compute
\begin{equation}
\begin{split}
H(\bm{x}) \cdot H(\bm{y}) &= \left (\sum_{i=1}^5 \bm{x}\cdot\bm{p}_i\bm{q}_i^{-1} \right)\cdot\left (\sum_{j=1}^5 \bm{x}\cdot\bm{p}_j\bm{q}_j^{-1} \right)\\
                          &= \left (\sum_{i=1}^5 \bm{x}\cdot\bm{p}_i^{-1}\bm{q}_i \right)\cdot\left (\sum_{j=1}^5 \bm{x}\cdot\bm{p}_j\bm{q}_j^{-1} \right)\\
                          &= \sum_{i=1}^5 \bm{x}\cdot\bm{p}_i^{-1}\bm{y}\cdot\bm{p}_i\\
                          &= \left (\sum_{i=1}^5 \bm{x}\cdot\bm{p}_i\bm{p}_i^{-1} \right)\cdot\left (\sum_{j=1}^5 \bm{x}\cdot\bm{p}_j\bm{p}_j^{-1} \right)\\
                          &= \bm{x}\cdot\bm{y}
\end{split}
\end{equation}
\end{proof}
\end{fornextpaper}

\begin{fornextpaper}
\section{Outermorphisms}
\label{sec:outermorphisms}

The linear transformation $f=f(\bm{x})$ induces a linear mapping $\munderbar{f}(\bm{X})$ of every multivector $\bm{X}$ in $\mathcal{G}_{p,q}$, which we call the {\it outermorphism} of $f$, or the {\it differential outermorphism} of $f$ when we want to distinguish it from the adjoint outermorphism. The outermorphism $\munderbar{f}$ satisfies
\begin{subequations}
\begin{align}
\munderbar{f}(\langle\bm{X}\rangle) &= \langle\bm{X}\rangle\\
\munderbar{f}(\bm{x}) &= f(\bm{x})\\
\munderbar{f}(\bm{X}\wedge\bm{Y}) &= \munderbar{f}(\bm{X})\wedge \munderbar{f}(\bm{Y})\\
\munderbar{f}(\bm{X}+\bm{Y}) &= \munderbar{f}(\bm{X})+ \munderbar{f}(\bm{Y})
\end{align}
\end{subequations}
for all vectors $\bm{x}$ and all multivectors $\bm{X}$ and $\bm{Y}$. The adjoint $\bar{f}$ outermorphism of $f$ is defined as the outermorphism of 
. Note how in particular the following equalities hold
\begin{align}
\munderbar{f}(\bm{X}) &= \munderbar{f}\left( \sum_{k=0}^n \langle \bm{X}\rangle_k \right) = \sum_{k=0}^n \munderbar{f}(\langle \bm{X}\rangle_k)\\
\munderbar{f}(\bm{a}_1\wedge\bm{a}_2\wedge\cdots\wedge\bm{a}_k) &= \munderbar{f}(\bm{a}_1)\wedge \munderbar{f}(\bm{a}_2)\wedge \cdots\wedge \munderbar{f}(\bm{a}_k)
\end{align}
where $\bm{a}_1,\bm{a}_2,\dots,\bm{a}_k\in\mathcal{A}_{p,q}$.

The outermorphism of an orthogonal transformation $U$ satisfies
\begin{subequations}
\begin{equation}
\langle\munderbar{U}(\bm{X}) \munderbar{U}(\bm{Y})\rangle = \langle\bm{X}\bm{Y} \rangle\label{eq:definition:orthogonal:outermorphism:transformation}
\end{equation}
and can be written as
\begin{equation}
\munderbar{U}(\bm{A}_k) \equiv (-1)^{ks}\bm{UA}_k\bm{U}^{-1}
\end{equation}
where $\bm{U} = \bm{u}_1\bm{u}_2\cdots\!\ \bm{u}_s$.
\end{subequations}

\end{fornextpaper}

\section{Conformal Geometric Algebra}
\label{sec:CGA}
Conformal Geometric Algebra is the extension of a vector in $\mathcal{A}_{p,q}$ space by adding two extra dimensions. To understand how points in $\mathcal{A}_{p,q}$ relate with points in $\mathcal{A}_{p+1,q+1}$, we define the conformal mapping:
\begin{definition}
The conformal mapping ${c}:\mathcal{A}_{p,q}\mapsto \mathcal{A}_{p+1,q+1}$ is defined as the operation that takes a point in $\mathcal{A}_{p,q}$ and transforms it as:
\begin{subequations}
\begin{equation}
    c(\bm{x}) = \bm{e}_o + \bm{x} + \tfrac{1}{2}\bm{x}^2\bm{e}_{\infty}\label{eq:the:conformal:mapping}
\end{equation}
To define $\bm{e}_o$ and $\bm{e}_{\infty}$ we first introduce the extra two dimensions via two orthogonal vectors $\bm{e}_+$ and $\bm{e}_-$ that satisfy $\bm{e}_+^2 = 1,\  \bm{e}_-^2 = -1,\ \bm{e}_+\cdot \bm{e}_- = 0 $ and $\bm{e}_i\cdot\bm{e}_+ = \bm{e}_i\cdot \bm{e}_- = 0$ for $i=1,2,\dots,p+q$ then 
\begin{equation}
\bm{e}_o = \frac{\bm{e}_- + \bm{e}_+}{\sqrt{2}}, \quad \bm{e}_{\infty} = \frac{\bm{e}_- - \bm{e}_+}{\sqrt{2}}
\end{equation}
\label{eq:eo:einf:conf:mapping:CGA}
\end{subequations}
\end{definition}
The inner product of two points $\bm{x},\bm{y}\in\mathcal{A}_{p,q}$ extended through the conformal mapping (\ref{eq:the:conformal:mapping}) relates to their Euclidean distance as
\begin{subequations}
\begin{equation}
    c(\bm{x})\cdot c(\bm{y}) = -\frac{1}{2}\| \bm{x} - \bm{y}\|^2
\end{equation}
Thus, it is clear that the conformal mapping (\ref{eq:the:conformal:mapping}) maps a point $\bm{x}\in\mathcal{A}_{p,q}$ to a null point in $\mathcal{A}_{p+1,q+1}$, that is
\begin{equation}
    \|c(\bm{x})\|^2 = 0
\end{equation}
Note that the distance between the two conformal points $c(\bm{x})$ and $c(\bm{y})$ is equal to the Euclidean distance between $\bm{x}$ and $\bm{y}$, that is
\begin{equation}
    \|c(\bm{x}) - c(\bm{y})\|^2 = \|c(\bm{x})\|^2 + \|c(\bm{y})\|^2 -2 c(\bm{x})\cdot c(\bm{y}) = \| \bm{x} - \bm{y}\|^2
\end{equation}
\end{subequations}
We will denote the pseudoscalar of $\mathcal{G}_{p,q}$ by $\bm{I}$ and the pseudoscalar of $\mathcal{G}_{p+1,q+1}$ by $\bm{i}$. They relate via
\begin{equation}
\bm{i} \equiv \bm{I}\bm{e}_o\wedge\bm{e}_\infty = \bm{I}\bm{e}_+\bm{e}_-
\end{equation}
In conformal geometric algebra translations are expressed via the rotor $\bm{T} = 1 + \tfrac{1}{2}\bm{e}_\infty \bm{t}$. The translation can be written as
\begin{subequations}
\begin{equation}
T(\bm{x}) = \bm{Tx}\bm{T}^\dagger\label{eq:def:translation}
\end{equation}
We call $\bm{T}$ a translator. The translator $\bm{T}$ is unitary since it satisfies $\bm{TT}^\dagger = 1$. A translation applied to an extended point (\ref{eq:the:conformal:mapping}) yields
\begin{equation}
    T({c}(\bm{x})) = \bm{T}{c}(\bm{x})\bm{T}^\dagger = {c}(\bm{x} + \bm{t})
\end{equation}
Rotation $R$ in the extended conformal space can be expressed exactly in the same manner as in the Euclidean space since
\begin{equation}
R(c(\bm{x})) = \bm{R}c(\bm{x})\bm{R} = c(\bm{RxR}^\dagger) = c(R(\bm{x}))
\end{equation}
\label{eq:rot:trans:emb:relations}
\end{subequations}

The coefficients $\mathdutchbcal{P}_1,\mathdutchbcal{P}_2,\dots,\mathdutchbcal{P}_4\in\mathcal{G}_{p,q}$ of $\bm{P}\in\mathcal{G}_{p+1,q+1}$ can be understood by expressing the multivector $\bm{P}$ via the basis vectors $\bm{e}_\infty$ and $\bm{e}_o$ as 
\begin{equation}
\begin{split}
\bm{P} &\equiv \bm{e}_o\mathdutchbcal{P}_1 + \bm{e}_\infty\mathdutchbcal{P}_2 + \bm{e}_o\wedge \bm{e}_{\infty} \mathdutchbcal{P}_3 + \mathdutchbcal{P}_4  \\
        & = \bm{e}_o\wedge\mathdutchbcal{P}_1 + \bm{e}_\infty\wedge\mathdutchbcal{P}_2 + \bm{e}_o\wedge \bm{e}_{\infty} \wedge\mathdutchbcal{P}_3 + \mathdutchbcal{P}_4
\end{split}
\label{eq:coefficients:of:P}
\end{equation}
To determine the coefficients of a multivector $\bm{X}\in\mathcal{G}_{p+1,q+1}$ we can use the following set of operations
\begin{subequations}
\begin{align}
    \mathdutchcal{C}_1(\bm{X}) &\equiv -P_{\bm{I}}(\bm{e}_\infty \cdot\bm{X})\\
    \mathdutchcal{C}_2(\bm{X}) &\equiv -P_{\bm{I}}(\bm{e}_o\cdot\bm{X})\\
    \mathdutchcal{C}_3(\bm{X}) &\equiv P_{\bm{I}}((\bm{e}_o\wedge\bm{e}_{\infty})\cdot\bm{X})\\
    \mathdutchcal{C}_4(\bm{X}) &\equiv P_{\bm{I}}(\bm{X})
\end{align}
\label{The:coefficients:functions}
\end{subequations}
$\mathdutchcal{C}_i(\bm{X})$ extracts the $i$-th coefficient of the multivector $\bm{X}$. Note that $P_{\bm{I}}$ is given by (\ref{eq:projection:def}) with $\bm{I}$ the unit pseudoscalar of $\mathcal{G}_{p,q}$.

\begin{definition}
The distance $d^2$ between two multivectors $\bm{P}$ and $\bm{Q}$ in $\mathcal{G}_{p+1,q+1}$ is defined as 
\begin{equation}
d^2(\bm{P},\bm{Q}) \equiv  \|\mathdutchcal{C}_3(\bm{P}) - \mathdutchcal{C}_3(\bm{Q})\|^2 + \|\mathdutchcal{C}_4(\bm{P}) - \mathdutchcal{C}_4(\bm{Q})\|^2\label{eq:def:dist:squared:CGA:translation}
\end{equation}
where $\mathdutchcal{C}_i(\bm{X})$ is defined by (\ref{The:coefficients:functions})
\end{definition}

\printProofs[theo:coefficients]
\printProofs[theo:exact:trans]

\begin{fornextpaper}
{
\begin{theoremEnd}{theorem}
\label{theo:coeffs:plus:norm}
Let $p=3$ and $q=0$, consider $\bm{P}$ given by~\eqref{eq:coefficients:of:P:main:body}, then the following holds
\begin{subequations}
\begin{equation}
\begin{split}
\bm{P}^{+\textup{\dag}} &= (\bm{e}_o\mathdutchbcal{P}_1 + \bm{e}_\infty\mathdutchbcal{P}_2 + \bm{e}_o\wedge \bm{e}_{\infty} \mathdutchbcal{P}_3 + \mathdutchbcal{P}_4)^{+\dag} \\
&= -\bm{e}_\infty\mathdutchbcal{P}_1 + -\bm{e}_o\mathdutchbcal{P}_2 - \bm{e}_o\wedge \bm{e}_{\infty} \mathdutchbcal{P}_3 + \mathdutchbcal{P}_4  \\
\end{split}
\label{eq:dagger:plus:of:P}
\end{equation}
furthermore the plus norm can be written with respect to the coefficients of $\bm{P}$ as
\begin{equation}
\|\bm{P}\|^2_+ = \sum_{k=1}^4 \|\mathdutchcal{C}_k(\bm{P})\|^2 = \sum_{k=1}^4 \|\mathdutchbcal{P}_k\|^2
\end{equation}
\end{subequations}
\end{theoremEnd}}
\end{fornextpaper}

\subsection{Extending Euclidean points to CGA }\label{sec:ext:euc:CGA}

Having described the general methods, which can be applied to multivector clouds in CGA, we also propose an approach that can deal with 3D points in Vanilla Geometric Algebra (VGA). Concretely, we consider two point clouds in VGA and extend it to CGA. We then express how the CGA vectors relate with the VGA vectors. We consider the noisy relationship between point clouds in VGA, then understand how that noisy relationship extends to CGA.

Given two point clouds $\bm{x}_i'\in\mathcal{A}_3$ and $\bm{y}_i'\in \mathcal{A}_3$, for $i=1,2,\dots,\ell$, we aim to find the rigid transformation that best aligns both point clouds consisting of a rotator $\bm{R}\in \mathcal{R}$ and a vector $\bm{t}\in\mathcal{A}_3$. Assuming
\begin{equation}
    \bm{y}_{i}' = {R}(\bm{x}_{j_i}') + \bm{t} + \bm{n}_i \label{eq:euclidean:points}
\end{equation}
where $\bm{n}_i$ is assumed to be a measurement noise modeled by a Gaussian distribution. Also, we assume that $j_i$ is some unknown index vector which does not have repeated indices, making it not possible for one-to-many correspondences to exist. In this context of the registration problem, we set $\bm{X}_i = \bm{x}_i$ and $\bm{Y}_i=\bm{y}_i$ in (\ref{eq:covariance:functions:general}) then we can proceed via the method described in Section \ref{sec:gen:approach}.

\begin{definition}
We define the conformal points $\bm{x}_i\in\mathcal{A}_{4,1}$ and $\bm{y}_i\in\mathcal{A}_{4,1}$ as the points extended from the Euclidean points $\bm{x}_i'\in \mathcal{A}_3$ and $\bm{y}_i'\in \mathcal{A}_3$ via the conformal mapping (\ref{eq:the:conformal:mapping}), that is
\begin{subequations}
\begin{align}
    \bm{x}_i = c(\bm{x}_i') &=  \bm{e}_o + \bm{x}_i' + \tfrac{1}{2}(\bm{x}_i')^2\bm{e}_\infty\\
    \bm{y}_i = c(\bm{y}_i') &=  \bm{e}_o + \bm{y}_i' + \tfrac{1}{2}(\bm{y}_i')^2\bm{e}_\infty
\end{align}
\label{eq:conformal:points}
\end{subequations}
\end{definition}

\begin{theorem}
\label{theo:noise:cga:embedding}
When the Euclidean points relate via (\ref{eq:euclidean:points}) then the conformal points (\ref{eq:conformal:points}) will relate via
\begin{subequations}
\begin{equation}
\bm{y}_i = TR(\bm{x}_{j_i}) + \bm{r}_i\label{eq:q:p:noisy:trans:rot}
\end{equation}
where $\bm{r}_i$ is the conformal noise given by
\begin{equation}
\bm{r}_i = \bm{n}_i + \tfrac{1}{2}\left(\bm{n}_i^2 + 2\bm{n}_i\cdot \left(R(\bm{x}_{j_i}') + \bm{t}\right)\right)\bm{e}_\infty
\end{equation}
\end{subequations}
Furthermore, the noise from the view point of $\bm{x}_{j_i}$ is given by
\begin{subequations}
\begin{equation}
\bm{w}_i = \bar{R}\bar{T}(\bm{r}_i) = \bar{R}(\bm{n}_i) + \tfrac{1}{2}(\bm{n}_i^2 + 2\bar{R}(\bm{n}_i)\cdot\bm{x}_{j_i}')\bm{e}_\infty\label{eq:the:other:conformal:noise}
\end{equation}
and
\begin{equation}
\bm{x}_{j_i} = \bar{R}\bar{T}(\bm{y}_i) -\bm{w}_i
\end{equation}
\end{subequations}
When $\bm{n}_i$ is assumed to be correlated Gaussian noise then the $\bm{w}_i$ has a Generalized chi-squared distribution. When $\bm{n}_i$ is Gaussian but uncorrelated then the probability distribution does not depend on the rotation $R$.
\end{theorem}
\begin{proof}
Recall~\eqref{eq:euclidean:points} and~\eqref{eq:conformal:points} then
\begin{equation}
\begin{split}
c(\bm{y}_i') &= c\left(\bm{Rx}_{j_i}'\bm{R}^\dagger + \bm{t} + \bm{n}_i\right)  \\
&=\bm{e}_o + \bm{Rx}_{j_i}'\bm{R}^\dagger + \bm{t} + \bm{n}_i + \tfrac{1}{2}\left(\bm{Rx}_{j_i}'\bm{R}^\dagger + \bm{t} + \bm{n}_i\right)^2\bm{e}_\infty\\
&= \bm{e}_o + \bm{Rx}_{j_i}'\bm{R}^\dagger + \bm{t} + \tfrac{1}{2}\left(\bm{Rx}_{j_i}'\bm{R}^\dagger + \bm{t}\right)^2\bm{e}_\infty + \bm{n}_i + \bm{n}_i\cdot\left(\bm{Rx}_{j_i}'\bm{R}^\dagger + \bm{t}\right)\bm{e}_\infty + \tfrac{1}{2}\bm{n}_i^2\bm{e}_\infty \\
&=c\left(\bm{Rx}_{j_i}'\bm{R}^\dagger + \bm{t} \right) + \bm{n}_i + \bm{n}_i\cdot\left(\bm{Rx}_{j_i}'\bm{R}^\dagger + \bm{t}\right)\bm{e}_\infty + \tfrac{1}{2}\bm{n}_i^2\bm{e}_\infty
\end{split}
\end{equation}
recalling how translations and rotations relate from the embedding space to the embedded space by~\eqref{eq:rot:trans:emb:relations}, we can easily find that $c\left(\bm{Rx}_{j_i}'\bm{R}^\dagger + \bm{t} \right) = TRc(\bm{x}_{j_i}')=TR(\bm{x}_{j_i})$. By also setting $\bm{r}_i = \bm{n}_i + \bm{n}_i\cdot\left(\bm{Rx}_{j_i}'\bm{R}^\dagger + \bm{t}\right)\bm{e}_\infty + \tfrac{1}{2}\bm{n}_i^2\bm{e}_\infty $ we arrive at~\eqref{eq:q:p:noisy:trans:rot}.
To show how~\eqref{eq:the:other:conformal:noise} holds we start by taking $\bar{T}$ to $\bm{r}_i$ thus
\begin{equation}
\begin{split}
\bar{T}(\bm{r}_i) &= \bar{T}(\bm{n}_i) + \tfrac{1}{2}\left(\bm{n}_i^2 + 2\bm{n}_i\cdot \left(R(\bm{x}_{j_i}') + \bm{t}\right)\right)\bar{T}(\bm{e}_\infty)\\
&= \bm{n}_i - \bm{t}\cdot\bm{n}_i\bm{e}_\infty + \tfrac{1}{2}\left(\bm{n}_i^2 + 2\bm{n}_i\cdot \left(R(\bm{x}_{j_i}') + \bm{t}\right)\right)\bm{e}_\infty\\
& = \bm{n}_i + \tfrac{1}{2}\left(\bm{n}_i^2 + 2\bm{n}_i\cdot R(\bm{x}_{j_i}')\right)\bm{e}_\infty\\
& = \bm{n}_i + \tfrac{1}{2}\left(\left(\bar{R}(\bm{n}_i)\right)^2 + 2\bar{R}(\bm{n}_i)\cdot \bm{x}_{j_i}'\right)\bm{e}_\infty
\end{split}
\end{equation}
then applying a rotation $\bar{R}$ to the above result we obtain
\begin{equation}
\begin{split}
\bar{R}\bar{T}(\bm{r}_i) &= \bar{R}(\bm{n}_i) + \tfrac{1}{2}\left(\left(\bar{R}(\bm{n}_i)\right)^2 + 2\bar{R}(\bm{n}_i)\cdot \bm{x}_{j_i}'\right)\bar{R}(\bm{e}_\infty)\\
&=\bar{R}(\bm{n}_i) + \tfrac{1}{2}\left(\left(\bar{R}(\bm{n}_i)\right)^2 + 2\bar{R}(\bm{n}_i)\cdot \bm{x}_{j_i}'\right)\bm{e}_\infty
\end{split}
\end{equation}
note how we can replace $\bm{z}_i = \bar{R}(\bm{n}_i)$ to get
\begin{equation}
\bar{R}\bar{T}(\bm{r}_i) = \bm{z}_i + \tfrac{1}{2}\left(\bm{z}_i^2 + 2\bm{z}_i\cdot \bm{x}_{j_i}'\right)\bm{e}_\infty
\end{equation}
If $\bm{n}_i$ is uncorrelated Gaussian distributed then $\bm{z}_i$ will have the same distribution as $\bm{n}_i$ that does not depend on the rotation $R$. 

Now with respect to the probability distribution of $\bm{r}_i$ first note that the terms $\bm{n}_i$ and $\bm{n}_i\cdot \left(R(\bm{x}_{j_i}') + \bm{t}\right)$ have a Gaussian probability distribution, while the term $\bm{n}_i^2$ has a noncentral chi-squared distribution. 

Note that a generalized chi squared distribution is obtained by the sum of a random variable with a Gaussian probability distribution, that is the terms $\bm{n}_i$ and $\bm{n}_i\cdot \left(R(\bm{x}_{j_i}') + \bm{t}\right)$, and a random variable with a noncentral chi-squared distribution, which means that the overall noise $\bm{r}_i$ must have a generalized chi-squared distribution. However the component of $\bm{r}_i$ in $\mathcal{A}_{p,q}$, that is $\bm{n}_i$, will continue to have a Gaussian probability distribution.

\end{proof}


\section{Multilinear Transformations}
\label{sec:multi:trans}
\begin{definition}
A multilinear transformation $F:\mathcal{G}_{p,q}\mapsto \mathcal{G}_{p,q}$ is a transformation which is linear in its arguments and maps multivectors into multivectors. $F$ is linear when it is equal to its own differential, then given the following definition for the differential $\munderbar{F}$ of $F$ 
\begin{subequations}
\begin{equation}
\munderbar{F}(\bm{X}) = \bm{X}*\partial_{\bm{Y}} F(\bm{Y}) = \bm{X}^\dagger *\partial_{\bm{Y}}^\dagger F(\bm{Y})
\end{equation}
We state that $F$ is linear when 
\begin{equation}
\munderbar{F}(\bm{X}) \equiv F(\bm{X})
\end{equation}
The adjoint of multilinear transformations follows a similar definition to the adjoint of linear transformations, it is defined as
\begin{equation}
\bar{F}(\bm{X}) = \partial_{\bm{Y}}^\dagger \bm{X}^\dagger * \munderbar{F}(\bm{Y}) \label{eq:adjoint:def}
\end{equation}
implying that 
\begin{equation}
\munderbar{F}(\bm{X})*\bm{Y}^\dagger = \bm{X}^\dagger * \bar{F}(\bm{Y}) \label{eq:implication:adjoint}
\end{equation}
\end{subequations}
\end{definition}

Using (\ref{eq:adjoint:def}) and the rules of the derivatives (\ref{eq:dir:deriv}), we show how we obtain (\ref{eq:implication:adjoint})
\begin{equation}
\begin{split}
\bar{F}(\bm{X})*\bm{Y}^\dagger &= (\partial_{\bm{Z}}^\dagger \bm{X}^\dagger * \munderbar{F}(\bm{Z}))*\bm{Y}^\dagger = \partial_{\bm{Z}}^\dagger *\bm{Y}^\dagger \bm{X}^\dagger * \munderbar{F}(\bm{Z})\\
&= \partial_{\bm{Z}} *\bm{Y} \bm{X}^\dagger * \munderbar{F}(\bm{Z}) = \bm{X}^\dagger * \munderbar{F}\left(\bm{Y}*\partial_{\bm{Z}} \bm{Z}\right)\\
& = \bm{X}^\dagger * \munderbar{F}\left(\bm{Y}\right)
\end{split}
\end{equation}

\begin{theorem}
\label{theo:symmetric:multi}
The multilinear function 
\begin{equation}
F(\bm{X}) = \sum_{k=1}^m  \bm{A}_k^{\dagger}\bm{X}\bm{B}_k + \bm{A}_k\bm{X}\bm{B}_k^\dagger
\end{equation}
is symmetric for any $\bm{A}_k,\bm{B}_k\in\mathcal{G}_{p,q}$.
\end{theorem}
\begin{proof}
Considering the relation~\eqref{eq:implication:adjoint}, we then compute 
\begin{equation}
\begin{split}
F(\bm{X})*\bm{Y}^\dagger &= \sum_{k=1}^m\langle \bm{A}_k^{\dagger}\bm{X}\bm{B}_k\bm{Y}^\dagger + \bm{A}_k\bm{X}\bm{B}_k^\dagger\bm{Y}^\dagger\rangle\\
&=\sum_{k=1}^m \langle (\bm{A}_k^{\dagger}\bm{X}\bm{B}_k)^\dagger\bm{Y} + (\bm{A}_k\bm{X}\bm{B}_k^\dagger)^\dagger\bm{Y}\rangle\\
&= \sum_{k=1}^m \langle \bm{B}_k^\dagger\bm{X}^\dagger\bm{A}_k\bm{Y} + \bm{B}_k\bm{X}^\dagger\bm{A}_k^\dagger\bm{Y}\rangle\\
&= \sum_{k=1}^m\langle \bm{X}^\dagger\bm{A}_k\bm{Y}\bm{B}_k^\dagger + \bm{X}^\dagger\bm{A}_k^\dagger\bm{Y}\bm{B}_k\rangle\\
&= \bm{X}^\dagger*F(\bm{Y})
\end{split}
\end{equation}
where we used~\eqref{eq:A:dagger:star:B}, ~\eqref{eq:commutation:property} and~\eqref{eq:dagger:prop:1}.
\end{proof}
\begin{definition}
\label{def:eigenmultivector}
An eigenmultivector $\bm{X}$ of a multilinear function $F$ is a multivector that satisfies 
\begin{equation}
F(\bm{X}) = \lambda \bm{X}\label{eq:eig:FX:lambdaX}
\end{equation}
where $\lambda$ is a scalar.
\end{definition}
A function $F$ which can be expressed via its eigenmultivectors, has a special decomposition. In particular symmetric multilinear functions can always be expressed via the theorem that follows (see Theo.~\ref{theo:multiliner:spectral:decomp}). Note that symmetric multilinear functions satisfy $\bar{F}(\bm{X}) = \munderbar{F}(\bm{X})$. 
\begin{theorem}
\label{theo:multiliner:spectral:decomp}
Let $F$ be a symmetric multilinear function, let $\lambda_1,\lambda_2,\dots,\lambda_m$ be its unique eigenvalues, let $\bm{P}_1,\bm{P}_2,\dots,\bm{P}_m$ be its eigenmultivectors and let $\bm{P}^1,\bm{P}^2,\dots,\bm{P}^m$ be the reciprocal multivectors.
Then, we can write the multilinear function $F$ in the following form
\begin{subequations}
\begin{equation}
F(\bm{X}) = \sum_{k=1}^m \lambda_k \langle \bm{X}\bm{P}^k\rangle \bm{P}_k\label{eq:spectral:decomp}
\end{equation}
with $\langle \bm{P}_i\bm{P}^j \rangle = \delta_{ij}$. The $\bm{P}_i$'s satisfy the eigenvalue equation (\ref{eq:eig:FX:lambdaX}), that is
\begin{equation}
F(\bm{P}_i) = \lambda_i\bm{P}_i
\end{equation}
\end{subequations}
\end{theorem}

\begin{lemma}
\label{lemma:eigendecomp:G=UFU}
Assume that the following relationship holds $G(\bm{Z}) = \munderbar{U}F\bar{U}(\bm{Z})$. Let $F$ have the spectral decomposition as in Theorem \ref{theo:multiliner:spectral:decomp}, then $G$ has the spectral decomposition
\begin{subequations}
\begin{equation}
G(\bm{Z}) = \sum_{k=1}^m \lambda_k \langle \bm{Z}\bm{Q}^k\rangle \bm{Q}_k
\end{equation}
with
\dequations{\bm{Q}_k= s_k\munderbar{U}(\bm{P}_k) = s_k\bm{UP}_k\bm{U}^\dagger}{\bm{Q}^k= s_k^{-1}\munderbar{U}(\bm{P}^k) = s_k^{-1}\bm{UP}^k\bm{U}^\dagger}{eq:ref:Q:s:U(P)}
where $s_k\in\mathbb{R}$.
\end{subequations}
\end{lemma}
\begin{proof}
Given $G(\bm{Z}) = \munderbar{U}F\bar{U}(\bm{Z})$ and assuming that $F$ can be expressed as in (\ref{eq:spectral:decomp}) we have that 
\begin{equation}
\begin{split}
G(\bm{Z}) &= \sum_{k=1}^m \langle \bm{U}^\dagger \bm{ZUP}^k\rangle \bm{UP}_k\bm{U}^\dagger \\ 
&=\sum_{k=1}^m \langle  \bm{ZUP}^k\bm{U}^\dagger\rangle \bm{UP}_k\bm{U}^\dagger \\
&= \sum_{k=1}^m \langle \bm{ZQ}^k\rangle \bm{Q}_k
\end{split}
\end{equation}
with $\bm{Q}_k= s_k\munderbar{U}(\bm{P}_k) = s_k\bm{UP}_k\bm{U}^\dagger$ and $\bm{Q}^k= s_k^{-1}\munderbar{U}(\bm{P}^k) = \bm{UP}^k\bm{U}^\dagger$, with $s_k\in\mathbb{R}$. To conclude we also show that the reciprocity between the eigenmultivectors $\bm{Q}_k$ still holds 
\begin{equation}
\langle \bm{Q}_i\bm{Q}^j \rangle = \langle s_k\bm{UP}_i\bm{U}^\dagger s_k^{-1}\bm{UP}^j\bm{U}^\dagger\rangle = \langle \bm{P}_i\bm{P}^j\rangle = \delta_{ij}
\end{equation}
\end{proof}

\printProofs[cor:scaled:relation]

\begin{theorem}
\label{theo:sign:est:vectors}
Assume that $\bm{X}_i$ and $\bm{Y}_i$ are vectors. Let
\begin{subequations}
\begin{align}
\bm{P}_{\refe} &= (1+\bm{i})(1 + \bar{\bm{X}})\wedge \bm{e}_\infty\\
\bm{Q}_{\refe} &= (1+\bm{i})(1 + \bar{\bm{Y}})\wedge \bm{e}_\infty
\end{align}
\label{eq:refes:from:vectors}
\end{subequations}
where 
\dequations{\bar{\bm{X}} = \tfrac{1}{\ell}\sum_{i=1}^\ell \bm{X}_i,}{\bar{\bm{Y}} = \tfrac{1}{\ell}\sum_{i=1}^\ell \bm{Y}_i}{eq:mean:multiclouds}
and the multivectors $\bm{X}_i$ and $\bm{Y}_i$ satisfy the relationship~\eqref{eq:model:Y=TR(X)+N}, with $\bm{N}_i=0$ then
\begin{equation}
\bm{Q}_{\refe}  = \munderbar{U}(\bm{P}_{\refe})
\end{equation}
\end{theorem}
\begin{proof}
Considering \eqref{eq:model:Y=TR(X)+N}, it is straightforward to show that 
\begin{equation}
\bar{\bm{Y}} = \tfrac{1}{\ell}\sum_{i=1}^\ell \bm{Y}_i = \tfrac{1}{\ell}\sum_{i=1}^\ell \munderbar{U}(\bm{X}_{i}) = \tfrac{1}{\ell}\sum_{i=1}^\ell = \munderbar{U}(\bar{\bm{X}})
\end{equation}
Recall that the point at infinity and the unit pseudoscalar of $\mathcal{G}_{p+1,q+1}$ are invariant to rotation and translations,  that is,  $U(\bm{e}_\infty) = TR(\bm{e}_\infty) = \bm{e}_\infty$ and $\munderbar{U}(\bm{i}) = \bm{UiU}^\dagger = \bm{iUU}^\dagger = \bm{i}$. Then by the outermorphism properties of orthogonal transformations we have
\begin{equation}
\begin{split}
\bm{Q}_{\refe} &= (1+\bm{i})(1 + \bar{\bm{Y}})\wedge \bm{e}_\infty = (\bm{UU}^\dagger+\bm{UiU}^\dagger)(\bm{UU}^\dagger + \bm{U}\bar{\bm{X}}\bm{U}^\dagger)\wedge (\bm{Ue}_\infty\bm{U}^\dagger)  \\
&=\bm{U}\left( (1+\bm{i})(1 + \bar{\bm{Y}})\wedge \bm{e}_\infty \right)\bm{U}^\dagger = \munderbar{U}(\bm{P}_{\refe})
\end{split}
\end{equation}
\end{proof}

\begin{definition}[Multilinear Space]
\label{def:multi:space}
A multilinear space $\mathbb{S}$ is a linear space that spans some multivector basis. Let $\bm{A}_1,\bm{A}_2,\dots,\bm{A}_m\in\mathcal{G}_{p,q}$ be a basis for the multilinear space $\mathbb{S}$, then the projection to the multilinear space $\mathbb{S}$ can be expressed as
\begin{equation}
P_{\mathbb{S}}(\bm{X}) = \sum_{k=1}^m \langle\bm{XA}^k\rangle \bm{A}_k
\end{equation}
where $\bm{A}^k$ are the reciprocals of $\bm{A}_k$, that is, $\bm{A}_i*\bm{A}^j = \delta_{ij}$.
\end{definition}

\begin{fornextpaper}
\section{Versor Theory}\label{sec:versor}
We start by defining the group of versors, in {\it Definition}~\ref{def:versor}, which include elements that act as rotations and translations.
{\it Definition} \ref{def:rotator} refers to the space of versors which act as a rotations in $\mathcal{G}_{p,q}$, {\it Definition} \ref{def:translator} refers to the space of versors which act as translations in $\mathcal{G}_{p,q}$ and {\it Definition} \ref{def:motor} refers to the space of versors which act as rigid transformations in $\mathcal{G}_{p,q}$.

\begin{definition}[Versor]
\label{def:versor}
A versor $\bm{V}$ is a multivector such that the transformation $U(\bm{X}) = \bm{VXV}^\dagger\neq 0$ for all $0\neq \bm{X}\in\mathcal{G}_{p,q}$ and is grade preserving (see {\it Def.} \ref{def:grade:pres}). We denote the group of versors as 
\begin{equation}
\mathcal{V}_{p,q} \equiv \{ 0\neq \bm{X}_r\in\mathcal{G}^r_{p,q},\bm{V}\in\mathcal{G}_{p,q}\ |\ \bm{VX}_r\bm{V}^\dagger\equiv \langle \bm{VX}_r\bm{V}^\dagger\rangle_r\neq 0,\ \text{for}\ r=1,2,\dots,p+q \}
\end{equation}
\end{definition}

To best understand versors and their properties we give a general outline, which is described by the following bullet points
\begin{outline}
\1 A versor $\bm{V}$ is a multivector that can be written as the product of invertible vectors $\bm{V} = \bm{v}_1\bm{v}_2\cdots\bm{v}_k$, with $|\bm{v}_i|\neq 0$.
\1 A versor has the property that $\bm{VV}^\dagger=\langle \alpha\rangle\neq 0$, thus $\bm{V}^{-1} = \frac{\bm{V}^\dagger}{\|\bm{V}\|^2} = \frac{\bm{V}^\dagger}{\bm{VV}^\dagger}$.
\1 The sandwich product of a multivector $\bm{X}=\sum_r\bm{X}_r = \sum_r\langle\bm{X}_r\rangle_r$ with a versor is always grade preserving, {\it i.e.} $\bm{VX}_r\bm{V}^\dagger\equiv\langle \bm{VX}_r\bm{V}^\dagger\rangle_r$.
\1 The simplest type of versor is a vector $\bm{V} = \bm{v}\equiv \langle \bm{v}\rangle_1$.
\1 An even versor can always be written as the product of commuting simple versors, $\bm{V} = \bm{V}_1\bm{V}_2\cdots\bm{V}_s$, where the $\bm{V}_i$ are simple versors and $\bm{V}_i\bm{V}_j = \bm{V}_j\bm{V}_i$, for $i,j=1,\dots s$. 
\1 A simple versor $\bm{V}$ is a special type of versor which can be written as the product of two invertible vectors $\bm{V}=\bm{v}_1\bm{v}_2$ and as the exponential $\bm{V} = \rho \exp(\theta\bm{B})$, where $\bm{B}$ is a $2$-blade.   
\1 An invertible blade $\bm{V} =\bm{v}_1\wedge\bm{v}_2\wedge\cdots\wedge\bm{v}_k$ is a special kind of versor.
\1 Simple versors can describe {\it elliptic}, {\it hyperbolic} and {\it parabolic} simple rotations {\it i.e.} 
\begin{equation}
\bm{V} = \exp(\theta\bm{B}) = 
\begin{cases}
\cos\theta + \bm{B}\sin\theta &\text{if}\ \ \|\bm{B}\|^2 = 1\quad\ \ \text{(elliptical)}\\
\cosh\theta + \bm{B}\sinh\theta &\text{if}\ \ \|\bm{B}\|^2 = -1\quad \text{(hyperbolic)}\\
    1 + \bm{B}\theta &\text{if}\ \  \|\bm{B}\|^2 = 0\quad\ \ \text{(parabolic)}\\
\end{cases}
\end{equation}
\1 Versors describe orthogonal transformations ({\it i.e.} $O(p,q)$ transformations), via the sandwich product $R(\bm{x})=(-1)^k\bm{VxV}^{-1}$ (recall that an orthogonal transformation satisfies $\langle R(\bm{x})\cdot R(\bm{y}) = \bm{x}\cdot\bm{y}$) 
\end{outline}\mbox{}\\

\begin{definition}[Grade Preserving Transformation]
\label{def:grade:pres}
    A grade preserving transformation $F$ is a transformation such that 
    \begin{equation}
    F(\bm{X}_r) \equiv \langle F(\bm{X}_r)\rangle_r,\ \text{for } r=1,2,\dots,n
    \end{equation}
    where $\bm{X}_r\equiv\langle\bm{X}_r\rangle_r$, for each $r=1,2,\dots,n$.
\end{definition}

\begin{definition}[k-Versor]
A $k$-versor is a multivector which can be written as the product of at least $k$ vectors and we denote it by ${\mathcal{V}^k_{p,q} \equiv \{ \bm{V}\in\mathcal{G}_{p+1,q+1}\ |\ \bm{V} = \bm{v}_1\bm{v}_2\cdots\bm{v}_k,\ \text{with}\ \bm{v}_i\in\mathcal{A}_{p,q},\ \bm{v}_i^2\neq 0 \}} $
\end{definition}


\begin{definition}[Even Versor]
\label{def:even:versor}
Versors are all multivectors which can be expressed as the product of an even number of vectors. The group of versors is denoted by $\mathcal{V}^+_{p,q} \equiv \mathcal{V}^2_{p,q}\oplus \mathcal{V}^4_{p,q}\oplus \cdots\oplus \mathcal{V}^s_{p,q}$. Where $s=n$ if $n$ is even and $s=n-1$ if $n$ is odd.
\end{definition}

\begin{definition}[Rotator]
\label{def:rotator}
The group of rotators is denoted by $\mathcal{R}\equiv\{ \bm{R}\in\mathcal{V}^+_{p,q}\ \ |\ \bm{RR}^\dagger = 1 \}$
\end{definition}
\begin{definition}[Translator]
\label{def:translator}
The group of translators is denoted $\mathcal{T}\equiv \{ \bm{t} \in\mathcal{A}_{p,q}\ |\ \bm{T} = 1 + \tfrac{1}{2}\bm{e}_\infty\bm{t} \}$
\end{definition}
\begin{definition}[Motor]
\label{def:motor}
The group of motors is denoted $\mathcal{M} \equiv \{ \bm{M} = \bm{TR}  \ |\ \bm{R} \in\mathcal{R} \ \text{and}\ \bm{T}\in\mathcal{T} \}$.
\end{definition}


\begin{theorem}
Let $p=3$ and $q=0$ then the projection to the group of rotators $\mathcal{R}$ given by {\it Definition} \ref{def:rotator} is given by the following equation
\begin{equation}
P_{\mathcal{R}}(\bm{X}) = \frac{\langle P_{\bm{I}}(\bm{X})\rangle_{0,2}}{\|\langle P_{\bm{I}}(\bm{X})\rangle_{0,2}\|}
\end{equation}
\end{theorem}
\begin{theorem}
\label{theo:motor:multilinear:basis}
When $\bm{M} = \bm{TR}$ is a motor $\mathcal{M}$ then we can determine the translation vector $\bm{t}\in\mathcal{A}_{p,q}$ and the rotator $\bm{R}\in\mathcal{R}$ as
\begin{equation}
\bm{R} = P_{\bm{I}}(\bm{M}),\quad \bm{t} = -2\bm{e}_o\cdot\bm{M}\bm{R}^\dagger\label{eq:trans:from:motor}
\end{equation}
\end{theorem}
\begin{proof}
We write $\bm{M} = \bm{TR} = (1+\tfrac{1}{2}\bm{e}_{\infty}\bm{t})\bm{R} = \bm{R} + \tfrac{1}{2}\bm{e}_{\infty}\bm{t}\bm{R}$, then it is easy to verify that 
\begin{equation}
P_{\bm{I}}(\bm{M}) = P_{\bm{I}}(\bm{R}) + \tfrac{1}{2}P_{\bm{I}}(\bm{e}_{\infty}\bm{t}\bm{R}) = \bm{R}
\end{equation}
Let $\bm{R} = \alpha + \bm{B}$. Then taking the inner product with $\bm{e}_o$ we get 
\begin{equation}
\begin{split}
\bm{e}_o\cdot \bm{M} &= \tfrac{1}{2}\bm{e}_o\cdot(\bm{e}_{\infty}\bm{tR}) = \tfrac{1}{2}\bm{e}_o \cdot (\bm{e}_{\infty}\alpha) + \tfrac{1}{2}\bm{e}_o \cdot (\bm{e}_{\infty}\bm{B})\\
                     &= -\tfrac{1}{2}\bm{t}\alpha - \tfrac{1}{2}\bm{tB} = -\tfrac{1}{2}\bm{tR}
\end{split}
\end{equation}
Multiplying on the left by $-2\bm{R}^\dagger$ we get to the same expression as in the right hand side of (\ref{eq:trans:from:motor}).
\end{proof}

\begin{theorem}
\label{theo:motor:basis:3D:GA}
\begin{subequations}
Let $p=3$ and $q=0$ then the multivector basis for the motors is given by
\begin{equation}
\mathbb{M} = \{1,\bm{Ie}_1,\bm{Ie}_2,\bm{Ie}_3,\bm{e}_\infty\bm{e}_1,\bm{e}_\infty\bm{e}_2,\bm{e}_\infty\bm{e}_3,\bm{I}\bm{e}_\infty \} \label{eq:multi:basis:motor}
\end{equation}
furthermore the motor $\bm{M}\in\mathcal{M}$ can be written in the form 
\begin{equation}
\bm{M} = \alpha + \bm{B} + \bm{e}_\infty(\bm{a} + \beta\bm{I})\label{eq:m:linear:manifold:U:TR}
\end{equation}
where $\alpha$ and $\beta$ are scalars, $\bm{B}$ is a bivector and $\bm{a}$ a vector. Projections to $\mathbb{M}$ can be expressed as
\begin{equation}
P_{\mathbb{M}}(\bm{X}) = \langle P_{\bm{I}}(\bm{X})\rangle_{0,2} - \langle\bm{e}_\infty P_{\bm{I}}(\bm{e}_o\cdot \bm{X})\rangle_{2,4}
\label{eq:projection:motor:multilinearspace}
\end{equation}
thus the space of motors can be succinctly described by 
\begin{equation}
\bm{M}\in\mathcal{M}\ \Leftrightarrow\ \bm{M}\in\mathbb{M}\ \ \text{and}\ \ \bm{MM}^\dagger = 1 
\end{equation}
furthermore we have the equivalent relationship
\begin{equation}
\bm{M}\in\mathbb{M}\ \Leftrightarrow\ \bm{M} = P_{\mathbb{M}}(\bm{M})
\end{equation}
\end{subequations}

\end{theorem}
\begin{proof}
Let $\bm{M} = \bm{TR}$, with $\bm{T} = 1+\tfrac{1}{2}\bm{e}_\infty\bm{t}$ and $\bm{R} = \gamma + \bm{A}$, with $\bm{t}$ a vector and $\bm{A}\in\mathcal{G}_3^2$ a bivector, then
\begin{equation}
    \bm{M} = \left(1+\tfrac{1}{2}\bm{e}_\infty\bm{t}\right)(\gamma + \bm{A}) = \gamma + \bm{A} + \tfrac{1}{2}\bm{e}_{\infty}(\gamma \bm{t} + \bm{t}\cdot\bm{A} + \bm{t}\wedge\bm{A}) = \bm{R} + \tfrac{1}{2}\bm{e}_\infty\bm{tR}
\end{equation}
then compute 
\begin{align*}
P_{\bm{I}}(\bm{M}) &= P_{\bm{I}}(\bm{R}) + P_{\bm{I}}(\bm{e}_\infty)\wedge P_{\bm{I}}(\bm{tR})=\bm{R}\\
P_{\bm{I}}(\bm{e}_o\cdot\bm{M}) &= P_{\bm{I}}(\bm{e}_o\cdot\bm{R}) + P_{\bm{I}}(\bm{e}_o\cdot(\bm{e}_\infty \bm{tR})) = -P_{\bm{I}}(\bm{tR}) =-\bm{tR}
\end{align*}
\end{proof}
\end{fornextpaper}

\section{Solving the Optimal Translation and Rotation Problem}\label{sec:opt:appendix}
\begin{lemma}
The Lagrangian associated with the constraint $\bm{XX}^\dagger = 1$ is given by 
\begin{equation}
\mathcal{L}(\bm{\Lambda},\bm{X}) = \langle \bm{\Lambda}(\bm{ XX}^\dagger - 1)\rangle
\end{equation}
where $\bm{\Lambda}\in\mathcal{G}_{p,q}$ is the Lagrange multiplier.
\end{lemma}

\begin{lemma}[Multilinear Constraint]
\label{theo:multilinear:constraint}
Under the constraint $\bm{X}\in\mathbb{S}$, where $\mathbb{S}$ is a multilinear space defined in {\it Definition} \ref{def:multi:space} ,  the stationary points of $J(\bm{X})$ satisfy the equation
\begin{subequations}
\begin{equation}
P_{\mathbb{S}}FP_{\mathbb{S}}(\bm{X}) = 0
\end{equation}
where $F(\bm{X}) = \partial_{\bm{X}}J(\bm{X})$. As a particular case of this theorem we also state that when $\bm{X}$ is restricted to some grades ${K=(k_1,k_2,\dots,k_s)}$ then the stationary points of $J$ will be the solution to the equation
\begin{equation}
\langle F(\langle \bm{X}\rangle_K)\rangle_K = 0
\end{equation}
\end{subequations}
\end{lemma}

\begin{theorem}
\label{theo:maxi:lagrange:motor}
Let $\mathbb{S}$ and $P_{\mathbb{S}}$ be defined by {\it Definition} \ref{def:multi:space}. The equation associated with the optimization problem
\begin{maxi}|l|
{\bm{X}}{ J(\bm{X}) }
{}{}
\addConstraint{\bm{XX}^\dagger = \pm 1}
\addConstraint{\bm{X}\in\mathbb{S}}{}
\label{opt:versor:manifold:estimation}
\end{maxi}
is given by the multilinear eigenvalue equation
\begin{subequations}
\begin{equation}
P_{\mathbb{S}}FP_{\mathbb{S}}(\bm{X}) = P_{\mathbb{S}}((\bm{\Lambda}+\bm{\Lambda}^\dagger) P_{\mathbb{S}}(\bm{X}))
\end{equation}
where $\bm{\Lambda}$ is the multivector Lagrange multiplier associated with the constraint ${\bm{XX}^\dagger = \pm1}$ and where ${F(\bm{X}) = \partial_{\bm{X}}^\dagger J(\bm{X})}$. Note that the solution is free to chose the sign $s=\pm 1$ for the constraint $\bm{XX}^\dagger = s$.  The constraint is necessary to restrict the quantity $\bm{XX}^\dagger$ to be a non-vanishing scalar. If $\bm{X} = P_{\mathbb{S}}(\bm{X})$ then we may write
\begin{equation}
P_{\mathbb{S}}F(\bm{X}) = P_{\mathbb{S}}((\bm{\Lambda}+\bm{\Lambda}^\dagger)\bm{X})
\end{equation}
\end{subequations}

\end{theorem}

\begin{theorem}\label{theo:versor:R3}
Let $p=3$ and $q=0$ then when $\bm{X}$ is restricted to live in $\mathcal{R}$, then the Lagrange multiplier $\bm{\Lambda}$ associated with the optimization problem
\begin{equation}
\underset{\bm{X}}{\text{maximize}}\ J(\bm{X}), \quad\text{subject to}\ \bm{X}\in\mathcal{R}
\end{equation}
can be just a scalar.
\end{theorem}
\begin{proof}
In $\mathcal{G}_3$ a rotator $\bm{X}\in\mathcal{R}$ belongs to the even subalgebra of $\mathcal{G}_3$, thus $\bm{X}\in\mathcal{G}_3^+=\mathcal{G}_3^0 \oplus\mathcal{G}_3^2$, then we can write $\bm{X} = \alpha + \bm{B}$, where $\alpha\in\mathbb{R}$ and where $\bm{B}\in\mathcal{G}^2_3$ is a bivector. But since $\bm{B}$ lives in $\mathcal{G}_3$, then it is always a blade, being a blade then it squares to a scalar. Thus we write
\begin{equation}
\bm{XX}^\dagger = (\alpha + \bm{B})(\alpha - \bm{B}) = \alpha^2 - \bm{B}^2 + \alpha\bm{B} - \alpha\bm{B} = \alpha^2 - \bm{B}^2 = \left\langle \alpha^2 - \bm{B}^2 \right\rangle
\end{equation}
since the quantity $\bm{XX}^\dagger$ is a scalar then writing $\bm{\Lambda}$ in its different grades parts ${\bm{\Lambda} = \bm{\Lambda}_0+\bm{\Lambda}_1 + \bm{\Lambda}_2 + \bm{\Lambda}_3}$ and considering~\eqref{eq:scalar:diff:grade} we have
\begin{equation}
\langle \bm{\Lambda XX}^\dagger\rangle = \sum_{k=0}^3 \langle\bm{\Lambda}_k \bm{ XX}^\dagger\rangle = \langle\bm{\Lambda}_0 \bm{ XX}^\dagger\rangle
\end{equation}
thus it is always valid that $\bm{\Lambda} = \langle\bm{\Lambda}\rangle = \bm{\Lambda}_0$
\end{proof}

By Theorem \ref{theo:versor:R3} for a rotor $\bm{R}\in\mathcal{G}_{3}$ the Lagrangian associated with the constraint ${\bm{RR}^\dagger = 1}$ is ${\mathcal{L}({\lambda,\bm{R}}) = \lambda\langle \bm{ UU}^\dagger - 1\rangle}$. Where ${\lambda=\langle \lambda\rangle}$.

\begin{fornextpaper}
\begin{theorem}
\label{theo:motor:constraint}
Let $\bm{U}\in \mathcal{M}$ be a multivector constrained to $\mathcal{M}$ then the Lagrange multiplier associated with the constraint $\bm{UU}^\dagger = 1$ is $\bm{\Lambda} = \lambda_1 + \lambda_2\bm{e}_o\bm{I}$ for some scalars $\lambda_1$ and $\lambda_2$, the Lagrangian is then given by 
\begin{equation}
\mathcal{L}(\bm{\Lambda},\bm{U}) = \mathcal{L}(\lambda_1,\lambda_2,\bm{U}) = \left\langle (\lambda_1 + \lambda_2\bm{e}_o\bm{I})(\bm{UU}^\dagger - 1)\right\rangle = \left\langle \bm{\Lambda}(\bm{UU}^\dagger - 1)\right\rangle\label{eq:lagrange:multipliers}
\end{equation}
\end{theorem}
\begin{proof}
Let $\bm{U}$ be given by (\ref{eq:m:linear:manifold:U:TR}) then we compute 
\begin{equation}
\begin{split}
\bm{UU}^\dagger &= \left( \alpha + \bm{B} + \bm{e}_\infty(\bm{a} + \beta\bm{I})\right)\left( \alpha - \bm{B} + (\bm{a} - \beta\bm{I})\bm{e}_\infty\right)\\
                &= \alpha^2 + \|\bm{B}\|^2 + 2\alpha\beta \bm{e}_\infty\bm{I} - \bm{e}_\infty(\bm{Ba} + \bm{aB} - \beta\bm{BI}+\beta\bm{IB})\\
                &= \alpha^2 + \|\bm{B}\|^2 + 2\alpha\beta \bm{e}_\infty\bm{I} - 2\bm{e}_\infty \bm{B}\wedge\bm{a}
\end{split}
\end{equation}
which clearly shows that $\bm{UU}^\dagger$ only has scalar components and components scalar multiples of $\bm{e}_\infty\bm{I}$, which is a 4-blade. Let $\bm{UU}^\dagger = \omega + \gamma\bm{e}_\infty\bm{I}$ then we show that the scalar product with the Lagrange multiplier $\bm{\Lambda} = \lambda_1 + \lambda_2\bm{e}_o\bm{I}$
\begin{equation}
\begin{split}
    \langle \bm{\Lambda UU}^\dagger\rangle &= \langle (\lambda_1 + \lambda_2\bm{e}_o\bm{I})(\omega + \gamma\bm{e}_\infty\bm{I}) \rangle \\
                                           &= \langle \lambda_1\omega + \lambda_1\gamma\bm{e}_\infty\bm{I} + \omega\lambda_2\bm{e}_o\bm{I} + \lambda_2\gamma\bm{e}_o\bm{I}\bm{e}_\infty\bm{I}\rangle\\
                                           &=\langle \lambda_1\omega - \lambda_2\gamma \bm{e}_o\bm{e}_\infty\bm{I}^2\rangle\\
                                           &=\langle \lambda_1\omega + \lambda_2\gamma (\bm{e}_o\cdot\bm{e}_\infty +\bm{e}_o\wedge\bm{e}_\infty) \rangle\\
                                           &=\langle \lambda_1\omega - \lambda_2\gamma \rangle = \lambda_1\omega - \lambda_2\gamma 
\end{split}
\end{equation}
thus for each basis element of $\bm{UU}^\dagger$ we have an associated Lagrange multiplier.
\end{proof}
\end{fornextpaper}


\printProofs[theo:opt:rot:trans]


\begin{toberemoved}
\subsection{A Grade Agnostic Approach to the Optimal Translation Problem}

For the particular case of scalar linear functions we provide an approach to compute the derivative which is agnostic to the grades of the multivectors. In particular consider for example a scalar linear function $\phi(\bm{x})$ 
then instead of using the geometric derivative as is, we instead opt to consider a basis $\bm{e}_i$ and its reciprocal $\bm{e}^i$ to compute the derivative as
\begin{equation}
\bm{u} = \partial_{\bm{x}}\phi(\bm{x}) = \sum_{i=1}^n \bm{e}_i\bm{e}^i\cdot\partial_{\bm{x}} \phi(\bm{x}) = \sum_{i=1}^n \bm{e}_i\phi(\bm{e}^i) \label{eq:deriv:scalar:linear:func}
\end{equation}
Since $\phi$ is linear in its argument then we can move the scalar $\bm{e}^i\cdot\partial_{\bm{x}}$ close to the argument $\bm{x}$. Note how $\bm{u}$ does not depend on $\bm{x}$. Does by considering 
\begin{equation}
\phi(\bm{t}) = \sum_{i=1}^N \langle \bm{t}\cdot(\mathdutchbcal{S}_{i1} + \mathdutchbcal{Q}_{i1})(\mathdutchbcal{S}_{i3} - \mathdutchbcal{Q}_{i3})^\dagger\rangle + \langle \bm{t}\wedge(\mathdutchbcal{S}_{i1} + \mathdutchbcal{Q}_{i1})(\mathdutchbcal{S}_{i4} - \mathdutchbcal{Q}_{i4})^\dagger\rangle
\end{equation}
in (\ref{eq:cost:expanded:opt:trans}) we can use the approach described by (\ref{eq:deriv:scalar:linear:func}) to compute the derivative of $\phi$. As a programmatic approach it is rather useful, yet for understanding the structure of the solutions of the problems at hand it is rather useless. Let $\bm{u} = \partial_{\bm{t}}\phi(\bm{t})$ then the minimizer of (\ref{eq:cost:function:opt:trans}) could be simply given by 
\begin{equation}
\bm{t} = s_N^{-1}\bm{u} = s_N^{-1}\partial_{\bm{t}}\phi(\bm{t})
\end{equation}
where $s_N$ is given by (\ref{eq:sum:squares:opt:trans}).

\section{Robert Valkenburg and Leo Dorst Approach}\label{sec:dorst}

In~\cite{Valkenburg_Dorst_2011} the authors propose an approach which estimates rotations and translations from a variety of geometrical data in CGA. They define a cost measure and then proceed to show that if the multivectors to be registered are admissible normalized multivectors then the minimum, constrained to being a motor, of that cost measure has to be a motor. Note that our approach improves on that by lifting that restriction, that is the input multivectors can be of any type.

We illustrate the authors' approach in a more optimizer friendly point of view. Assume that we have $\bm{P}_1,\bm{P}_2,\dots,\bm{P}_m\in\mathcal{G}_{4,1}$ and $\bm{Q}_1,\bm{Q}_2,\dots,\bm{Q}_m\in\mathcal{G}_{4,1}$ which satisfy
\begin{equation}
\bm{Q}_i = \bm{UP}_i\bm{U}^\dagger + \bm{N}_i
\end{equation}
where $\bm{U}$ is some motor and $\bm{N}_i$ is Gaussian multivector noise. Then we define the cost measure
\begin{equation}
J(\bm{U}) = \tfrac{1}{2}\sum_{k=1}^m w_k \|\bm{UP}_k\bm{U}^\dagger - \bm{Q}_k\|^2 \label{eq:cost:measure:leo:dorst}
\end{equation}
for some scalars $w_k$. Taking the derivative with respect to $\bm{U}^\dagger$ we readily find that 
\begin{equation}
L(\bm{U}) \equiv -\tfrac{1}{2}\partial_{\bm{U}^\dagger} J(\bm{U}) = \sum_{k=1}^m w_k\left(\bm{Q}_k^\dagger\bm{U}\bm{P}_k + \bm{Q}_k\bm{U}\bm{P}_k^\dagger\right)
\end{equation}
Recall that by {\it Theorem}~\ref{theo:motor:constraint} the Lagrangian associated with the problem of minimizing~\eqref{eq:cost:measure:leo:dorst} under the constraint $\bm{U}\in\mathcal{M}$ is given by~\eqref{eq:lagrange:multipliers}. Thus the solution to the optimization problem is given by the solution to the following eigenvalue problem
\begin{equation}
P_{\mathbb{M}}L(\bm{U}) = P_{\mathbb{M}}(\bm{\Lambda U})\label{eq:generalized:ga:eig:prob}
\end{equation}
with $\bm{\Lambda} = \lambda_1 + \lambda_2 \bm{Ie}_o$. The authors of~\cite{Valkenburg_Dorst_2011} provide us with a strategy that makes solving~\eqref{eq:generalized:ga:eig:prob} much more trivial. They consider that the $\bm{P}_k$ and $\bm{Q}_k$ are admissible normalized simple multivectors, by which then they can relax the constraint $\bm{UU}^\dagger = 1$ to just $\langle\bm{UU}^\dagger\rangle = 1$, which makes $\lambda_2 = 0$. Thus equation~\ref{eq:generalized:ga:eig:prob} can be written in the form $P_{\mathbb{M}}L(\bm{U}) = \lambda\bm{U}$ instead. That is, the solution to the eigenvalue problem $P_{\mathbb{M}}L(\bm{U}) = \lambda\bm{U}$ will be the minimizer of $J$ and will live in the space of motors $\mathcal{M}$.
\end{toberemoved}
\begin{fornextpaper}
\section{The Optimal Motor Approach}
\label{sec:MOTOR}

In this section we consider vectors $\bm{p}_1,\bm{p}_2,\dots,\bm{p}_5$ and $\bm{q}_1,\bm{q}_2,\dots,\bm{q}_5$, which forms two sets of orthogonal vectors for the entire vector space $\mathcal{A}_{4,1}$. Furthermore we assume that the magnitude of the $i$'th vectors are the same, that is
\begin{subequations}
\label{eq:orthogonal:set}
\begin{equation*}
\refstepcounter{equation}\latexlabel{firsthalf}
\refstepcounter{equation}\latexlabel{secondhalf}
\bm{p}_i\cdot\bm{p}_j^{-1} = \bm{q}_i\cdot\bm{q}_j^{-1} = \delta_{ij}, \qquad \bm{p}_i^2 = \bm{q}_i^2.
\tag{\ref*{firsthalf}, \ref*{secondhalf}}
\end{equation*}
\end{subequations}
The method that we will describe, takes this sets of vectors to construct a linear function ${H(\bm{x}) = \sum_{i=1}^5 \bm{x}\cdot\bm{p}_i\bm{q}_i^{-1}}$, which by \eqref{eq:orthogonal:set}, defines an orthogonal transformation ({\it Theo.} \ref{eq:theo:ortho}). With the use geometric algebra tools we then proceed to find a decomposition of $H$ in the form ${H(\bm{x})= (-1)^r\bm{VxV}^{-1}}$ with ${\bm{V} \in\mathcal{V}^r_{4,1}}$ an $r$-versor. When considering that the $\bm{p}_i$'s and $\bm{q}_i$'s are eigenvectors of $F$ and $G$ respectively, when in a noise free environment, $H$ will be the rigid transformation that relate the multivectors $\bm{X}_i$ with the multivectors $\bm{Y}_i$. Yet when noise is not disregarded in (\ref{eq:model:Y=TR(X)+N}), then $H$ will describe a general orthogonal transformation, that will closely relate with $U$, thus $\bm{V}$ will express a general versor in $\mathcal{G}_{4,1}$, which is approximate to $\bm{U}$. We are particularly interested in finding versors which describe rigid transformations in CGA. As such we aim to determine a versor $\bm{M}$ that mostly relates to the versor $\bm{V}$ while still being a motor $\bm{M}\in\mathcal{M}$, that is a rotator  $\bm{R}\in\mathcal{R}$ times a translator $\bm{T}\in\mathcal{T}$. 


Concretely, we propose to find the solution to the following optimization problem  
\begin{equation}
\underset{\bm{M}}{\text{maximize}}\ \|\bm{V} - \bm{M} \|_+^2,\quad \text{subject to}\ \bm{M}\in\mathcal{M}\label{opt:Motor:manifold:estimation}
\end{equation}
where $\mathcal{M}$ is the space of motors defined by {\it Definition} \ref{def:motor}, and $\|\cdot\|_+^2$ is defined by {\it Definition}~\ref{def:the:plus:norm}. To understand the motivation of this approach consider the following theorem, which considers the $\bm{p}_i$'s and the $\bm{q}_i$'s as eigenvectors of $F$ and $G$ respectively.

\begin{theoremEnd}[category=theoHTR,no link to proof,proof here]{theorem}
\label{theo:H:TR}
Assume that the relation (\ref{eq:model:Y=TR(X)+N}) is noise free, and let $\bm{p}_i$ and $\bm{q}_i$, $i=1,2\dots,5$ be the eigenvectors of $F$ and $G$ respectively and consider that $s_i=1$ in \eqref{eq:sign}, then the following holds
\begin{equation}
    H(\bm{x}) = \sum_{k=1}^5 \bm{x}\cdot \bm{p}_i\bm{q}_i^{-1} = TR(\bm{x}) = U(\bm{x}) \label{eq:H:eig:relations:CGA}
\end{equation}
\end{theoremEnd}
\begin{proofEnd}
From {\it Theorem}~\ref{theo:eigmvs:noise:free} and since $s_i=1$ we have $\bm{q}_i = U(\bm{p}_i)$. First note the same applies to the inverse of the vectors since $\bm{p}_i^2= \bm{q}_i^2$ then
\begin{subequations}
\begin{equation}
\bm{q}_i^{-1} = \frac{\bm{q}_i}{\bm{q}_i^2} =  \frac{U(\bm{p}_i)}{\bm{q}_i^2} = U\left(\frac{\bm{p}_i}{\bm{p}_i^2}\right) = U(\bm{p}_i^{-1})
\end{equation}
it easily follows that 
\begin{equation}
H(\bm{x}) = \sum_{k=1}^5 \bm{x}\cdot\bm{p}_i\bm{q}_i^{-1} = \sum_{k=1}^5 \bm{x}\cdot\bm{p}_i U(\bm{p}_i^{-1}) = U\left( \sum_{k=1}^5 \bm{x}\cdot\bm{p}_i \bm{p}_i^{-1}\right) = U(\bm{x})
\end{equation}
\end{subequations}
since the $\bm{p}_i$'s form an orthogonal basis for $\mathcal{A}_{4,1}$ then ${\sum_{k=1}^5 \bm{x}\cdot\bm{p}_i\bm{p}_i^{-1} = \bm{x}}$ is the identity transformation.
\end{proofEnd}

While under noise free assumptions $H$ expresses the rigid transformation between the points $\bm{X}_i$ and $\bm{Y}_i$, when we assume that noise is not ignored, then that assumption stops being valid, and that is why we consider the optimization problem \eqref{opt:Motor:manifold:estimation}. The solution to that optimization problem is given by the following theorem 
\begin{theoremEnd}[proof here,no link to proof]{theorem}[Optimal Motor]
\label{theo:optimal:motor:1}
The solution to the optimization problem \eqref{opt:Motor:manifold:estimation}
is given by
\begin{subequations}
\label{opt:motor:manifold:plus:norm}
\begin{equation}
\bm{M} = \alpha \left(\gamma \bm{W}_1 + \bm{IW}_2\right) + \frac{\bm{e}_\infty\bm{W}_2}{1+\gamma^{-1}\alpha^{-1}}
\end{equation}
with
\begin{align}
\alpha &= \pm\left( 4\beta^{-2}\|\bm{W}_2\|^4\|\bm{W}_1\|^2 + \|\bm{W}_2\|^2 - 4\beta^{-1}\|\bm{W}_2\|^2\langle \bm{W}_1\bm{W}_2^\dagger\bm{I}\rangle \right)^{-\tfrac{1}{2}}\label{eq:alpha:M:signed}\\
\gamma &= 2\beta^{-1}\|\bm{W}_2\|^2\\
\beta &= \langle \bm{I}^\dagger (\bm{W}_2\bm{W}_1^\dagger - \bm{W}_1\bm{W}_2^\dagger)\rangle
\end{align}
\end{subequations}
with $\bm{W} = P_{\mathbb{M}}(\bm{V}) = \bm{W}_1 + \bm{e}_\infty\bm{W}_2$, $\bm{W}_1 = P_{\bm{I}}(\bm{V})$ and $\bm{W}_2 = -P_{\bm{I}}(\bm{e}_o\cdot\bm{V})$. Note that the sign in~\eqref{eq:alpha:M:signed} has to be chosen in order for ${\|\bm{V} - \bm{M}\|^2_+}$ to be minimized.
\end{theoremEnd}
\begin{proofEnd}
By {\it Theorem} \ref{theo:motor:basis:3D:GA} the solution to \ref{opt:Motor:manifold:estimation} is found by determining the solution to the following optimization problem 
\begin{maxi}|l|
{\bm{M}}{\|\bm{V} - \bm{M} \|^2_+}
{}{}
\addConstraint{\bm{MM}^\dagger =1}
\addConstraint{\bm{M}\in\mathbb{M}}{}
\label{opt:Motor:manifold:estimation:optidef}
\end{maxi}
where $\mathbb{M}$ is given by \eqref{eq:multi:basis:motor}. Then by {\it Theorem}~\ref{theo:maxi:lagrange:motor} and {\it Theorem}~\ref{theo:motor:constraint} the Lagrangian of the cost function is given by 
\begin{equation}
\mathcal{L}(\lambda_1,\lambda_2,\bm{M}) = \|\bm{V} - \bm{M}\|^2_+ + \left\langle (\lambda_1 + \lambda_2\bm{e}_o\bm{I})(\bm{MM}^\dagger - 1)\right\rangle
\end{equation}

Then using~\eqref{eq:norm:plus:expanded} the derivative with respect to $\bm{M}^\dagger$ gives
\begin{equation}
\partial_{\bm{M}^\dagger} \mathcal{L}(\lambda_1,\lambda_2,\bm{M}) = 2\bm{M}^{+\dagger} + 2\bm{V}^{+\dagger} + 2(\lambda_1 + \lambda_2\bm{e}_o\bm{I})\bm{M} \label{eq:derivative:proof:theo:opt:motor}
\end{equation}
where we used
\begin{equation}
\begin{split}
\partial_{\bm{M}^\dagger} \langle \bm{\Lambda MM}^\dagger \rangle &= \left( \partial_{\bm{M}}\langle \bm{MM}^\dagger\bm{\Lambda}\rangle \right)^\dagger = \left( \partial_{\bm{M}} \bm{M} *(\bm{M}^\dagger\bm{\Lambda}) \right)^\dagger \\ 
& = \left( \dot{\partial}_{\bm{M}} \dot{\bm{M}} *(\bm{M}^\dagger\bm{\Lambda}) \right)^\dagger + \left( \dot{\partial}_{\bm{M}} \dot{\bm{M}} *(\bm{M}^\dagger\bm{\Lambda}^\dagger) \right)^\dagger\\
& = (\bm{M}^\dagger\bm{\Lambda} + \bm{M}^\dagger\bm{\Lambda}^\dagger)^\dagger = (\bm{\Lambda} + \bm{\Lambda}^\dagger) \bm{M}
\end{split}
\end{equation}
with the fact that $\bm{\Lambda}^\dagger = \bm{\Lambda}$. Also note that by using~\eqref{eq:synmmetric:plus:inner:prod} ${\partial_{\bm{M}^\dagger} \bm{M}*\bm{V}^+ = \bm{V}^{+\dagger}}$ and ${\partial_{\bm{M}^\dagger} \bm{M}*\bm{M}^+ = 2\bm{M}^{+\dagger}}$. Let
\begin{equation}
\bm{M} = P_{\mathbb{M}}(\bm{M}) = \bm{M}_1+\bm{e}_\infty\bm{M}_2\label{eq:U:P_M(U)} 
\end{equation}
with $\bm{M}_1 = \langle \bm{M}_1\rangle + \langle \bm{M}_1\rangle_2 \in\mathcal{G}_3$ and $\bm{M}_2 = \langle\bm{M}_2\rangle_1 + \langle \bm{M}_2\rangle_3 \in\mathcal{G}_3$. 
Let $F(\bm{M}) = \partial_{\bm{M}^\dagger} \mathcal{L}(\bm{\Lambda},\bm{M})$, then by {\it Theorem} \ref{theo:multilinear:constraint} and by (\ref{eq:U:P_M(U)}), to find the optimal solution for $\bm{M}\in\mathbb{M}$ we have to find the solution to the following system of equations
\begin{subequations}
\doubleequation{P_{\mathbb{M}}F(\bm{M}) = 0,}{\bm{MM}^\dagger = 1}
\label{eq:sys:eqs:opt:motor}
\end{subequations}
Next we will proceed to show that 
\begin{equation}
\begin{split}
\tfrac{1}{2} P_{\mathbb{M}}F &= P_{\mathbb{M}}(\bm{M}^{+\dagger}) + \overbrace{P_{\mathbb{M}}(\bm{V}^{+\dagger})}^{\bm{W}} + P_{\mathbb{M}}(\bm{\Lambda M})\\
&= \bm{W} + (1+\lambda_1)\bm{M}_1 + (\bm{e}_\infty\lambda_1 + \bm{I}\lambda_2)\bm{M}_2
\end{split}
\label{eq:proj:of:F}
\end{equation}
where we have defined 
\begin{equation}
\bm{W} \equiv P_{\mathbb{M}}(\bm{V}^{+\dagger}) = \bm{W}_1 + \bm{e}_\infty\bm{W}_2\label{eq:proj:of:V+:multi:basis}
\end{equation}
with $\bm{W}_1 = \langle \bm{W}_1\rangle + \langle \bm{W}_1\rangle_2 \in\mathcal{G}_3$ and $\bm{W}_2 = \langle\bm{W}_2\rangle_1 + \langle \bm{W}_2\rangle_3 \in\mathcal{G}_3$. To show how we got~\eqref{eq:proj:of:F} first note that by using~\eqref{eq:dagger:plus:of:P} we have 
\begin{subequations}
\label{eq:proj:of:M}
\begin{equation}
\bm{M}^{+\dagger} = \bm{M}_1 - \bm{e}_o\bm{M}_2\label{eq:dagger:plus:of:M}
\end{equation}
which by projecting to $\mathbb{M}$, using the algebraic definition of $P_{\mathbb{M}}$ given by (\ref{eq:projection:motor:multilinearspace}), gives
\begin{equation}
P_{\mathbb{M}}(\bm{M}^{+\dagger}) = P_{\mathbb{M}}(\bm{M}_1) - P_{\mathbb{M}}(\bm{e}_o\bm{M}_2) = \bm{M}_1
\end{equation}
where we used
\begin{equation}
P_{\mathbb{M}}(\bm{e}_o\bm{M}_2) = \underbrace{P_{\bm{I}}(\bm{e}_o)}_{0}P_{\bm{I}}(\bm{M}_2) -\bm{e}_\infty P_{\bm{I}}(\bm{e}_o\cdot(\bm{e}_o\bm{M}_2)) = -\bm{e}_\infty P_{\bm{I}}(\bm{e}_o^2\bm{M}_2) = 0
\end{equation}
\end{subequations}
Expanding $\bm{\Lambda}$ and using the algebraic definition of $P_{\mathbb{M}}$ given by (\ref{eq:projection:motor:multilinearspace}) we have
\begin{subequations}
\begin{equation}
P_{\mathbb{M}}(\bm{\Lambda M}) = P_{\mathbb{M}}\left((\lambda_1 + \lambda_2\bm{e}_o\bm{I})\bm{M}\right) = \lambda_1 \bm{M} + \lambda_2 \bm{IM}_2 
\label{eq:proj:lambda:U}
\end{equation}
where we computed
\begin{equation}
\begin{split}
P_{\mathbb{M}}(\bm{e}_o\bm{IM}) &= P_{\mathbb{M}}(\bm{e}_o\bm{I}(\bm{M}_1+\bm{e}_\infty\bm{M}_2) = P_{\mathbb{M}}(\bm{e}_o\bm{IM}_1) - P_{\mathbb{M}}(\bm{Ie}_o\bm{e}_\infty \bm{M}_2)\\
&= P_{\mathbb{M}}(\bm{e}_o\bm{IM}_1) - P_{\mathbb{M}}(\bm{Ie}_o\cdot\bm{e}_\infty \bm{M}_2) - P_{\mathbb{M}}(\bm{Ie}_o\wedge\bm{e}_\infty \bm{M}_2)\\
&= P_{\mathbb{M}}(\bm{IM}_2) = \bm{IM}_2
\end{split}
\end{equation}
\end{subequations}
We note that multivectors with $\bm{e}_o$ or $\bm{e}_o\wedge\bm{e}_\infty$ components get annihilated by the projection $P_{\mathbb{M}}$, since ${\bm{M}_1,\bm{M}_2\in\mathcal{G}_3}$.

Equations \eqref{eq:proj:lambda:U} and \eqref{eq:proj:of:M} can then be used to show how we arrive at~\eqref{eq:proj:of:F}. Equations \eqref{eq:sys:eqs:opt:motor} and~\eqref{eq:proj:of:F} tells us that the solution to the optimization problem satisfies
\begin{equation}
(\lambda_1 - 1)\bm{M}_1 + (\bm{e}_\infty\lambda_1 + \bm{I}\lambda_2)\bm{M}_2 = \bm{W}\label{eq:M:W:lambdas:0}
\end{equation}
where we absorbed the minus sign inside the Lagrange multipliers. Taking the inner product with $-\bm{e}_o$ we can immediately determine that 
\begin{equation}
\lambda_1\bm{M}_2 = -\bm{e}_o\cdot\bm{W} = \bm{W}_2\ \Leftrightarrow\ \bm{M}_2 = \lambda_1^{-1}\bm{W}_2
\end{equation}
Replacing this result back into~\eqref{eq:M:W:lambdas:0} and using~\eqref{eq:proj:of:V+:multi:basis} we find that 
\begin{equation}
\begin{split}
& (\lambda_1 - 1)\bm{M}_1 + (\bm{e}_\infty\lambda_1 + \bm{I}\lambda_2)\lambda_1^{-1}\bm{W}_2 = \bm{W}_1 + \bm{e}_\infty\bm{W}_2\ \\
\Leftrightarrow\ \ & \bm{M}_1 = (\bm{W}_1 - \lambda_2\lambda_1^{-1}\bm{IW}_2)(\lambda_1 - 1)^{-1}
\end{split}
\end{equation}
Then it follows that 
\begin{equation}
\bm{M} = \bm{M}_1 + \bm{e}_\infty\bm{M}_2 = \alpha_1 \bm{W}_1 + \alpha_2 \bm{IW}_2 + \alpha_3\bm{e}_\infty\bm{W}_2
\label{eq:M:of:alphas:W1:W2}
\end{equation}
with
\begin{subequations}
\begin{align}
\alpha_1 &= (\lambda_1 - 1)^{-1}\\
\alpha_2 &= -\lambda_2\lambda_1^{-1}(\lambda_1 - 1)^{-1}\\
\alpha_3 &= \lambda_1^{-1}
\end{align}
\label{eq:alphas:lambda}
\end{subequations}
To solve for each $\alpha_i$ we take $\bm{M}$ and then find the $\alpha_i$ such that $\bm{MM}^\dagger$ is a scalar equal to one. First note that the quantity $\alpha_1 \bm{W}_1 + \alpha_2 \bm{IW}_2$ is of grades zero and two, then 
\begin{equation}
(\alpha_1 \bm{W}_1 + \alpha_2 \bm{IW}_2)(\alpha_1 \bm{W}_1 + \alpha_2 \bm{IW}_2)^\dagger = \alpha_1^2\|\bm{W}_2\|^2+\alpha_2^2\|\bm{W}_2\|^2 - 2\alpha_1\alpha_2\langle \bm{W}_1\bm{W}_2\bm{I}\rangle 
\end{equation}
then note that the point at infinity anticommutes with $\bm{W}_2$ and commutes with $\bm{W}_1$ that is $\bm{e}_\infty\bm{W}_1=\bm{W}_1\bm{e}_\infty$ and $\bm{e}_\infty\bm{W}_2=-\bm{W}_2\bm{e}_\infty$. Then it follows that 
\begin{subequations}
\begin{equation}
\begin{split}
\bm{MM}^\dagger &= (\alpha_1 \bm{W}_1 + \alpha_2 \bm{IW}_2 + \alpha_3\bm{e}_\infty\bm{W}_2)(\alpha_1 \bm{W}_1 + \alpha_2 \bm{IW}_2 + \alpha_3\bm{e}_\infty\bm{W}_2)^\dagger \\
&=\alpha_1^2\|\bm{W}_1\|^2+\alpha_2^2\|\bm{W}_2\|^2 - 2\alpha_1\alpha_2\langle \bm{W}_1\bm{W}_2^\dagger\bm{I}\rangle +\\
& + \alpha_3\bm{e}_\infty\bm{W}_2(\alpha_1 \bm{W}_1 + \alpha_2 \bm{IW}_2)^\dagger + (\alpha_1 \bm{W}_1 + \alpha_2 \bm{IW}_2)\alpha_3\bm{W}_2^\dagger\bm{e}_\infty + \alpha_3^2\bm{e}_\infty^2 \|\bm{W}_2\|^2\\
& = \alpha_1^2\|\bm{W}_1\|^2+\alpha_2^2\|\bm{W}_2\|^2 - 2\alpha_1\alpha_2\langle \bm{W}_1\bm{W}_2^\dagger\bm{I}\rangle +\alpha_3\bm{e}_\infty \bm{I} (\alpha_1\beta -2\alpha_2\|\bm{W}_2\|^2) 
\end{split}
\label{eq:MM:dagger:scalars:alphas}
\end{equation}
where we computed
\begin{equation}
\begin{split}
&\alpha_3\bm{e}_\infty\bm{W}_2(\alpha_1 \bm{W}_1 + \alpha_2 \bm{IW}_2)^\dagger + (\alpha_1 \bm{W}_1 + \alpha_2 \bm{IW}_2)\alpha_3\bm{W}_2^\dagger\bm{e}_\infty = \\
 =\  &\alpha_3\bm{e}_\infty \left((\bm{W}_2(\alpha_1 \bm{W}_1 + \alpha_2 \bm{IW}_2)^\dagger - (\alpha_1 \bm{W}_1 + \alpha_2 \bm{IW}_2)\bm{W}_2^\dagger\right)\\
 =\  & \alpha_3\bm{e}_\infty \left( \alpha_1 (\bm{W}_2\bm{W}_1^\dagger - \bm{W}_1\bm{W}_2^\dagger) - 2\alpha_2\bm{I}\|\bm{W}_2\|^2  \right)\\
 =\ &\alpha_3\bm{e}_\infty \bm{I} (\alpha_1\beta -2\alpha_2\|\bm{W}_2\|^2)
\end{split}
\end{equation}
and replaced 
\begin{equation}
\beta\bm{I} \equiv \bm{W}_2\bm{W}_1^\dagger - \bm{W}_1\bm{W}_2^\dagger
\end{equation}
\end{subequations}
then setting~\eqref{eq:MM:dagger:scalars:alphas} equal to one we find that
\begin{subequations}
\begin{equation*}
\refstepcounter{equation}\latexlabel{eq:MM:dagger:scalar:quadvector:a}
\refstepcounter{equation}\latexlabel{eq:MM:dagger:scalar:quadvector:b}
\alpha_1^2\|\bm{W}_1\|^2+\alpha_2^2\|\bm{W}_2\|^2 - 2\alpha_1\alpha_2\langle \bm{W}_1\bm{W}_2^\dagger\bm{I}\rangle = 1, \qquad \alpha_3\bm{e}_\infty \bm{I} (\alpha_1\beta -2\alpha_2\|\bm{W}_2\|^2) = 0
\tag{\ref*{eq:MM:dagger:scalar:quadvector:a}, \ref*{eq:MM:dagger:scalar:quadvector:b}}
\end{equation*}
\label{eq:MM:dagger:scalar:quadvector}
\end{subequations}
Recalling that $\alpha_3 = \lambda_1^{-1}\neq0$ we have 
\begin{equation}
\alpha_1 = 2\beta^{-1}\alpha_2\|\bm{W}_2\|^2
\end{equation}
which by replacing in~\eqref{eq:MM:dagger:scalar:quadvector:a} and solving with respect to $\alpha_2$ we find
\begin{equation}
\alpha_2 = \pm\left( 4\beta^{-2}\|\bm{W}_2\|^4\|\bm{W}_1\|^2 + \|\bm{W}_2\|^2 - 4\beta^{-1}\|\bm{W}_2\|^2\langle \bm{W}_1\bm{W}_2^\dagger\bm{I}\rangle \right)^{-\tfrac{1}{2}}
\label{eq:alpha2:sign:soloution}
\end{equation}
To find the sign which makes the function $\|\bm{V}-\bm{M}\|_+^2$ smaller, rather than bigger we use~\eqref{eq:alphas:lambda} to find a relationship between $\alpha_1$ and $\alpha_3$ as $\alpha_3 = \frac{1}{1+\alpha_1^{-1}}$ then writing $\alpha_1$ and $\alpha_3$ with respect to $\alpha_2\equiv\alpha$ we have
\begin{subequations}
\begin{align}
\alpha_1 &= \gamma\alpha_2 = \gamma\alpha\\
\alpha_3 &= \frac{1}{1+\gamma^{-1}\alpha_2^{-1}} = \frac{1}{1+\gamma^{-1}\alpha^{-1}}
\end{align}
\end{subequations}
with $\gamma =2\beta^{-1}\|\bm{W}_2\|^2$. Replacing back into~\eqref{eq:M:of:alphas:W1:W2} we find that 
\begin{equation}
\bm{M} = \alpha \left(\gamma \bm{W}_1 + \bm{IW}_2\right) + \left(1+\gamma^{-1}\alpha^{-1}\right)^{-1}\bm{e}_\infty\bm{W}_2
\end{equation}
Then we recall~\eqref{eq:norm:plus:expanded} to write $\|\bm{V}-\bm{M}\|_+^2$ in the form
\begin{equation}
\begin{split}
\|\bm{V}-\bm{M}\|_+^2 &= \|\bm{V}\|^2_+ + \langle\bm{M}\bm{M}^+\rangle - 2\langle \bm{MV}^+\rangle \\
                    &= \|\bm{V}\|^2_+ + \left\langle\left(\alpha \left(\gamma \bm{W}_1 + \bm{IW}_2\right) + \frac{\bm{e}_\infty\bm{W}_2}{1+\gamma^{-1}\alpha^{-1}}\right)\left(\alpha \left(\gamma \bm{W}_1^\dagger + \bm{W}_2^\dagger\bm{I}\right) - \frac{\bm{W}_2^\dagger\bm{e}_o}{1+\gamma^{-1}\alpha^{-1}}\right)\right\rangle \\
                    &-2\langle \alpha \left(\gamma \bm{W}_1 + \bm{IW}_2\right)\bm{V}^+ + \left(1+\gamma^{-1}\alpha^{-1}\right)^{-1}\bm{e}_\infty\bm{W}_2\bm{V}^+\rangle\\
                    & = \|\bm{V}\|^2_+ + \alpha^2\|\gamma \bm{W}_1 + \bm{IW}_2\|^2 + \frac{\|\bm{W}_2\|^2}{\left(1+\gamma^{-1}\alpha^{-1}\right)^2} - 2\alpha\langle \left(\gamma \bm{W}_1 + \bm{IW}_2\right)\bm{V}^+\rangle -2\frac{\langle\bm{e}_\infty\bm{W}_2\bm{V}^+\rangle}{1+\gamma^{-1}\alpha^{-1}}\\
                    &= c_1 + \alpha c_2 + \alpha^2 c_3 + \frac{c_4}{\left(1+\gamma^{-1}\alpha^{-1}\right)^2} + \frac{c_5}{1+\gamma^{-1}\alpha^{-1}}
\end{split}
\end{equation}
To find the sign in~\eqref{eq:alpha2:sign:soloution} which minimizes the above we multiply it by ${\left(1+\gamma^{-1}\alpha^{-1}\right)^2=1 + \gamma^{-2}\alpha^{-2} + 2\gamma^{-1}\alpha^{-1}} > 0$ then
\begin{equation}
\begin{split}
&(1 + \gamma^{-2}\alpha^{-2} + 2\gamma^{-1}\alpha^{-1})\left(c_1 + \alpha c_2 + \alpha^2 c_3\right) + c_4 + \left(1+\gamma^{-1}\alpha^{-1}\right)c_5\\
=\ & \alpha^{-2} \gamma^{-2}c_1 + \alpha^{-1}\left( \gamma^{-1} c_5 + \gamma^{-2}c_2 + 2\gamma^{-1}c_1 \right) + 1 + c_1 + c_4+\gamma^{-2}c_3 + 2\gamma^{-1} c_2 + \alpha (c_2 + 2\gamma^{-1} c_3) + \alpha^2 c_3
\end{split}
\end{equation}
the only components which contribute with a sign are the odd powers of $\alpha$, thus we are interested in finding the sign 
which minimizes 
\begin{equation}
\alpha^{-1}\left( \gamma^{-1} c_5 + \gamma^{-2}c_2 + 2\gamma^{-1}c_1 \right)+ \alpha (c_2 + 2\gamma^{-1} c_3)
\end{equation}
The choice of the sign will depend on the value of $\beta$ and on the value of what is inside of the parenthesis of~\eqref{eq:alpha2:sign:soloution}.
\end{proofEnd}
The $\alpha_i$'s are scalars such that $\bm{M}$ is a motor that satisfies $\bm{MM}^\dagger = 1$.

To summarize, with this approach we take the eigenvectors $\bm{P}_i$ and $\bm{Q}_i$ of $F$ and $G$, form the orthogonal transformation $H$, compute the versor $\bm{V}$ of $H$, estimate the motor $\widehat{\bm{M}}$ that best aligns with $\bm{V}$ by solving the optimization problem \eqref{opt:Motor:manifold:estimation}, and then extract the rotor $\bm{R}$ and the translator $\bm{T}$ from $\widehat{\bm{M}}$. This strategy is illustrated step by step in {\it Algorithm} \ref{alg:rbm:opt:motor}.

\begin{algorithm}[H]
\caption{Rigid Transformation Estimation (Optimal Motor Approach)}
\label{alg:rbm:opt:motor}
{\bf Input Data:} The eigenvectors $\bm{p}_1,\bm{p}_2,\dots,\bm{p}_5, \bm{q}_1,\bm{q}_2,\dots,\bm{q}_5$ of $F$ and $G$ respectively\;
Define the multilinear function $H(\bm{x}) = \sum_{i=1}^5 \bm{x}\cdot\bm{p}_i\bm{q}_i^{-1}$\;
Determine the versor $\bm{V}$ associated with the linear function $H$\;
Project $\bm{V}$ to the multilinear space $\mathbb{M}$ as $\bm{W} =\bm{W}_1 + \bm{e}_\infty\bm{W}_2 = P_{\mathbb{M}}(\bm{V})$\;
Compute the versor $\widehat{\bm{M}} = \underset{\bm{M}\in\mathcal{M}}{\text{argmax}}\ \|\bm{V} - \bm{M} \|_+^2$ (See ~\eqref{opt:motor:manifold:plus:norm})\;
Extract the rotation and the translation from $\widehat{\bm{M}}$ as $\widehat{\bm{R}} = P_{\bm{I}}(\widehat{\bm{M}})$ and $\widehat{\bm{T}} = 1 - \bm{e}_\infty\bm{e}_o\cdot\widehat{\bm{M}}\widehat{\bm{R}}^\dagger$\;
\end{algorithm}

\end{fornextpaper}
\begin{fornextpaper}
\section{Rotation Only Approach}
\label{eq:rot:only:approach}
As noted in Sec.~\ref{sec:Bkg-Mot}, the multivector coefficients proposal deals with bivectors or trivectors. 
In the particular case when $\bm{P}_i$ and $\bm{Q}_i$ are either vectors or quadvectors, the rotation components will be scalar or pseudoscalar $\bm{I}$ of $\mathcal{G}_3$, respectively. This is straightforward. Assuming $\bm{P}$, $\bm{Q}$ of grade $r$, and applying $\mathdutchbcal{P}_1 = \mathdutchcal{C}_1(\bm{P})$, $\mathdutchbcal{Q}_1 = \mathdutchcal{C}_1(\bm{Q})$ (recall \eqref{eq:coefficients:CiTRP:CiRTQ}), then $\mathdutchbcal{Q}_1$, $\mathdutchbcal{P}_1$ become of grade $r-1$. In this scenario, when $r=2$ then $\mathdutchbcal{P}_1$ and $\mathdutchbcal{Q}_1$ become scalars, when $r=4$ then $\mathdutchbcal{P}_1$ and $\mathdutchbcal{Q}_1$ become psudoscalars of $\mathcal{G}_{3}$, both of which are unaffected by rotations, thus not being able to estimate the rigid transformation. Under this scenario we consider the following:

 Assume we are given a set of multivectors $\bm{X}'_i\in\mathcal{G}_{4,1}$ and $\bm{Y}_i'\in\mathcal{G}_{4,1}$ which are vectors plus quadvectors. Then this approach takes this elements and projects them into 3D euclidean space $\mathcal{G}_3$ by projecting vectors in $\mathcal{G}_{4,1}$ into vectors in $\mathcal{G}_{3}$ and quadvectors in $\mathcal{G}_{4,1}$ into bivectors in $\mathcal{G}_3$, thus

\begin{subequations}
\doubleequation{\bm{X}_i = P_{\bm{I}}(\bm{X}_i') + \bm{I}P_{\bm{I}}(\bm{iX}_i'),}{\bm{Y}_i = P_{\bm{I}}(\bm{Y}_i') + \bm{I}P_{\bm{I}}(\bm{iY}_i')}
\label{eq:proj:to:euclidean}
\end{subequations}

Note that when $\bm{X}_i'$ and $\bm{Y}_i'$ relate via a rigid transformation in $\mathcal{G}_{4,1}$ then $\bm{X}_i$ and $\bm{X}_i$ will relate as
\begin{equation}
\bm{Y}_i' = \munderbar{T}\munderbar{R}(\bm{X}_i') + \bm{N}_i\ \ \Leftrightarrow\ \ \bm{Y}_i = \bm{RX}_i\bm{R}^\dagger + \bm{\Gamma}_i\bm{t} + \bm{W}_i \label{eq:proj:euc:space}
\end{equation}
where $\bm{\Gamma}_i  = \langle\bm{\Gamma}_i\rangle + \langle\bm{\Gamma}_i\rangle_3$ and with
\begin{subequations}
\begin{equation*}
\refstepcounter{equation}\latexlabel{eq:alphai}
\refstepcounter{equation}\latexlabel{eq:betai}
\langle\bm{\Gamma}_i\rangle = 
\begin{cases}
0 & \text{if}\ \langle \bm{X}_i\rangle_1 = \langle \bm{Y}_i\rangle_1 = 0\\
1& \text{otherwise}
\end{cases}, \qquad 
\langle\bm{\Gamma}_i\rangle_3 = 
\begin{cases}
0& \text{if}\ \langle \bm{X}_i\rangle_2 = \langle \bm{Y}_i\rangle_2 = 0\\
\bm{I}& \text{otherwise}
\end{cases}
\tag{\ref*{eq:alphai}, \ref*{eq:betai}}
\end{equation*}
\end{subequations}
Note that if $\langle \bm{X}_i\rangle_1$ is zero then $\langle \bm{Y}_i\rangle_1$ must also be zero, the same applies to $\langle \bm{X}_i\rangle_2$ and $\langle \bm{Y}_i\rangle_2$.

\subsection{Determining the Rotation from the Eigendecomposition}
\label{eq:rot:eig:rot:only:PCA}
We define two functions which have the equivariant property with respect to rotations
\begin{subequations}
\label{eq:functions:rot:only}
\begin{equation*}
\refstepcounter{equation}\latexlabel{eq:functions:rot:only:x}
\refstepcounter{equation}\latexlabel{eq:functions:y}
f(\bm{z}) = \sum_{i=1}^\ell (\bm{X}_i - \bm{\Gamma}_i\bar{\bm{X}})\bm{z}(\bm{X}_i - \bm{\Gamma}_i\bar{\bm{X}})^{-1}, \qquad g(\bm{z}) = \sum_{i=1}^\ell (\bm{Y}_i - \bm{\Gamma}_i\bar{\bm{Y}})\bm{z}(\bm{Y}_i - \bm{\Gamma}_i\bar{\bm{Y}})^{-1}
\tag{\ref*{eq:functions:rot:only:x}, \ref*{eq:functions:y}}
\end{equation*}
where 
\doubleequation[eq:mean:VGA:multiclouds:x,eq:mean:VGA:multiclouds:y]{\bar{\bm{X}} = s_{\ell}^{-1}\sum_{k=1}^\ell \langle\bm{\Gamma}_k^\dagger\bm{X}_k\rangle_1,}{\quad\bar{\bm{Y}} = s_{\ell}^{-1}\sum_{k=1}^\ell \langle\bm{\Gamma}_k^\dagger\bm{Y}_k\rangle_1}
with $s_\ell = \sum_{k=1}^\ell \|\bm{\Gamma}_k\|^2$.
\end{subequations}
\begin{theorem}
\label{theo:RfR}
Let $\bm{Y}_i' = \munderbar{T}\munderbar{R}(\bm{X}_i')$ then $f$ and $g$ defined in~\eqref{eq:functions:rot:only} satisfy the relation
\begin{equation}
g(\bm{z}) = \munderbar{R}f\bar{R}(\bm{z})
\end{equation}
\end{theorem}
\begin{proof}[proof of Theorem~\ref{theo:RfR}]
Assume that $\bm{W}_i = 0$, then consider the right hand side of~\eqref{eq:proj:euc:space}, and compute 
\begin{equation}
\langle\bm{\Gamma}_k^\dagger\bm{Y}_k\rangle_1 = \left\langle \bm{\Gamma}_k^\dagger\left( \bm{RX}_k\bm{R}^\dagger + \bm{\Gamma}_k\bm{t} \right)  \right\rangle = \left\langle \bm{R}\bm{\Gamma}_k^\dagger\bm{X}_k \bm{R}^\dagger + \bm{\Gamma}_k^\dagger\bm{\Gamma}_k\bm{t}\right\rangle = \bm{R}\langle \bm{\Gamma}_k^\dagger\bm{X}_k\rangle_1\bm{R}^\dagger + \|\bm{\Gamma}_k\|^2\bm{t}
\end{equation}
then 
\begin{equation}
\begin{split}
\bar{\bm{Y}} &= s_\ell^{-1}\sum_{k=1}^\ell \langle\bm{\Gamma}_k^\dagger\bm{Y}_k\rangle_1 = s_\ell^{-1}\bm{R}\left(\sum_{k=1}^\ell\langle \bm{\Gamma}_k^\dagger\bm{X}_k\rangle_1\right)\bm{R}^\dagger + s_\ell^{-1}\bm{t}\sum_{k=1}^\ell \|\bm{\Gamma}_k\|^2\\
& = \bm{R}\bar{\bm{X}}\bm{R}^\dagger + \bm{t}
\end{split}
\end{equation}
then we show that
\begin{equation}
\begin{split}
\bm{Y}_k - \bm{\Gamma}_k\bar{\bm{Y}} &= \bm{RX}_k\bm{R}^\dagger + \bm{\Gamma}_k\bm{t} - \bm{\Gamma}_k\left(\bm{R}\bar{\bm{X}}\bm{R}^\dagger + \bm{t}\right)\\
&= \bm{RX}_k\bm{R}^\dagger - \bm{\Gamma}_k\bm{R}\bar{\bm{X}}\bm{R}^\dagger = \bm{R}\left( \bm{X}_k - \bm{\Gamma}_k\bar{\bm{X}}\right)\bm{R}^\dagger
\end{split}
\end{equation}
and we finally show that 
\begin{equation}
\begin{split}
g(\bm{z}) &= \sum_{i=1}^\ell (\bm{Y}_i - \bm{\Gamma}_i\bar{\bm{Y}})\bm{z}(\bm{Y}_i - \bm{\Gamma}_i\bar{\bm{Y}})^{-1}\\
&= \sum_{i=1}^\ell \bm{R}\left( \bm{X}_i - \bm{\Gamma}_i\bar{\bm{X}}\right)\bm{R}^\dagger \bm{z} (\bm{R}\left( \bm{X}_k - \bm{\Gamma}_k\bar{\bm{X}}\right)\bm{R}^\dagger)^{-1}\\
& = \sum_{i=1}^\ell \bm{R}\left( \bm{X}_i - \bm{\Gamma}_i\bar{\bm{X}}\right)\bm{R}^\dagger \bm{z}\bm{R} \left( \bm{X}_k - \bm{\Gamma}_k\bar{\bm{X}}\right)^{-1}\bm{R}^\dagger\\
& = \bm{R} f(\bm{R}^\dagger \bm{zR})\bm{R}^\dagger = \munderbar{R}f\bar{R}(\bm{z})
\end{split}
\end{equation}
\end{proof}
Consider that $\bm{p}_1,\bm{p}_2,\bm{p}_3$ and $\bm{q}_1,\bm{q}_2,\bm{q}_3$ are the eigenvectors of $f$ and $g$ respectively. Recall that from {\it Lemma}~\ref{lemma:eigendecomp:G=UFU} and from {\it Theorem}~\ref{theo:RfR} that
\begin{equation}
\bm{q}_i = s_i R(\bm{p}_i)\label{eq:rot:sign:rot:approach}
\end{equation}
for some scalar $s_i$, then to determine the rotation consider the following two theorems

\begin{lemma}\label{lemma:plane:rotation}
The plane of rotation $\bm{B}$ that satisfies the equation
\begin{equation}
\bm{Ra}\bm{R}^\dagger = \bm{b}
\end{equation}
must satisfy
\begin{equation}
(\bm{a} - \bm{b})\wedge\bm{B} = 0
\end{equation}
where $\bm{R} = e^{\theta\bm{B}}$.
\end{lemma}
\begin{proof}
Let the equality $\bm{RaR}^\dagger = \bm{b}$ hold for some vector $\bm{a}$ and some vector $\bm{b}$ then using~\eqref{eq:proj:rej:rotation} we find that 
\begin{equation}
\bm{R}^2P(\bm{a}) + P_\perp(\bm{a}) = P(\bm{b}) + P_\perp(\bm{b})
\end{equation}
then since $P_\perp(\bm{a})$ and $P_\perp(\bm{b})$ are orthogonal to the plane of rotation, then they must be equal, thus
\begin{align}
P_\perp(\bm{a}) &= P_\perp(\bm{b})\label{eq:Pperp:relation}\\
\bm{R}^2P(\bm{a}) &= P(\bm{b})
\end{align}
then we note that since $P_\perp(\bm{x}) = \bm{x}\wedge\bm{BB}^{-1}$ then from~\eqref{eq:Pperp:relation} follows that 
\begin{equation}
\bm{a}\wedge\bm{BB}^{-1} = \bm{b}\wedge\bm{BB}^{-1}\ \Leftrightarrow\ \bm{a}\wedge\bm{B} = \bm{b}\wedge\bm{B}\ \Leftrightarrow\ (\bm{a} - \bm{b})\wedge\bm{B} = 0
\end{equation}
thus $\bm{a}-\bm{b}$ is in the plane of rotation $\bm{B}$ of $\bm{R}$.
\end{proof}

\begin{theorem}
    \label{theo:exact:rotation}
    Assume that the scalar $s_i$ in (\ref{eq:rot:sign:rot:approach}) is known and equal to one. Then we can find a unique rotor $\bm{R}$ via
    \begin{equation}
        \bm{R}^2 = \left(\bm{q}_1 - P_\perp(\bm{p}_1) \right)\left[P(\bm{p}_1) \right]^{-1}\label{eq:Rsq:of:p1:q1}
    \end{equation}
    where $P(\bm{x}) = \bm{x}\wedge\bm{n}\bm{n}^{-1}$ and $P_\perp(\bm{x}) = \bm{x}\cdot\bm{nn}^{-1}$ are the projection to the plane of rotation and the orthogonal projection away from the plane of rotation. $\bm{n}$ is the normal to the plane of rotation and is given by:
    \begin{equation}
    \bm{n} = \bm{I}(\bm{p}_1 - \bm{q}_1)\wedge(\bm{p}_2 - \bm{q}_2)
    \end{equation}
\end{theorem}
\begin{proof}
\begin{subequations}
Since we are aiming to find a rotor $\bm{R}$, use~\eqref{eq:proj:rej:rotation}, then
\begin{align}
    \bm{R}\bm{p}_1\bm{R}^\dagger &= \bm{q}_1\ \Leftrightarrow\ \bm{R}^2P(\bm{p}_1) + P_\perp(\bm{p}_1) = \bm{q}_1\label{eq:Rp1:q1}\\
    \bm{R}\bm{p}_2\bm{R}^\dagger &= \bm{q}_2\ \Leftrightarrow\ \bm{R}^2P(\bm{p}_2) + P_\perp(\bm{p}_2) = \bm{q}_2
\end{align}
from Lemma~\ref{lemma:plane:rotation} we readily find that the plane of rotation $\bm{B}$ must satisfy
\begin{equation}
(\bm{p}_1 - \bm{q}_1)\wedge\bm{B} = 0 \quad\text{and}\quad (\bm{p}_2 - \bm{q}_2)\wedge\bm{B}
\end{equation}
Thus $\alpha\bm{B} = (\bm{p}_1 - \bm{q}_1)\wedge (\bm{p}_2 - \bm{q}_2)$, for some $\alpha\in\mathbb{R}$, is a valid plane of rotation as long as $\bm{p}_1 - \bm{q}_1\neq\bm{p}_2 - \bm{q}_2\neq 0$. Let 
\begin{equation}
\bm{n} = \bm{I}(\bm{p}_1 - \bm{q}_1)\wedge (\bm{p}_2 - \bm{q}_2) = \bm{IB}
\end{equation}
then 
\begin{align}
P(\bm{x}) = \bm{x}\cdot\bm{BB}^{-1} = \bm{x}\cdot(\bm{I}^\dagger\bm{n})(\bm{I}^\dagger\bm{n})^{-1} = \bm{x}\wedge\bm{nn}^{-1}\\
P_\perp(\bm{x}) = \bm{x}\wedge\bm{BB}^{-1} = \bm{x}\wedge(\bm{I}^\dagger\bm{n})(\bm{I}^\dagger\bm{n})^{-1} = \bm{x}\cdot\bm{nn}^{-1}
\end{align}
\end{subequations}
Finally solving the right hand side of~\eqref{eq:Rp1:q1} with respect to $\bm{R}^2$ gives us~\eqref{eq:Rsq:of:p1:q1}.
\end{proof}
 
\begin{theorem}
    \label{eq:rotation:spectral:decomp}
    Define the function
    \begin{equation}
        H(\bm{x}) = \sum_{i=1}^3\bm{x}\cdot\bm{p}_i\bm{q}_i^{-1}\label{eq:H(x):pi:qi}
    \end{equation}
    then $H=R$ is the rotation from $\bm{p}_i$ to $\bm{q}_i$. The rotor is determined by first finding an eigenvector $\bm{a}$ of $H_+(\bm{x}) = \frac{1}{2}(\munderbar{H}(\bm{x}) + \bar{H}(\bm{x}))$ with eigenvalue of the smallest magnitude. Then the rotor can be determined via
    \begin{equation}
        \bm{R}^2 = H(\bm{a})\bm{a}^{-1}
    \end{equation}
\end{theorem}
\begin{proof}
Let $\bm{q}_i = \bm{Rp}_i\bm{R}^\dagger$ and $\bm{p}_i^2 = \bm{q}_i^2$, thus $\bm{q}_i^{-1} = \bm{Rp}_i^{-1}\bm{R}^\dagger$. Then we find from~\eqref{eq:H(x):pi:qi} that  
\begin{equation}
H(\bm{x}) = \sum_{i=1}^3 \bm{x}\cdot\bm{p}_i\bm{Rp}_i^{-1}\bm{R} = \bm{R}\left(\sum_{i=1}^3\bm{x}\cdot\bm{p}_i\bm{p}_i^{-1} \right)\bm{R}^\dagger = \bm{RxR}^\dagger
\end{equation}
where since $\bm{p}_i$ is an orthogonal basis for $\mathcal{A}_3$ we have $\sum_{i=1}^3 \bm{x}\cdot\bm{p}_i\bm{p}_i^{-1} = \bm{x}$. Let $\bm{R}=e^{\theta\bm{B}}=\alpha + \bm{A} = \cos(\theta/2) + \bm{B}\sin(\theta/2)$, then we show that 
\begin{equation}
\begin{split}
H(\bm{x})& =\bm{RxR}^\dagger = (\alpha + \bm{A})\bm{x}(\alpha - \bm{A}) =  \alpha^2\bm{x} - \bm{AxA} + \alpha(\bm{Ax} - \bm{xA})\\
&= \alpha^2 \bm{x} - (2\bm{x}\cdot\bm{AA}^\dagger -\bm{x}\|\bm{A}\|^2) + 2\alpha\bm{A}\cdot\bm{x} = (\alpha^2 + \|\bm{A}\|^2)\bm{x} - 2\bm{x}\cdot\bm{AA}^\dagger -2 \bm{x}\cdot\bm{A}\\
&= \bm{RR}^\dagger\bm{x} - 2\bm{x}\cdot\bm{AA}^\dagger -2 \bm{x}\cdot\bm{A} = \bm{x} - 2\bm{x}\cdot\bm{AA}^\dagger -2 \bm{x}\cdot\bm{A}
\end{split}
\end{equation}
where we used ${\bm{RR}^\dagger = (\alpha + \bm{A})(\alpha - \bm{A}) = \alpha^2 + \bm{AA}^\dagger = \alpha^2+ \|\bm{A}\|^2 = 1}$, ${\tfrac{1}{2}(\bm{Ax}-\bm{xA}) = \bm{A}\cdot\bm{x}}$ and ${\bm{x}\cdot\bm{AA} = \tfrac{1}{2}\left(\bm{xA} - \bm{Ax} \right)\bm{A} = \tfrac{1}{2}\bm{xA}^2 - \tfrac{1}{2}\bm{AxA}}\ \Leftrightarrow\ \bm{AxA} = -\bm{x}\|\bm{A}\|^2 + 2\bm{x}\cdot\bm{AA}^\dagger$. The adjoint of $H$ is computed as 
\begin{equation}
\bar{H}(\bm{x}) = \bm{R}^\dagger\bm{x}\bm{R} = \bm{x} - 2\bm{x}\cdot\bm{AA}^\dagger + 2\bm{x}\cdot\bm{A}
\end{equation}
thus the symmetric transformation $H_+(\bm{x}) = \frac{1}{2}(\munderbar{H}(\bm{x}) + \bar{H}(\bm{x}))$ is written as
\begin{equation}
H_+(\bm{x}) = \bm{x} - 2\bm{x}\cdot\bm{AA}^\dagger = \cos\theta\bm{x}\cdot\bm{AA}^{-1} + \bm{x}\wedge\bm{AA}^{-1}
\end{equation}
where we used $\bm{x} = \bm{x}\cdot\bm{AA}^{-1} + \bm{x}\wedge\bm{AA}^{-1}$ and
\begin{equation}
\begin{split}
    \bm{x}\cdot\bm{AA}^{-1} - 2\bm{x}\cdot\bm{AA}^\dagger &= \|\bm{A}\|^{-2}(1 - 2\|\bm{A}\|^{2})\bm{x}\cdot\bm{AA}^\dagger\\
    &= (1-2\sin^2(\theta/2))\bm{x}\cdot\bm{AA}^{-1} = \cos\theta\bm{x}\cdot\bm{AA}^{-1}
\end{split}
\end{equation}
with $\|\bm{A}\|^2 = \sin^2(\theta/2)$. Thus the eigenblades of $H_+$ are $\bm{A}$ with eigenvalue $\cos\theta$ and $\bm{IA}$ with eigenvalue one.
\end{proof}
To determine $\bm{R}$ after knowing $\bm{R}^2$ we need to estimate the square root of a rotor. When expressed via an exponential $\bm{R} = e^{\theta\bm{B}}$ we can easily compute the square root as $\bm{R} = \sqrt{\bm{R}^2} = \sqrt{e^{2\theta\bm{B}}} = e^{\theta\bm{B}}$.

\subsection{Multivector Correspondences in VGA}\label{sec:mv:corrs:VGA}
Assume that $\bm{Y}_i$ and $\bm{X}_i$ relate via~\eqref{eq:proj:euc:space} then consider the problem of minimizing the cost function
\begin{equation}
J(\bm{R},\bm{t}) = \sum_{k=1}^\ell \|\bm{RX}_k\bm{R}^\dagger + \bm{\Gamma}_k\bm{t} - \bm{Y}_k\|^2
\label{eq:VGA:cost:func}
\end{equation}
\begin{theorem}
The minimizer with respect to $\bm{t}$ of the cost function~\eqref{eq:VGA:cost:func} is given by 
\begin{equation}
\bm{t} = \bar{\bm{Y}} - \bm{R}\bar{\bm{X}}\bm{R}^\dagger\label{eq:trans:opt:sol:VGA}
\end{equation}
with $\bar{\bm{X}}$ and $\bar{\bm{Y}}$ given by~\eqref{eq:mean:VGA:multiclouds:x} and~\eqref{eq:mean:VGA:multiclouds:y} respectively. The rotation is determined by forming the multilinear function
\begin{equation}
F(\bm{R}) = \sum_{k=1}^\ell (\bm{Y}_k - \bm{\Gamma}_k\bar{\bm{Y}})^\dagger\bm{R}(\bm{X}_k - \bm{\Gamma}_k\bar{\bm{X}}) + (\bm{Y}_k - \bm{\Gamma}_k\bar{\bm{Y}})\bm{R}(\bm{X}_k - \bm{\Gamma}_k\bar{\bm{X}})^\dagger\label{eq:F:VGA:eig:prob}
\end{equation}
and then solving the eigenvalue problem
\begin{equation}
\langle F(\bm{R})\rangle_{0,2} = \lambda \bm{R}
\end{equation}
where $\bm{R}\equiv \langle\bm{R}\rangle + \langle\bm{R}\rangle_2$, is a versor. The optimal $\bm{R}$ is given as the eigenrotator of $\langle F(\bm{R})\rangle_{0,2}$ with the largest eigenvalue.
\end{theorem}
\begin{proof}
Expanding the cost function gives 
\begin{equation}
\sum_{k=1}^\ell \|\bm{RX}_k\bm{R}^\dagger - \bm{Y}_k\|^2 + \|\bm{\Gamma}_k^2\bm{t}^2 + 2\langle \bm{t\Gamma}_k^\dagger(\bm{RX}_k\bm{R}^\dagger - \bm{Y}_k)\rangle
\end{equation}
Then the derivative of the cost function with respect to $\bm{t}$ is given by
\begin{equation}
\partial_{\bm{t}} J = 2\bm{t} \sum_{k=1}^\ell \|\bm{\Gamma}_k\|^2 + 2\sum_{k=1}^\ell \langle\bm{\Gamma}_k^\dagger(\bm{RX}_k\bm{R}^\dagger - \bm{Y}_k)\rangle_1
\end{equation}
solving $\partial_{\bm{t}} J = 0$ we easily arrive at~\eqref{eq:trans:opt:sol:VGA}. Replacing~\eqref{eq:trans:opt:sol:VGA} back into the cost function we find that 
\begin{equation}
J(\bm{R}) = \|\bm{R}(\bm{X}_k - \bm{\Gamma}_k\bar{\bm{X}})\bm{R}^\dagger - \bm{Y}_k + \bm{\Gamma}_k\bar{\bm{Y}}\|^2
\end{equation}
which by {\it Theorem}~\ref{theo:rot:opt} the minimum is given as the eigenrotator with largest eigenvalue of $\langle F(\bm{R})\rangle_{0,2}$ where $F$ is given by~\eqref{eq:F:VGA:eig:prob}.
\end{proof}

\subsection{Multivector Correspondences in CGA}
\label{eq:rot:only:approach:with:known:corr}
Assume that $\bm{Y}_i = \munderbar{R}(\bm{X}_i) + \bm{N}_i$ for $i=1,2,\dots,\ell$, then the optimal rotation between the two sets of multivectors $\bm{X}_1,\bm{X}_2,\dots,\bm{X}_\ell$ and $\bm{Y}_1,\bm{Y}_2,\dots,\bm{Y}_\ell$ in Conformal Geometric Algebra, where we know that $\bm{X}_i$ corresponds to $\bm{Y}_i$, is the maximizer of the Lagrangian
\begin{equation}
\mathcal{L}(\bm{R}) = \sum_{j=1}^\ell \sum_{i=1}^4 \|\mathdutchcal{C}_i(\bm{Y}_j) - \mathdutchcal{C}_i(\bm{RX}_j\bm{R}^\dagger)\|^2 + \lambda \bm{RR}^\dagger
\end{equation}
then by Theorem \ref{theo:rot:opt} the optimal rotor $\bm{R}$ satisfies the eigenvalue equation
\begin{equation}
    \sum_{j=1}^\ell \sum_{i=1}^4 \langle \mathdutchcal{C}_i(\bm{Y}_j^\dagger) \bm{R}\mathdutchcal{C}_i(\bm{X}_j)\rangle_{0,2} + \langle \mathdutchcal{C}_i(\bm{Y}_j) \bm{R}\mathdutchcal{C}_i(\bm{X}_j^\dagger)\rangle_{0,2}  = \lambda \bm{R}
\end{equation}
\end{fornextpaper}
\begin{fornextpaper}
\section{Eigendecomposition Algorithm}
\label{sec:eig:decomp:alg}
Here we provide important insights regarding the Eigenmultivector Extraction described in {\it Sec.}~\ref{sec:gen:approach}. Namely, there are some nuances regarding the decomposition of multilinear functions. First and foremost when considering a geometric algebra $\mathcal{G}_{p+1,q+1}$ we will have a basis of $2^{p+q+2}$ multivectors which gets pretty large for increasing $n$. The second nuance is that we have to convert the multilinear transformation into matrix, this can be done by choosing a basis and a reciprocal basis and computing all the coefficients for each basis and reciprocal basis (see Sec.~\ref{sec:conv:matrices}). For the particular case of the multilinear functions we are considering we can simplify the problem of solving for the entire geometric algebra by relating eigenmultivectors by duality and by showing that for grade preserving multilinear transformations we can separate the problem into the different grade parts. Which means that in practice, with $p=3$ and $q=0$, instead of solving for $2^5-2$ (scalar and pseudoscalars do not matter) basis elements we only have to solve for ${5 \choose 1} =5$ basis elements for the eigenvectors and for ${5 \choose 2} =10$ basis elements for the eigenbivectors.

In the following section we show that for~\eqref{eq:covariance:functions:general:F} we can always find an eigenmultivector that is dual to another eigenmultivector.
\subsection{Dual eigenmultivectors}
\label{sec:dual:eigmvs}
Assume that $\bm{P}_k$ are eigenmultivectors of the multilinear function $F$ given by~\eqref{eq:covariance:functions:general:F}. Then,  $\bm{iP}_k$ are also eigenmultivectors of $F$, having
$F(\bm{\alpha P}_k) = \bm{\alpha P}_k$ for any $\bm{\alpha}\equiv \langle \bm{\alpha}\rangle + \langle \bm{\alpha}\rangle_{m}$, where $m=p+q+2$. This is easily shown by assuming that for some multivector $\bm{Z}$ we have $F(\bm{Z}) = \lambda\bm{Z}$, then it follows that  
\begin{equation}
F(\bm{iZ}) = \sum_{i=1}^\ell \bm{X}_i\bm{iZ}\bm{X}_i = \bm{i}\sum_{i=1}^\ell \bm{X}_i\bm{Z}\bm{X}_i = \bm{i}F(\bm{Z}) = \lambda\bm{iZ}
\end{equation}
that is $F(\bm{iZ}) = \lambda \bm{iZ}$. Then $\bm{\alpha P}_k$ is an eigenmultivector for any `complex' value $\bm{\alpha}$. 

\subsection{Grade preserving transformation}
\label{sec:grade:pres}
The following theorem shows that finding the decomposition for each grade is equivalent to finding the decomposition for all grades, assuming $F$ is grade preserving. 

\begin{theorem}
The decomposition of a grade preserving multilinear function is equivalent to the decomposition of its different grade parts. Let $F$ be a grade preserving multilinear transformation, ${F(\langle \bm{Z}\rangle_k)\equiv \langle F(\langle \bm{Z}\rangle_k)\rangle_k}$, then the sum of the decompositions of $\langle F(\bm{Z}_k)\rangle_k$ where $\bm{Z}_k$ is a $k$-vector for each $k$, is equivalent to the decomposition of  $F$.
\end{theorem}
\begin{proof}
Let $\bm{Z} = \sum_k \bm{Z}_k$, with $\bm{Z}_k\equiv\langle\bm{Z}_k\rangle_k$. Assume $F$ to be grade preserving multilinear function, then
\begin{equation}
F(\bm{Z}) = F\left( \sum_{k}\bm{Z}_k\right) = \sum_k\langle F(\bm{Z}_k)\rangle_k
\end{equation}
Now recall that multivectors of different grades are orthogonal under the scalar product,  that is,  $\bm{Z}_i*\bm{Z}_j = 0$ for $i\neq j$. Then, each grade part of $F$ lives in a different orthogonal space. Also, let $\bm{P}_{kj}$ be the eigen-$k$-vectors of $F_k(\bm{Z}) = \langle F(\bm{Z})\rangle_k$, then we may write
\begin{equation}
F_k(\bm{Z}) = \sum_{i=1}^{m_k} \lambda_{ki}\bm{P}_{ki}\langle\bm{ZP}_k^i\rangle\label{eq:decomp:k:grade:part}
\end{equation}
where $\bm{P}_{ki}*\bm{P}_{k}^j = \delta_{ij}$ for $i,j=1,2,\dots,m_k$. Then the $\bm{P}_{ki}$'s are also eigenmultivectors of $F$. 
To prove this consider $F$ as the sum of the decomposition of the $F_k$, {\it i.e.}
\begin{equation}
F(\bm{Z}) = \sum_k F_k(\bm{Z}) = \sum_k F_k(\langle \bm{Z}\rangle_k) = \sum_k\sum_{i=1}^{m_k} \lambda_{ki}\bm{P}_{ki}\langle\bm{ZP}_k^i\rangle
\end{equation}
It easily follows that $F(\bm{P}_{ki}) = \lambda_{ki}\bm{P}_{ki}$ since $\bm{P}_{ki}*\bm{P}_{lj} = \delta_{ijkl}$.
\end{proof}

Consider the particular case where $p=3$ and $q=0$ then recall that vectors are dual to quadvectors and bivectors are dual to trivectors. As such by what was explained in Sec.~\ref{sec:dual:eigmvs} we can reduce the problem of solving the eigenvalue problem for the entire geometric algebra, but to only consider grades one and two. We can simplify the problem even further by recalling that $F$ and $G$ are grade preserving transformations. Which, by what was said in Sec.~\ref{sec:grade:pres} we can consider $F(\langle\bm{X}\rangle_i)$ for each grade $i$ separately. This enables us to decrease the complexity of the decomposition from $\left({5 \choose 1} + {5\choose 2}\right)^2 = 15^2 = 225$ to the complexity ${5\choose 1}^2 + {5 \choose 2}^2 = 5^2 + 10^2 = 125$ which is almost half in terms of complexity.

\subsection{Computing the Matrix of the Transformation}
\label{sec:conv:matrices}
Multiple softwares exist that enable the computation of the eigendecomposition of matrices, yet we have described decompositions of multilinear functions in Geometric Algebra. To be able to solve this kind of problems with standard software we need to translate the language of linear and multilinear transformations to matrices. We define matrices of a linear and multilinear transformations by defining a basis $\bm{A}_1,\bm{A}_2,\dots,\bm{A}_m$ with reciprocal elements $\bm{A}^1,\bm{A}^2\dots,\bm{A}^m$, thus $\bm{A}_i*\bm{A}^j = \delta_{ij}$, then assuming that $F$ is a general multilinear function we compute the matrix of $F$ as
\begin{equation}
\mathrm{F}_{ij} = F(\bm{A}_i)*\bm{A}^j
\end{equation}
then consider the orthogonal vectors $\bm{e}_1,\bm{e}_2,\dots,\bm{e}_m$ in some geometric algebra of $m$ dimension, such that ${\bm{e}_i\cdot\bm{e}_j = \delta_{ij}}$, then define the linear function
\begin{equation}
f(\bm{x}) = \sum_{i,j} \mathrm{F}_{ij}\bm{e}_j\bm{e}_i\cdot\bm{x}\label{eq:f:of:matrix:of:F}
\end{equation}
\begin{theorem}
Let $\bm{u}$ be an eigenvector of $f$ with a real eigenvalue $\lambda$, then $\bm{U} = \sum_i\bm{u}\cdot\bm{e}_i\bm{A}_i$ is an eigenmultivector of $F$ with eigenvalue $\lambda$.
\end{theorem}
\begin{proof}
First note that the following holds
\begin{equation}
\bm{u} = \sum_i \bm{U}*\bm{A}^i \bm{e}_i \label{eq:u:in:terms:of:U}
\end{equation}
replace $\bm{U} = \sum_i\bm{u}\cdot\bm{e}_i\bm{A}_i$ in the above to show that 
\begin{equation}
\sum_i \bm{U}*\bm{A}^i \bm{e}_i = \sum_i \left(\sum_j\bm{u}\cdot\bm{e}_j\bm{A}_j \right)*\bm{A}^i \bm{e}_i = \sum_{i,j} \bm{A}_j*\bm{A}^i\bm{u}\cdot\bm{e}_j\bm{e}_i  = \sum_i\bm{u}\cdot\bm{e}_i\bm{e}_i = \bm{u}
\end{equation}
where we used the reciprocity relation between the $\bm{A}_i$'s and $\bm{A}^j$'s. Now note that by using~\eqref{eq:u:in:terms:of:U} we find
\begin{equation}
f(\bm{u}) = f\left(\sum_i \bm{U}*\bm{A}^i \bm{e}_i\right) = \sum_i \bm{U}*\bm{A}^i f(\bm{e}_i) = \sum_j\sum_i \bm{U}*\bm{A}^i \mathrm{F}_{ij}\bm{e}_j\label{eq:f:expandend}
\end{equation}
where we used~\eqref{eq:f:of:matrix:of:F} to find that 
\begin{equation}
f(\bm{e}_i) = \sum_{j,k} \mathrm{F}_{jk}\bm{e}_k\bm{e}_j\cdot\bm{e}_i = \sum_k \mathrm{F}_{ik}\bm{e}_k = \sum_j F_{ij}\bm{e}_j
\end{equation}
then setting $f(\bm{u})=\lambda\bm{u} = \lambda\sum_i \bm{U}*\bm{A}^j\bm{e}_j$ in~\eqref{eq:f:expandend} we have
\begin{equation}
\begin{split}
\sum_j\sum_i \bm{U}*\bm{A}^i \mathrm{F}_{ij}\bm{e}_j &= \lambda\sum_i \bm{U}*\bm{A}^j\bm{e}_j\\
\sum_i \bm{U}*\bm{A}^i \mathrm{F}_{ij} &= \lambda \bm{U}*\bm{A}^j,\ \text{for}\ i=1,2,\dots,m\\
\sum_j\sum_i \bm{U}*\bm{A}^i \mathrm{F}_{ij}\bm{A}_j &= \lambda \sum_j\bm{U}*\bm{A}^j\bm{A}_j\\
F(\bm{U}) &= \lambda\bm{U}
\end{split}
\end{equation}
Where we used the fact that the matrix $\mathrm{F}_{ij}$ of $F$ can be used to write $F$ as
\begin{equation}
F(\bm{X}) = \sum_{i,j}  \mathrm{F}_{ij} \bm{A}_j\bm{A}^i*\bm{X}
\end{equation}
this can be verified by computing 
\begin{equation}
\begin{split}
F(\bm{X}) &= F\left( \sum_j\bm{A}_j\bm{A}^j*\bm{X} \right) = \sum_j F(\bm{A}_j) \bm{A}^j*\bm{X} = \sum_{j}\left[ \sum_i F(\bm{A}_j)*\bm{A}^i \bm{A}_i \right]\bm{A}^j*\bm{X}\\
& = \sum_{i,j} \underbrace{F(\bm{A}_i)*\bm{A}^j}_{\mathrm{F}_{ij}} \bm{A}_j\bm{A}^i*\bm{X}
\end{split}
\end{equation}
Note that we are assuming that $\bm{X}\equiv\sum_i \bm{X}*\bm{A}^j\bm{A}_j$.
\end{proof}
The above theorem showed the equivalence between solutions to multilinear eigenvalue problems and vector linear eigenvalue problems. Thus provided we have the matrix $\mathrm{F}_{ij}$ of $F$ we can determine the decomposition of $F$. In particular we will consider symmetric transformations since these have always real eigenvalues in standard algorithms. 
\end{fornextpaper}
\begin{fornextpaper}
\section{Algorithms}
\begin{algorithm}[H]
\caption{Versor Computation (Symmetric Approach)}   
\label{alg:versor:computation}
{\bf Input Data:} The linear function $H$\;
Form the symmetric function $H_+(\bm{x}) = \tfrac{1}{2}(\munderbar{H}(\bm{x}) + \bar{H}(\bm{x}))$\;
Find the eigenvectors $\bm{a}_1,\bm{a}_2,\dots,\bm{a}_5$ of $H$\;
Compute the eigenvalues of $H_+$ as $\bm{\lambda}_i = H(\bm{a}_i)\bm{a}_i^{-1}$\;
Form a blade with all the eigenvectors that have $\bm{\lambda}_{j_i} = -1$, $\bm{A} = \bm{a}_{j_1}\wedge\bm{a}_{j_2}\wedge\cdots\wedge\bm{a}_{j_s}$\;
Form a versor with all the other eigenvalues taking $\bm{V} = \sqrt{\bm{\lambda}_{k_1}}\sqrt{\bm{\lambda}_{k_2}}\cdots\sqrt{\bm{\lambda}_{k_t}}$\;
Return the product of $\bm{V}$ with $\bm{A}$, $\bm{U} = \bm{VA}$ and a sign $(-1)^{s}$\; 
\end{algorithm}

\begin{algorithm}[H]
\caption{Versor Computation (Skew-Symmetric Approach) }   
\label{alg:versor:computation:skew}
{\bf Input Data:} The linear function $H$\;
Compute the bivector $\bm{B}$ associated with the skew symmetric part of $H$, as $\bm{B} = -\tfrac{1}{2}\partial\wedge H$\ \;
Form the symmetric linear function $G(\bm{x}) = (\bm{x}\cdot\bm{B})\cdot\bm{B}^\dagger$\ \;
Find the eigenvectors $\bm{a}_1,\bm{a}_2,\dots,\bm{a}_5$ with associated eigenvalues $\alpha_1,\alpha_2,\dots,\alpha_5$ of $G$ \;
For each eigenvector with associated nonzero eigenvalue $\alpha_i$ compute the eigenvalue of $H$ as $\bm{\lambda}_i = H(\bm{a}_i)\bm{a}_i^{-1}$\;
Form the versor by taking the product of the square-roots of the 'complex` eigenvalues $\bm{V} = \sqrt{\bm{\lambda}_{k_1}}\sqrt{\bm{\lambda}_{k_2}}\cdots\sqrt{\bm{\lambda}_{k_t}}$\ \;
{\bf Return} the versor $\bm{V}$ 
\end{algorithm}

\begin{algorithm}[H]
\caption{Versor Computation (Manual Approach)}
{\bf Input Data:} The linear function $H:\mathcal{A}_{4,1}\rightarrow \mathcal{A}_{4,1}$\;
Compute the bivector $\bm{B}$ associated with the skew symmetric part of $H$, as $\bm{B} = -\tfrac{1}{2}\partial\wedge H$\ \;
Find the decomposition of $\bm{B}=\bm{B}_1 + \bm{B}_2$\;
if $\bm{B}_1^2=0$:
    then $\bm{R}_1 = 1 + (1/2)\bm{B}_1$\;
else:
    chose random vector $\bm{v}_1\in\mathcal{A}_{4,1}$\;
    Compute $\bm{u}_1 = \bm{v}_1\cdot\bm{B}_1\bm{B}_1^{-1}$\;
    Determine the rotor $\bm{R}_1 = \sqrt{H(\bm{v}_1)\bm{v}_1^{-1}}$\;

chose random vector $\bm{v}_2\in\mathcal{A}_{4,1}$\;
Compute $\bm{u}_2 = \bm{v}_2\cdot\bm{B}_2\bm{B}_2^{-1}$\;
Determine the rotor $\bm{R}_2 = \sqrt{H(\bm{v}_2)\bm{v}_2^{-1}}$\;
{\bf Return:} $\bm{U} = \bm{R}_1\bm{R}_2$
\end{algorithm}
\end{fornextpaper}

\end{appendices}


\bibliography{fundamentals}

\begin{thebibliography}{63}
\providecommand{\natexlab}[1]{#1}
\providecommand{\url}[1]{{#1}}
\providecommand{\urlprefix}{URL }
\providecommand{\doi}[1]{\url{https://doi.org/#1}}
\providecommand{\eprint}[2][]{\url{#2}}
 \bibcommenthead

\bibitem[{Agarwal et~al(2011)Agarwal, Furukawa, Snavely, Simon, Curless, Seitz, and Szeliski}]{agarwal2011building}
Agarwal S, Furukawa Y, Snavely N, et~al (2011) Building rome in a day. Communications of the ACM 54(10):105--112

\bibitem[{Aoki et~al(2019)Aoki, Goforth, Srivatsan, and Lucey}]{aoki2019pointnetlk}
Aoki Y, Goforth H, Srivatsan RA, et~al (2019) Pointnetlk: Robust \& efficient point cloud registration using pointnet. In: CVPR

\bibitem[{Artin(2016)}]{artin2016geometric}
Artin E (2016) Geometric algebra. Courier Dover Publications

\bibitem[{Bai et~al(2020)Bai, Luo, Zhou, Fu, Quan, and Tai}]{bai2020d3feat}
Bai X, Luo Z, Zhou L, et~al (2020) D3feat: Joint learning of dense detection and description of 3d local features. In: CVPR

\bibitem[{Barath and Matas(2018)}]{barath2018graph}
Barath D, Matas J (2018) Graph-cut ransac. In: ICCV

\bibitem[{Besl and McKay(1992)}]{besl1992method}
Besl PJ, McKay ND (1992) Method for registration of 3-d shapes. In: Sensor fusion IV: control paradigms and data structures, Spie, pp 586--606

\bibitem[{Bustos and Chin(2017)}]{bustos2017guaranteed}
Bustos AP, Chin TJ (2017) Guaranteed outlier removal for point cloud registration with correspondences. IEEE PAMI 40(12):2868--2882

\bibitem[{Celik and Ma(2008)}]{celik2008fast}
Celik T, Ma KK (2008) Fast object-based image registration using principal component analysis for super-resolution imaging. In: 2008 5th International Conference on Visual Information Engineering (VIE 2008)

\bibitem[{Chen et~al(2022)Chen, Sun, Yang, and Tao}]{chen2022sc2}
Chen Z, Sun K, Yang F, et~al (2022) Sc2-pcr: A second order spatial compatibility for efficient and robust point cloud registration. In: CVPR

\bibitem[{Choy et~al(2019)Choy, Park, and Koltun}]{choy2019fully}
Choy C, Park J, Koltun V (2019) Fully convolutional geometric features. In: ICCV

\bibitem[{Choy et~al(2020)Choy, Dong, and Koltun}]{choy2020deep}
Choy C, Dong W, Koltun V (2020) Deep global registration. In: CVPR

\bibitem[{Clifford(1871)}]{clifford1871preliminary}
Clifford (1871) Preliminary sketch of biquaternions. Proceedings of the London Mathematical Society 1(1):381--395

\bibitem[{Dang et~al(2022)Dang, Wang, Guo, and Salzmann}]{dang2022learning}
Dang Z, Wang L, Guo Y, et~al (2022) Learning-based point cloud registration for 6d object pose estimation in the real world. In: ECCV

\bibitem[{Deng et~al(2018)Deng, Birdal, and Ilic}]{deng2018ppf}
Deng H, Birdal T, Ilic S (2018) Ppf-foldnet: Unsupervised learning of rotation invariant 3d local descriptors. In: ECCV

\bibitem[{Deschaud(2018)}]{deschaud2018imls}
Deschaud JE (2018) Imls-slam: Scan-to-model matching based on 3d data. In: ICRA, IEEE

\bibitem[{Dorst and Mann(2002)}]{dorst2002geometric}
Dorst L, Mann S (2002) Geometric algebra: a computational framework for geometrical applications. IEEE Computer Graphics and Applications 22(3):24--31

\bibitem[{Dorst et~al(2007)Dorst, Fontijne, and Mann}]{dorst_geometric_2007}
Dorst L, Fontijne D, Mann S (2007) Geometric algebra for computer science: an object-oriented approach to geometry. Elsevier ; Morgan Kaufmann, Amsterdam, San Francisco, \urlprefix\url{http://geometricalgebra.net/}, oCLC: 182548505

\bibitem[{Fischler and Bolles(1981)}]{fischler1981random}
Fischler MA, Bolles RC (1981) Random sample consensus: a paradigm for model fitting with applications to image analysis and automated cartography. Communications of the ACM 24(6):381--395

\bibitem[{Frome et~al(2004)Frome, Huber, Kolluri, B{\"u}low, and Malik}]{frome2004recognizing}
Frome A, Huber D, Kolluri R, et~al (2004) Recognizing objects in range data using regional point descriptors. In: ECCV, Springer

\bibitem[{Fu et~al(2021)Fu, Liu, Luo, and Wang}]{fu2021robust}
Fu K, Liu S, Luo X, et~al (2021) Robust point cloud registration framework based on deep graph matching. In: CVPR

\bibitem[{Gonzalez(2009)}]{gonzalez2009digital}
Gonzalez RC (2009) Digital image processing. Pearson education india

\bibitem[{Hestenes(2015)}]{hestenes2015space}
Hestenes D (2015) Space-time algebra. Springer

\bibitem[{Hestenes and Sobczyk(1984)}]{hestenes_clifford_1984}
Hestenes D, Sobczyk G (1984) Clifford {Algebra} to {Geometric} {Calculus}. Springer Netherlands, Dordrecht, \doi{10.1007/978-94-009-6292-7}, \urlprefix\url{http://link.springer.com/10.1007/978-94-009-6292-7}

\bibitem[{Huang et~al(2021)Huang, Gojcic, Usvyatsov, Wieser, and Schindler}]{huang2021predator}
Huang S, Gojcic Z, Usvyatsov M, et~al (2021) Predator: Registration of 3d point clouds with low overlap. In: CVPR

\bibitem[{Huang et~al(2020)Huang, Mei, and Zhang}]{huang2020feature}
Huang X, Mei G, Zhang J (2020) Feature-metric registration: A fast semi-supervised approach for robust point cloud registration without correspondences. In: CVPR

\bibitem[{Le et~al(2019)Le, Do, Hoang, and Cheung}]{le2019sdrsac}
Le HM, Do TT, Hoang T, et~al (2019) Sdrsac: Semidefinite-based randomized approach for robust point cloud registration without correspondences. In: ICCV

\bibitem[{Leordeanu and Hebert(2005)}]{leordeanu2005spectral}
Leordeanu M, Hebert M (2005) A spectral technique for correspondence problems using pairwise constraints. In: ICCV, IEEE

\bibitem[{Li et~al(2020{\natexlab{a}})Li, Hu, and Ai}]{li2020gesac}
Li J, Hu Q, Ai M (2020{\natexlab{a}}) Gesac: Robust graph enhanced sample consensus for point cloud registration. ISPRS Journal of Photogrammetry and Remote Sensing 167:363--374

\bibitem[{Li et~al(2020{\natexlab{b}})Li, Zhang, Xu, Zhou, and Zhang}]{li2020iterative}
Li J, Zhang C, Xu Z, et~al (2020{\natexlab{b}}) Iterative distance-aware similarity matrix convolution with mutual-supervised point elimination for efficient point cloud registration. In: ECCV, Springer

\bibitem[{Li et~al(2019)Li, Wang, and Fang}]{li2019pc}
Li X, Wang L, Fang Y (2019) Pc-net: Unsupervised point correspondence learning with neural networks. In: 3DV, IEEE

\bibitem[{Li et~al(2020{\natexlab{c}})Li, Wang, and Fang}]{li2020unsupervised}
Li X, Wang L, Fang Y (2020{\natexlab{c}}) Unsupervised partial point set registration via joint shape completion and registration. arXiv preprint arXiv:200905290

\bibitem[{Li et~al(2021)Li, Pontes, and Lucey}]{li2021pointnetlk}
Li X, Pontes JK, Lucey S (2021) Pointnetlk revisited. In: CVPR

\bibitem[{Li and Harada(2022)}]{li2022lepard}
Li Y, Harada T (2022) Lepard: Learning partial point cloud matching in rigid and deformable scenes. ieee. In: CVPR

\bibitem[{Lowe(2004)}]{lowe2004distinctive}
Lowe DG (2004) Distinctive image features from scale-invariant keypoints. IJCV 60

\bibitem[{Lu et~al(2021)Lu, Chen, Liu, Zhang, Qu, Liu, and Gu}]{lu2021hregnet}
Lu F, Chen G, Liu Y, et~al (2021) Hregnet: A hierarchical network for large-scale outdoor lidar point cloud registration. In: ICCV

\bibitem[{Lu et~al(2019)Lu, Wan, Zhou, Fu, Yuan, and Song}]{lu2019deepvcp}
Lu W, Wan G, Zhou Y, et~al (2019) Deepvcp: An end-to-end deep neural network for point cloud registration. In: ICCV

\bibitem[{Lucas and Kanade(1981)}]{lucas1981iterative}
Lucas BD, Kanade T (1981) An iterative image registration technique with an application to stereo vision. In: IJCAI, pp 674--679

\bibitem[{Marchi et~al(2023)Marchi, Bunton, Gas, Gharesifard, and Tabuada}]{marchi2023sharp}
Marchi M, Bunton J, Gas Y, et~al (2023) Sharp performance bounds for pasta. IEEE Control Systems Letters

\bibitem[{Pepe and Lasenby(2023)}]{pepe2023cga}
Pepe A, Lasenby J (2023) Cga-posenet: Camera pose regression via a 1d-up approach to conformal geometric algebra. arXiv preprint arXiv:230205211

\bibitem[{Pepe et~al(2024)Pepe, Lasenby, and Buchholz}]{pepe2024cgaposenet+}
Pepe A, Lasenby J, Buchholz S (2024) Cgaposenet+ gcan: A geometric clifford algebra network for geometry-aware camera pose regression. In: WACV, pp 6593--6603

\bibitem[{Pillai and Megalingam(2020)}]{pillai2020detection}
Pillai SS, Megalingam RK (2020) Detection and 3d modeling of brain tumor using machine learning and conformal geometric algebra. In: 2020 Int. Conf. on Communication and Signal Processing (ICCSP), IEEE

\bibitem[{Qi et~al(2017)Qi, Su, Mo, and Guibas}]{qi2017pointnet}
Qi CR, Su H, Mo K, et~al (2017) Pointnet: Deep learning on point sets for 3d classification and segmentation. In: CVPR

\bibitem[{Qin et~al(2022)Qin, Yu, Wang, Guo, Peng, and Xu}]{qin2022geometric}
Qin Z, Yu H, Wang C, et~al (2022) Geometric transformer for fast and robust point cloud registration. In: CVPR

\bibitem[{Rehman and Lee(2018)}]{rehman2018automatic}
Rehman HZU, Lee S (2018) Automatic image alignment using principal component analysis. IEEE Access 6:72063--72072

\bibitem[{Ruhe et~al(2023{\natexlab{a}})Ruhe, Brandstetter, and Forr{\'e}}]{ruhe2023clifford}
Ruhe D, Brandstetter J, Forr{\'e} P (2023{\natexlab{a}}) Clifford group equivariant neural networks. arXiv preprint arXiv:230511141

\bibitem[{Ruhe et~al(2023{\natexlab{b}})Ruhe, Gupta, De~Keninck, Welling, and Brandstetter}]{ruhe2023geometric}
Ruhe D, Gupta JK, De~Keninck S, et~al (2023{\natexlab{b}}) Geometric clifford algebra networks. arXiv preprint arXiv:230206594

\bibitem[{Rusu et~al(2009)Rusu, Blodow, and Beetz}]{rusu2009fast}
Rusu RB, Blodow N, Beetz M (2009) Fast point feature histograms (fpfh) for 3d registration. In: ICRA, IEEE

\bibitem[{Salti et~al(2014)Salti, Tombari, and Di~Stefano}]{salti2014shot}
Salti S, Tombari F, Di~Stefano L (2014) Shot: Unique signatures of histograms for surface and texture description. CVIU 125

\bibitem[{Schonberger and Frahm(2016)}]{schonberger2016structure}
Schonberger JL, Frahm JM (2016) Structure-from-motion revisited. In: CVPR

\bibitem[{Turk and Levoy(1994)}]{Stanford}
Turk G, Levoy M (1994) Stanford 3d scanning repository. Available at http://graphics.stanford.edu/data/3Dscanrep/, accessed 6 Apr 2023

\bibitem[{Wang et~al(2022)Wang, Liu, Dong, and Wang}]{wang2022you}
Wang H, Liu Y, Dong Z, et~al (2022) You only hypothesize once: Point cloud registration with rotation-equivariant descriptors. In: Proceedings of the 30th ACM International Conference on Multimedia

\bibitem[{Wang and Solomon(2019)}]{wang2019deep}
Wang Y, Solomon JM (2019) Deep closest point: Learning representations for point cloud registration. In: ICCV, pp 3523--3532

\bibitem[{Wong et~al(2017)Wong, Kee, Le, Wagner, Mariottini, Schneider, Hamilton, Chipalkatty, Hebert, Johnson et~al}]{wong2017segicp}
Wong JM, Kee V, Le T, et~al (2017) Segicp: Integrated deep semantic segmentation and pose estimation. In: IROS, IEEE

\bibitem[{Yang et~al(2020{\natexlab{a}})Yang, Antonante, Tzoumas, and Carlone}]{yang2020graduated}
Yang H, Antonante P, Tzoumas V, et~al (2020{\natexlab{a}}) Graduated non-convexity for robust spatial perception: From non-minimal solvers to global outlier rejection. IEEE Robotics and Automation Letters 5(2):1127--1134

\bibitem[{Yang et~al(2020{\natexlab{b}})Yang, Shi, and Carlone}]{yang2020teaser}
Yang H, Shi J, Carlone L (2020{\natexlab{b}}) Teaser: Fast and certifiable point cloud registration. IEEE Transactions on Robotics 37(2):314--333

\bibitem[{Yang et~al(2015)Yang, Li, Campbell, and Jia}]{yang2015go}
Yang J, Li H, Campbell D, et~al (2015) Go-icp: A globally optimal solution to 3d icp point-set registration. IEEE transactions on pattern analysis and machine intelligence 38(11):2241--2254

\bibitem[{Yang et~al(2019)Yang, Xian, Wang, and Zhang}]{yang2019performance}
Yang J, Xian K, Wang P, et~al (2019) A performance evaluation of correspondence grouping methods for 3d rigid data matching. IEEE PAMI 43(6):1859--1874

\bibitem[{Yew and Lee(2020)}]{yew2020rpm}
Yew ZJ, Lee GH (2020) Rpm-net: Robust point matching using learned features. In: CVPR

\bibitem[{Yuan et~al(2020)Yuan, Eckart, Kim, Jampani, Fox, and Kautz}]{yuan2020deepgmr}
Yuan W, Eckart B, Kim K, et~al (2020) Deepgmr: Learning latent gaussian mixture models for registration. In: ECCV, Springer

\bibitem[{Zeng et~al(2017)Zeng, Song, Nie{\ss}ner, Fisher, Xiao, and Funkhouser}]{zeng20173dmatch}
Zeng A, Song S, Nie{\ss}ner M, et~al (2017) 3dmatch: Learning local geometric descriptors from rgb-d reconstructions. In: Proceedings of the IEEE conference on computer vision and pattern recognition, pp 1802--1811

\bibitem[{Zhang and Singh(2014)}]{zhang2014loam}
Zhang J, Singh S (2014) Loam: Lidar odometry and mapping in real-time. In: Robotics: Science and systems, pp 1--9

\bibitem[{Zhou et~al(2016)Zhou, Park, and Koltun}]{zhou2016fast}
Zhou QY, Park J, Koltun V (2016) Fast global registration. In: ECCV, Springer

\bibitem[{Zhu and Fang(2020)}]{zhu2020reference}
Zhu J, Fang Y (2020) Reference grid-assisted network for 3d point signature learning from point clouds. In: CVPR

\end{thebibliography}

\end{document}